\PassOptionsToPackage{table}{xcolor}
\documentclass[abstractbox]{arxiv/custom}
\usepackage[english]{babel}
\usepackage{xcolor}
\usepackage{float}
\usepackage{graphicx}
\usepackage{wrapfig}
\usepackage{amsmath}
\usepackage{amssymb}
\usepackage{amsthm}
\usepackage{booktabs}      % \toprule, \midrule, \bottomrule
\usepackage{array}
\usepackage{multirow}      
\usepackage{caption}  
\usepackage{microtype}
\usepackage{parskip}
\usepackage{titletoc}

\usepackage{enumitem}
\usepackage{etoolbox}

\newtcolorbox{figurepanel}{
    enhanced,
    colback=finn!8!white,
    colframe=finn!20!rose,
    boxrule=0.8pt,
    arc=6pt,
    left=10pt,right=10pt,top=10pt,bottom=10pt
}

\newtcolorbox{samplebox}{
    enhanced,
    colback=finn!5!white,
    colframe=finn!50!black,
    boxrule=0.6pt,
    arc=4pt,
    left=8pt,right=8pt,top=7pt,bottom=7pt
}

\usepackage{algorithm}
\usepackage{algpseudocode}

\newtcolorbox{propbox}{
    enhanced,
    colback=purple!10!white,
    colframe=finn,
    boxrule=0pt,
    leftrule=5pt,
    arc=0pt,
    left=8pt, right=8pt, top=8pt, bottom=8pt,
    fonttitle=\bfseries\sffamily,
    coltitle=finn
}

\input{arxiv/commands/math_commands.tex}
\newcommand{\pdata}{p_\text{data}}

\hypersetup{
   breaklinks=true,   
   colorlinks=true,   
   citecolor=finn,
   linkcolor=rose,
   urlcolor=finn,
}

\usepackage{subcaption}  
\usepackage[export]{adjustbox}
\usepackage{tikz}
\usetikzlibrary{arrows.meta, positioning}
\usepackage{extarrows}
\usepackage{xspace}
\usepackage{enumitem}

\newcommand{\scallop}{\adjustbox{valign=c}{\includegraphics[height=1em]{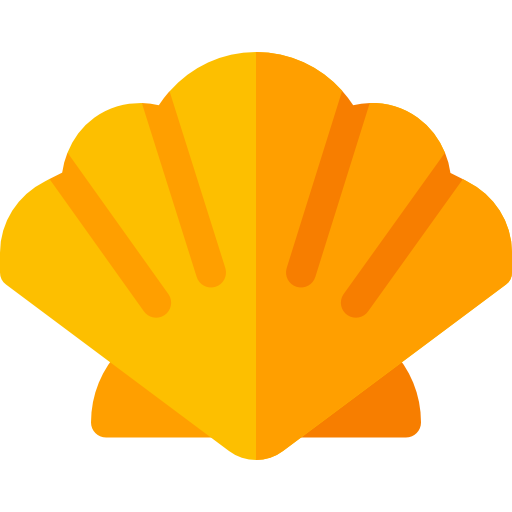}}~\texttt{\color{finn}SCALLOP}}
\newcommand{\scalloptt}{\texttt{SCALLOP}\xspace}

\title{\fontsize{0.69cm}{5.5cm}\selectfont Few-Step Boltzmann Generators via Scalable Likelihood Flow Maps}

\author[1]{RuiKang OuYang$^{\ast}$}
\author[2]{\, Hanlin Yu$^{\dagger}$}
\author[3]{\, Xinyue Ai$^{\dagger}$}
\author[4]{\, Yutong He}
\author[4]{\\Nicholas M. Boffi}
\author[4]{\, Pradeep Ravikumar}
\author[1]{\, José Miguel Hernández-Lobato}
\author[4,5]{\\Max Simchowitz}
\author[6]{\, Benjamin Kurt Miller}
\author[4]{\, Omar Chehab}

\authornote{$^\ast$Correspondence: \textcolor{finn}{\texttt{ro352@cam.ac.uk}}. $^\dagger$Equal contributions.}

\affiliation[1]{University of Cambridge}
\affiliation[2]{University of Helsinki}
\affiliation[3]{Peking University}
\affiliation[4]{Carnegie Mellon University}
\affiliation[5]{Amazon FAR}
\affiliation[6]{FAIR at Meta}

\renewcommand{\projectlinks}{%
  \vspace{0.15cm}
  {\small
    \mbox{\href{https://tonyauyeung.github.io/scallop/}{\raisebox{-0.1em}{\includegraphics[height=1em]{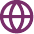}}\enspace\textbf{Project Page}}\qquad
    \href{https://github.com/tonyauyeung/scallop}{\raisebox{-0.1em}{\includegraphics[height=1em]{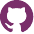}}\enspace\textbf{Code}}\qquad
    \href{https://neurips.cc/virtual/2026/loc/sydney/poster/154720}{\raisebox{-0.25em}{\includegraphics[height=1.3em]{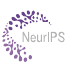}}\enspace\textbf{NeurIPS 2026}}}\par}
}

\abstract{
Recent progress in flow-based generative modeling has led to models that output high-quality samples while using only a small number of function evaluations. However, at present, there is a lack of similar advances in estimating the 
model likelihood.
In particular, most existing methods 
either rely on restrictive architectures that enable exact calculations, or use stochastic approximations such as Hutchinson’s trace estimator that introduce substantial variance.
In this work, we introduce \textsc{SCAlable LikeLihood distillation of flOw maPs} (\scallop). \scalloptt builds on the recently proposed F2D2, a likelihood flow map model that can generate samples and their densities in a small number of function evaluations. While {F2D2} uses Hutchinson's estimator during training, we introduce an alternative and more scalable likelihood distillation objective that is Hutchinson-free and admits a vectorized formulation. 
Empirically, we demonstrate the effectiveness of \scalloptt as a Boltzmann generator in molecular science, and further validate its benefits on image datasets. \scalloptt significantly reduces both training variance and training time while consistently improving performance compared to F2D2, and is competitive with the state-of-the-art while achieving up to a $10\times$ speedup in inference over the fastest baseline.

}

\begin{document}
\maketitle

% Sections
\section{Introduction}
\label{sec:introduction}

While the generative modeling community has focused on improving the empirical quality of generated samples and the speed of generation \citep{song2023consistency,song2024improved,lu2025simplifying, geng2025meanflowsonestepgenerative, boffi2026flowmaps}, numerous applications, notably in the sciences, require access to the model likelihood, e.g., for free energy estimation and enhanced sampling methods \citep{torrie1977nonphysical,grubmuller1995predicting,laio2002escaping,barducci2008well,noé2019boltzmanngeneratorssampling,rissanen2025progressivetemperingsamplerdiffusion,nam2025enhancing,xie2026enhanced}.
% The most expressive flow- and diffusion-based models, along with their few-step distilled counterparts, do not trivially admit likelihood evaluation; instead, they use model  restrictions~\citep{papamakarios2021normalizing}, expensive computation~\citep{rehman2025falconfewstepaccuratelikelihoods} or high-variance estimators~\citep{grathwohl2018ffjordfreeformcontinuousdynamics} \hanlin{this last sentence reads weird to me. Proposal: Previous methods suffer from restrictions on the model architecture~\citep{papamakarios2021normalizing}, expensive computation~\citep{rehman2025falconfewstepaccuratelikelihoods} or high-variance estimators~\citep{grathwohl2018ffjordfreeformcontinuousdynamics}.}.
To satisfy this requirement, early approaches restricted model architectures to ensure strict invertibility \citep{papamakarios2021normalizing}, which inherently limits expressivity. Conversely, continuous flow- and diffusion-based models—along with their few-step distilled counterparts—offer far greater expressivity but do not trivially admit scalable likelihood evaluation. Computing their likelihoods remains bottlenecked by either prohibitively expensive Jacobian calculations \citep{chen2019neuralordinarydifferentialequations,rehman2025falconfewstepaccuratelikelihoods} or high-variance stochastic estimators \citep{grathwohl2018ffjordfreeformcontinuousdynamics}.

% In this work, we particularly interested in the problem of sampling from the Boltzmann distribution of molecular systems, a central and long-standing problem in statistical physics with relevance to protein folding, ligand binding, and drug design. A Boltzmann distribution, $\pi$, is described by an energy function $U:\mathbb{R}^d\rightarrow\mathbb{R}$ with $\pi(x)\propto\exp(-U(x))$. Equilibrium samples drawn from the Boltzmann distribution are used for estimating observables such as free energies and binding affinity \citep{zwanzig1954high,kirkwood1935statistical,gilson1997statistical}. Classical methods such as Molecular Dynamics (MD) \citep{frenkel1997understanding} and Markov-chain Monte Carlo (MCMC) \citep{metropolis1953equation} often suffer from low mixing rate. To remedy, Boltzmann Generators (BGs) \citep{noé2019boltzmanngeneratorssampling} proposes to learn generative model from a biased distribution and then reweigh the generated samples through self-normalized importance sampling (SNIS), where the model density involves.

\paragraph{Current Approaches and Limitations}
Overall, the generative modeling literature has focused on learning a transport map from Gaussian samples to data samples, and then using the change-of-variables formula to compute the corresponding likelihood. Using that formula is often problematic: it requires computing costly quantities based on the neural network, during inference or training.

Likelihoods of normalizing flows~\citep{rezende2015variational,papamakarios2021normalizing} have limited expressivity and can be slow to compute. Normalizing flows parameterize transport maps as compositions of invertible neural network layers and compute likelihoods through the static change-of-variables formula, which requires evaluating their Jacobian log-determinants. To make this computation cheap, namely linear in the data dimension, the neural network has strong architectural constraints. These constraints can substantially reduce expressivity~\citep{gao2020flowcontrastive}, particularly in applications with important symmetries such as molecular and materials modeling, where autoregressive layers are often unsuitable. 
For practical applications, expressive architectures require many neural network layers. For example, normalizing flows based on scalable transformer-based architectures
\citep{zhai2025normalizingflowscapablegenerative} have been successfully applied to atomic systems \citep{tan2026scalableequilibriumsamplingsequential,tan2026amortizedsamplingtransferablenormalizing,schebek2025scalableboltzmanngeneratorsequilibrium}. Still, at inference time, they often suffer from a trade-off in speed between generating samples and computing likelihoods. 

Computing likelihoods from continuous normalizing flows (CNFs)~\citep{chen2019neuralordinarydifferentialequations} typically suffers from high variance. CNFs remove the architectural constraints of normalizing flows by parameterizing the velocity field of a continuous transport process. Likelihood computation then requires integrating the divergence of the velocity field along trajectories through the instantaneous change-of-variables formula~\citep{chen2019neuralordinarydifferentialequations}. In practice, this divergence is typically estimated using Hutchinson's trace estimator~\citep{grathwohl2018ffjordfreeformcontinuousdynamics}, whose variance can become prohibitively large in high dimensions: maintaining constant precision may require computational budgets scaling quadratically with the data dimension. While variance-reduction techniques have been developed~\citep{meyer2021hutchinson,liu2025hutch}, they are not mainstream.

% Most recent attempts to distill CNFs require heavy computation to obtain their likelihood. Flow map models~\citep{boffi2025buildconsistencymodellearning} compress sampling into a single neural network evaluation, yet likelihood evaluation is not available. Unlike normalizing flows, these models lack special architectural structure and may not satisfy the invertibility requirement. Therefore, recent works encourage to enforce model invertibility through a regularization term~\citep{draxler2024freeformflowsmakearchitecture,rehman2025falconfewstepaccuratelikelihoods}, yet the computation remains expensive, cubic in the data dimension, due to the requirement of calculating the log determinant of Jacobian of the neural network.
% \citet{ai2026f2d2} treats the likelihood evaluation problem as another distillation problem. They jointly distill the sampling and likelihood evaluation process into a few network evaluations.
% However, training still relies on Hutchinson-based likelihood estimates, inheriting their high variance. An alternative direction has been to train an energy-based parameterization of the CNF~\citep{aggarwal2025boltzncelearninglikelihoodsboltzmann,ouyang2026diffusiveclassificationlosslearning,plainer2026consistentsamplingsimulationmolecular,he2026rneplugandplaydiffusioninferencetime,guth2026learningnormalizedimagedensities}: at inference time, the likelihood is obtained in a single network evaluation, however it doesn't allow few-step generation.
Recent CNF distillation methods struggle to reconcile fast sampling with scalable and tractable likelihood evaluation. For instance, while flow map models~\citep{boffi2025buildconsistencymodellearning} successfully compress sampling into a single forward pass, they do not natively support tractable likelihood evaluation. Unlike standard normalizing flows, these models lack explicit architectural constraints and are not guaranteed to be bijective. To address this, subsequent works attempt to enforce model invertibility via regularization techniques~\citep{draxler2024freeformflowsmakearchitecture,rehman2025falconfewstepaccuratelikelihoods}; however, likelihood evaluation under these models remains computationally expensive—scaling cubically with the data dimension ($\mathcal{O}(d^3)$) due to the required Jacobian log-determinant calculation. To bypass this bottleneck, \citet{ai2026f2d2} frame likelihood evaluation as a parallel distillation task, jointly amortizing both the sampling process and the likelihood estimation into a few network evaluations. Nevertheless, training this joint architecture still relies on Hutchinson-based divergence estimators, thereby inheriting their characteristic high variance. Alternatively, energy-based parameterizations of CNFs~\citep{aggarwal2025boltzncelearninglikelihoodsboltzmann,ouyang2026diffusiveclassificationlosslearning,plainer2026consistentsamplingsimulationmolecular,he2026rneplugandplaydiffusioninferencetime,guth2026learningnormalizedimagedensities} enable likelihood estimation in a single network evaluation, but do not support few-step generation.

\paragraph{Contributions.}
In this work, we build on F2D2~\citep{ai2026f2d2} that extends flow maps into \textit{likelihood flow maps} that distill both the sampling procedure and the likelihood computation into a single neural network evaluation. We propose \textsc{SCAlable LikeLihood distillation of flOw maPs} (\scallop), which replaces the likelihood distillation loss in F2D2 with an alternative, based on reformulating the change-of-density along the flow. Our contributions are as follows:
\begin{itemize}[leftmargin=*]
    \item We introduce a \textit{Hutchinson-free}, vectorized likelihood-distillation objective with the same population gradients as divergence matching for an exact marginal velocity field (\Cref{prop:conditional-divergence-matching,prop:vectorized-divergence-matching}). Lower training variability is supported empirically by \Cref{fig:training_analysis,fig:loss-grad-comparison}.

    \item Empirical improvements over the F2D2 model, on which \scalloptt is based, are substantial: \scalloptt can reduce training variance by up to $100\times$, converges faster, and consistently obtains better test set performance.

    \item  In molecular domains, \scalloptt is much faster than state-of-the-art Boltzmann Generators at inference time: $10\times$ faster than normalizing flows~\citep{tan2026scalableequilibriumsamplingsequential} while achieving better performance, and $100\times$ faster than the comparable FALCON~\citep{rehman2025falconfewstepaccuratelikelihoods}.

    \item These empirical gains translate to higher-dimensional image domains, where \scalloptt estimates more accurate likelihoods than F2D2, while being faster and more stable to train. 

\end{itemize}

\section{Background}
\label{sec:background}
We motivate our study with a common setting: A practitioner is interested in sampling from an unnormalized distribution $\pi$ to approximate some macroscopic property like free energy differences, denoted $O$. $\pi$ takes on a Boltzmann distribution and the observable $O$ is computed in expectation:
\begin{align}
    \pi(x)=\frac{1}{Z}\exp({-U(x)}),
    \quad \mathbb{E}_{\pi}[O(X)] = \int O(x) \pi(x)\, dx,
\end{align}
where $U:\mathbb{R}^d\rightarrow\mathbb{R}$ is the energy and $Z$ is a normalizing constant. Since drawing unbiased samples from the Boltzmann distribution is difficult, Boltzmann Generators \citep{noé2019boltzmanngeneratorssampling} fit a generative model $p^{\theta}$ to a potentially biased, accessible distribution $\pdata$ with similarities to $\pi$. Samples from $p^{\theta}$ are then used to approximate $\mathbb{E}_\pi[O(X)]$ with self-normalized importance sampling (SNIS):
% 
% simulate biased data $\pdata$ and learn a generative model to approximate that data $p^{\theta}$. One can then utilize the generative model as a proposal to approximate $O$ with self-normalized importance sampling (SNIS):
% 
\begin{align}
    \label{eqn:snis}
    \mathbb{E}_{\pi}[O(X)]
    &=\mathbb{E}_{p^\theta}\left[\frac{\pi(X)}{p^\theta(X)}O(X)\right] 
    \approx \frac{\sum_{n=1}^N w_nO(x_n)}{\sum_{m=1}^Nw_m},\quad\text{where } w_n=\frac{\exp(-U(x_n))}{p^\theta(x_n)},
    \, x_n \sim p^\theta.
\end{align}
{SNIS requires evaluating the probability density of the generated sample $x$ under $p^\theta$, commonly referred to as the model likelihood. Provided no ambiguity arises, we use the terms \emph{sample density} and \emph{model likelihood} interchangeably throughout this paper.} 
% Presently, the most expressive models, such as flow models~\citep{song2021scorebasedgenerativemodelingstochastic,lipman2023flow}, rely on continuous time formulations where the density is computationally expensive to obtain. 
% Previous approaches have limited \hanlin{drop have?} themselves to less expressive models with a tractable likelihood, or utilized the Hutchinson trace estimator at inference-time or during training of a distilled flow map. We offer an alternative to the latter, namely, \emph{distillation of likelihood flow maps that generate samples with approximate densities from continuous-time models via a proposed objective that does not require Hutchinson during training or inference}. In addition to the motivation above, likelihood approximation is of general interest in machine learning with additional applications including reward tilting and model composition.
Presently, the most expressive generative architectures, such as continuous-time flow models~\citep{song2021scorebasedgenerativemodelingstochastic,lipman2023flow}, require computationally expensive simulations to evaluate exact densities. Prior works either restrict themselves to less expressive models with tractable likelihoods, or rely on the Hutchinson trace estimator during training or inference. In this work, we present an alternative to the latter: \emph{a novel distillation objective for likelihood flow maps that enables simultaneous sample generation and density approximation from continuous-time models without requiring Hutchinson's trace estimator at any stage}. Beyond our primary motivation, efficient likelihood estimation is of broad interest across machine learning, with applications extending to reward tilting and model composition.

\paragraph{Continuous Flows as Boltzmann Generators.} 
A continuous flow is an Ordinary Differential Equation (ODE) that transports samples from a base distribution $p_0=\mathcal{N}(0,I)$ to the data distribution $p_1=\pdata$, through a sequence of intermediate distributions $(p_t)_{t\in[0,1]}$. It can be used to generate samples and their densities by integrating
\begin{align}
    \frac{\dd}{\dd t}\begin{bmatrix}
x_t \\
\log p_t^\theta(x_t)
\end{bmatrix}=\begin{bmatrix}
v_t^\theta(x_t) \\
-\nabla\cdot v_t^\theta(x_t)
\end{bmatrix},\label{eq:augmented-pf-ode}\quad\text{with }x_0\sim p_0\text{ and }\log p_0^\theta(x_0)=\log p_0(x_0),
\end{align}
where $\nabla\cdot$ is the divergence operator. The upper equation for generating samples is called the \emph{Probability Flow ODE} (PF-ODE) and the lower equation for generating their log-densities is commonly called the \textit{instantaneous change-of-variables} equation~\citep{chen2019neuralordinarydifferentialequations}. 

In many machine learning applications, the intermediate distributions $(p_t)_t$ are defined by the user, for example,  using a \textit{linear interpolation} between samples, $x_t=(1-t)x_0+tx_1$ with $x_0\sim p_0$ and $x_1\sim p_1$. 
% \bkm{Don't we need to be careful here. Only the conditional distributions are defined by the user?} 
% \omar{Actually no :) The marginal distributions are defined by the user too. On top of that, the user can different conditional distributions (that are consistent with the marginals) for training.}
We will use this choice as a default throughout this paper. Flow matching (FM)~\citep{lipman2023flow,liu2022rectifiedflowmarginalpreserving} learns a corresponding velocity field  $v_t^{\theta}(x)$ and Diffusion Models (DMs) \citep{JMLR:v6:hyvarinen05a,ho2020denoisingdiffusionprobabilisticmodels,song2021scorebasedgenerativemodelingstochastic} learn 
% \bkm{estimate marginal scores} 
marginal scores $(\nabla\log p_t^{\theta})_t$ which are a simple reparameterization of the velocity field, following
\begin{align}
    v_t(x_t)=\frac{x_t}{t}+\frac{1-t}{t}\nabla\log p_t(x_t).\label{eq:velocity-score-relation-in-diffusion-main}
\end{align}

Note that evaluating the density using \Cref{eq:augmented-pf-ode} requires the exact computation of $\nabla\cdot v_t^\theta$, which is infeasible in high-dimensional spaces. A popular approximation is to use Hutchinson's trace estimator \citep{Hutchinson01011990}, which estimates $\nabla\cdot v_t^{\theta}(x_t)=\mathbb{E}_\epsilon [\epsilon^\top\nabla v_t^{\theta}(x_t)\epsilon]$ with a zero-mean random vector $\epsilon$, independent of $x_t$, satisfying $\mathbb{E}[\epsilon\epsilon^\top]=I$, where $\nabla v_t^{\theta}(x_t)\epsilon$ can be efficiently computed using the Jacobian-Vector Product (JVP) operation. However, this estimator suffers from high variance, and the JVP still introduces some computational overhead.

\paragraph{Flow Maps for Few-step Sampling.} 
Generating samples and their densities by numerical integration of~\Cref{eq:augmented-pf-ode}
requires hundreds of function evaluations. To reduce this computational inefficiency, flow maps \citep{frans2025stepdiffusionshortcutmodels,boffi2025buildconsistencymodellearning,geng2025meanflowsonestepgenerative} directly integrate the \textit{sample generation} given by the upper part of~\Cref{eq:augmented-pf-ode}. In practice, flow map models are neural networks, trained to satisfy the consistency constraint
\begin{align}
    \Phi(x_t, t, s)=x_t+\int_t^s v_\tau(x_\tau)\dd\tau=x_s.
\end{align}
Training objectives for the flow map are decomposed into two terms. For $t=s$, the flow map satisfies $\Phi(x_t,t,t)=x_t$, and matching $\partial_s\Phi(x_t,t,s)|_{s=t}=v_t(x_t)$ gives the classic FM objective; for $t< s$, the consistency is imposed via (1) the Lagrangian equation: $\dd_s\Phi(x_t, t, s)=v_s(x_s)$,  (2) the Eulerian equation: $\dd_t\Phi(x_t, t, s)=0$, or (3) the semigroup property: $\Phi(\Phi(x_t, t, r), r, s)=\Phi(x_t, t, s)$ for any $t<r<s$.
For a trained flow map with parameter $\theta$, the generation process is solved by iteratively applying $\hat{x}_{t_{i+1}}=\Phi_\theta(\hat{x}_{t_{i}}, t_i, t_{i+1})$ with $0=t_0<...<t_N=1$ and $\hat{x}_{t_0}\sim p_0$, involving $N\ll1000$ function evaluations (NFEs).

Extending flow map models so that they simultaneously integrate the \textit{likelihood generation} given by the bottom part of~\Cref{eq:augmented-pf-ode} was recently proposed by~\citet{ai2026f2d2}. 

% \tony{
% We can intro $X_t=\beta_tX_1+\gamma_tZ$ which can then easily generalized to SI $X_t=\alpha_tX_0+\beta_tX_1+\gamma_tZ$. Or we could directly introduce $X_t=\alpha_tX_0+\beta_tX_1+\gamma_tZ$, then DM is recovered by setting $X_0\equiv0$, i.e. $p_0=\delta_0$.
% }

\section{Scalable Likelihood Distillation in Flow Maps}
We focus on learning a flow map-based, few-step generative model with co-evaluation of the model likelihood. Following the terminology of~\citet{ai2026f2d2}, we call such a family \emph{likelihood flow maps} and introduce a scalable training objective after reviewing the recent \texttt{F2D2} \citep{ai2026f2d2}.

Throughout the paper, we use the linear interpolation $x_t=(1-t)x_0+tx_1$ with $x_0\sim \mathcal{N}(0,I)$ {to construct the corrupted samples}. This defines a sequence of marginal distributions $(p_t)_t$ that could be induced by a velocity field $(v_t)_t$ through solving the ODE $\dd x_t=v_t(x_t)\dd t$. 
% While a more generalized version \textit{w.r.t.} the Stochastic Interpolants \citep{albergo2025interpolant}, \textit{i.e.} $x_t=\alpha_tx_0+\beta_tx_1+\gamma_tz$, is presented and discussed in \Cref{tbd}.
Although $v_t$ is typically intractable, we assume a distillation setting with access to an approximate velocity field $v_t^\phi$. We omit $\phi$ for simplicity and assume $v_t^\phi \approx v_t$.
% We also consider a distillation setting, where a pretrained velocity $v_t^\phi\approx v_t$ is available. Without ambiguity, we denote $v_t$ as $v_t^\phi$ for notational simplicity.

\subsection{Likelihood Flow Maps Enable Few-Step Likelihood Evaluation} 

We first define the \emph{likelihood flow map}, which is an operator that outputs both the generated samples and the associated log-density change jointly. 
\begin{definition}[Likelihood Flow Map]
\label{def:likelihood-flow-map}
Given a density-augmented PF-ODE as described in \Cref{eq:augmented-pf-ode}, the likelihood flow map $\Psi:\mathbb{R}^d\times[0,1]^2\rightarrow\mathbb{R}^{d+1}$ is the solution operator that maps any state at time $t$ to its corresponding state at time $s$, with $s>t$, as well as the log-density change:
\begin{align}
    \Psi(x_t, t, s)=\begin{bmatrix}
        x_t\\0
    \end{bmatrix}+\int_t^s\begin{bmatrix}
         v_\tau(x_\tau)\\
        -\nabla\cdot v_\tau(x_\tau)
    \end{bmatrix}\dd\tau
    =\begin{bmatrix}
        x_s\\\log p_s(x_s)-\log p_t(x_t)
    \end{bmatrix}.\label{eq:augmented-flow-map}
\end{align}
\end{definition}
\citet{ai2026f2d2} propose to parameterize the likelihood flow map through the skip-connection mechanism
\begin{align}
    \Psi_\theta(x_t, t, s)=\begin{bmatrix}
        x_t\\0
    \end{bmatrix}+(s-t)\begin{bmatrix}
        u_\theta(x_t, t, s)\\D_\theta(x_t, t, s)
    \end{bmatrix}\Longrightarrow
    \begin{cases}
        u_\theta(x_t, t, t)=v_t(x_t)\\D_\theta(x_t, t, t)=-\nabla\cdot v_t(x_t)
    \end{cases},\label{eq:instantaneous-probability-flow-map}
\end{align}
where $u_\theta$ and $D_\theta$ are neural networks that share a subset of their parameters. When $s\rightarrow t$, this formulation naturally recovers the instantaneous velocity and instantaneous log-density-change.
For clarity, we will use $f_\theta(x_t, t, s)=\begin{bmatrix}
            u_\theta(x_t, t, s)\\
            D_\theta(x_t, t, s)
\end{bmatrix}$ and $b_t(x_t)=\begin{bmatrix}
v_t(x_t)\\
-\nabla\cdot v_t(x_t)
\end{bmatrix}$.

\subsection{How to Train a Likelihood Flow Map?}\label{sec:how-to-train-likelihood-flowmap}
According to the instantaneous behavior as described by \Cref{eq:instantaneous-probability-flow-map}, training a likelihood flow map \citep{ai2026f2d2} is analogous to training a standard flow map \citep{boffi2025buildconsistencymodellearning}: we simply allocate an extra dimension for the log-density. Under the skip-connection parameterization,
training a likelihood flow map is essentially learning to satisfy $(s-t)f_\theta(x_t, t, s)\approx \int_t^sb(x_\tau, \tau)\dd\tau$.
The training objective integrates the following \textit{three losses}:
% be written as a decomposed into a sum of \textit{three terms} as follows
%
\begin{align}
    \mathcal{L}_\mathrm{F2D2}(\theta)=\int_0^1 w_v(t)\mathcal{L}_{v}(\theta;t)+w_\mathrm{DM}(t)\mathcal{L}_{\mathrm{DM}}(\theta;t)\dd t + \int_0^1\int_0^s w(t, s)\mathcal{L}_\mathrm{SD}(\theta;t, s)\dd t\dd s,
\end{align}
where
$w:[0,1]^2\rightarrow\mathbb{R}^+$ and $w_v,w_\mathrm{DM}:[0,1]\rightarrow\mathbb{R}^+$ are weighting functions. The first two losses $\mathcal{L}_v$ and $\mathcal{L}_{\mathrm{DM}}$ are instantaneous losses written as
\begin{align}
    (\text{Flow Matching})\quad&\mathcal{L}_{v}(\theta;t)=\mathbb{E}_{x_0,x_1} \Bigl[{\left\|u_\theta(x_t, t, t)-v_t(x_t)\right\|^2}\Bigr],\\
    (\text{Divergence Matching})\quad&\mathcal{L}_{\mathrm{DM}}(\theta;t)=\mathbb{E}_{x_0,x_1} \Bigl[{\left|D_\theta(x_t, t, t)+\nabla\cdot v_t(x_t)\right|^2}\Bigr],
\end{align}
with $x_t=(1-t)x_0+tx_1$. The third loss $\mathcal{L}_\mathrm{SD}$ is called the self-distillation loss and has three different instantiations:
% We provide the Lagrangian self-distillation (LSD), \textit{i.e.} $\dd_s\Psi(x_t, t, s)=b_s(x_s)$ as follows, while leaving the other alternatives for self-distillation (ESD and PSD) in \Cref{app:sec:self-distillation-losses}:
% \begin{align}
%         \mathcal{L}_\mathrm{LSD}(\theta;t, s)=\mathbb{E}_{x_0,x_1}\Bigl[\bigl\|f_\theta(x_t, t, s)+(s-t)\mathrm{sg}\bigl(\partial_sf_\theta(x_t, t, s)\bigr) - \mathrm{sg}\bigl(f_\theta(x_s, s, s)\bigr)\bigr\|^2\Bigr],\label{eq:lsd-loss}
% \end{align}
% where $\mathrm{sg}(\cdot)$ denotes stop-gradient.
(1) the Lagrangian equation: $\dd_s\Phi(x_t, t, s)=v_s(x_s)$,  (2) the Eulerian equation: $\dd_t\Phi(x_t, t, s)=0$, or (3) the semigroup property: $\Phi(\Phi(x_t, t, r), r, s)=\Phi(x_t, t, s)$ for any $t<r<s$. Adapted to the likelihood flow map setting, these three consistency constraints read:
\begin{enumerate}[leftmargin=*,label=(\arabic*)]
    \item the Lagrangian self-distillation (LSD), \textit{i.e.} $\dd_s\Psi(x_t, t, s)=b_s(x_s)$
    \noindent
    \begin{equation}
        \mathcal{L}_\mathrm{LSD}(\theta;t, s)=\mathbb{E}_{x_0,x_1}\left[\bigl\|f_\theta(x_t, t, s)+(s-t)\mathrm{sg}\bigl(\partial_sf_\theta(x_t, t, s)\bigr) - \mathrm{sg}\bigl(f_\theta(x_s, s, s)\bigr)\bigr\|^2\right];\label{eq:lsd-loss}
    \end{equation}
    \item the Eulerian self-distillation (ESD), \textit{i.e.} $\dd_t\Psi(x_t, t, s)=0$
    \begin{multline}
        \mathcal{L}_\mathrm{ESD}(\theta;t, s)=\mathbb{E}_{x_0,x_1}\biggl[\Bigl\|f_\theta(x_t, t, s)- \mathrm{sg}\bigl(f_\theta(x_t, t, t)\bigr)\\
        -(s-t)\mathrm{sg}\bigl(\nabla f_\theta(x_t, t, s)\cdot b_t(x_t)+\partial_tf_\theta(x_t, t, s)\bigr)\Bigr\|^2\biggr];\label{eq:esd-loss}
    \end{multline}
    \item the progressive self-distillation (PSD), \textit{i.e.} $\Psi(\Psi(x_t, t, u), u, s)=\Psi(x_t, t, s)$ $\forall u\in[t, s]$
    \begin{multline}
        \mathcal{L}_\mathrm{PSD}(\theta;t, s)=\int_t^s\mu(u)\mathbb{E}_{x_0,x_1}\biggl[\Bigl\|f_\theta(x_t, t, s)\\
        -\mathrm{sg}\biggl(\frac{u-t}{s-t}f_\theta(x_t, t, u)+\frac{s-u}{s-t}f_\theta\bigl(\Psi_\theta(x_t, t, u), u, s\bigr)\biggr)\Bigr\|^2\biggr]\dd u,\label{eq:psd-loss}
    \end{multline}
\end{enumerate}
where $\mathrm{sg}(\cdot)$ denotes stop-gradient and $\mu:[0,1]\rightarrow\mathbb{R}^+$ is another weighting function for $u$ which is usually chosen to concentrate at the midpoint, $\mu(u)=\delta(u-\frac{t+s}{2})$ \citep{frans2025stepdiffusionshortcutmodels}.

\citet{ai2026f2d2} proposed {F2D2}, a likelihood flow map trained with the above objectives. Despite the mathematical simplicity, {F2D2} relies on Hutchinson's trace estimator \citep{Hutchinson01011990} to estimate the divergence term $\nabla\cdot v_t$ in $b_t$ to replace the computationally expensive $\mathcal{L}_\mathrm{DM}$ term, namely \emph{Hutchinson Divergence Matching}:
{
\begin{align}
    \mathcal{L}_{\mathrm{HDM}}(\theta;t)=\mathbb{E}_{x_0,x_1}\mathbb{E}_\epsilon\left[\left|D_\theta(x_t, t, t)+\epsilon^\top\nabla v_t(x_t)\epsilon\right|^2\right],
\end{align}}
where $\epsilon$ is independent of $(x_0,x_1)$, has zero mean and satisfies $\mathbb{E}[\epsilon\epsilon^\top]=I$, as for $\mathcal{N}(0,I)$, and it can be shown that $\arg\min_\theta \mathcal{L}_{\mathrm{HDM}}(\theta;t)=\arg\min_\theta \mathcal{L}_{\mathrm{DM}}(\theta;t)$. Though $\nabla v_t(x_t)\epsilon$ can be implemented through a \texttt{JVP}, making it more efficient than the exact divergence calculation, it remains computationally expensive during training, especially in large-scale systems, and suffers from high variance. 

\subsection{Scalable Likelihood Distillation via Conditional Divergence Matching}\label{sec:likelihood-flowmap-conditional-divergence-matching}
The computational bottleneck in \texttt{F2D2} is introduced by having to estimate $-\nabla\cdot v_t(x_t)=\dd_t\log p_t(x_t)$ whenever $b_t(x_t)$ needs to be calculated, which also suffers from high variance. Inspired by the success of Flow Matching \citep{lipman2023flow,liu2022rectifiedflowmarginalpreserving}, which leverages a conditional objective 
% \emph{Conditional Flow Matching} (CFM) 
to bypass the intractable velocity $v_t(x_t)$ during training, we raise the following question:
\begin{center}
    \textit{Is there any conditional objective that allows us to bypass the divergence calculation/estimation?}
\end{center}
% And the answer is, yes. 
Our answer is positive. The key derivation is given by the chain rule: given $x_t$ generated by a known velocity field $v_t$, the chain rule reads as
{
\begin{align}
    \dd_t\log p_t(x_t)=\nabla\log p_t(x_t)\cdot v_t(x_t) + \partial_t\log p_t(x_t); \label{eq:chain-rule-for-logp-main}
\end{align}}
this identity is also used by \citet{2021Density} for density ratio estimation.
\begin{remark}
    Alternatively, \Cref{eq:chain-rule-for-logp-main} can be derived from the continuity equation as stated in \Cref{app:sec:cond-total-derivative-through-continuity-equation}.
\end{remark}
It is well known that the Stein score $\nabla\log p_t$ can be learned via \emph{Denoising Score Matching} \citep[DSM,][]{JMLR:v6:hyvarinen05a,vincent2011dsm,ho2020denoisingdiffusionprobabilisticmodels,song2021scorebasedgenerativemodelingstochastic}, while the time score $\partial_t\log p_t$ can be learned via \emph{Conditional time Score Matching} \citep[C$t$SM,][]{yu2025density,guth2026learningnormalizedimagedensities}. Combining both, we propose the \emph{Conditional Divergence Matching} (CDM) objective, which bypasses the divergence operator:
% \omar{Not a fan of the name because our main contribution is rewriting the likelihood distillation loss \textit{without} the divergence. In F2D2 they call the likelihood distillation loss a ``divergence matching" loss because there is a divergence. But in our case, the likelihood distillation loss is no longer a ``divergence matching" loss because there no longer is a divergence and therefore no Hutchinson. Do you agree? Can we find a different name? We can say something like: ``we rewrite the likelihood distillation loss (Eq X) without the divergence operator, and call it ``conditional trajectory likelihood matching" / ``conditional pathwise likelihood matching" / whatever.}
%
\begin{align}
    \mathcal{L}_\mathrm{CDM}(\theta;t)=\mathbb{E}_{\xi,x_t}\Bigl[\bigl|D_\theta(x_t, t, t) - v_{t}(x_{t}) \cdot \nabla\log p_{t}(x_{t}|\xi) - \partial_{t}\log p_{t}(x_{t}|\xi)\bigr|^2\Bigr]
\end{align}
where $\xi$ is a conditioning variable such that conditional time and space scores are available in closed form. For example, when using the interpolation $x_t = (1-t) x_0 + t x_1$, then conditioning on $\xi = x_1$ leaves only one source of randomness in $x_0$, which is Gaussian. It follows that $p_{t}(x_{t}|\xi)$ becomes tractable, and plugging it into the above yields the following proposition:
%
% \hanlin{We need to stress that the following expression is for Cond OT path. Proposal:
% \begin{align}
%     \mathcal{L}_\mathrm{CDM}(\theta;t)&=\mathbb{E}_{t,z,x_{t}}\left[\left|D_\theta(x_t, t, t) - v_{t}(x_{t}) \cdot \nabla\log p_{t}(x_{t}|z) - \partial_{t}\log p_{t}(x_{t}|z)\right|^2\right]
% \end{align}
% Then state that it reduces to the current one for Cond OT path.
% }
%
\begin{proposition}[Conditional Divergence Matching]\label{prop:conditional-divergence-matching}
    Let $(p_t)_t$ be the laws of $x_t=(1-t)x_0+tx_1$, where $x_0\sim p_0=\mathcal{N}(0,I)$ and $x_1\sim p_1$ are independent.
    Let $v_t(x)=\mathbb{E}[x_1-x_0\mid x_t=x]$ be the corresponding marginal velocity field, and $D_\theta:\mathbb{R}^d\times[0,1]^2\rightarrow\mathbb{R}$ the parameterized divergence predictor.
    The Conditional Divergence Matching objective
    \begin{align}
        \mathcal{L}_\mathrm{CDM}(\theta;t)=\mathbb{E}_{x_0,x_1}\Bigl[\bigl|D_\theta(x_t, t, t)+\frac{x_0}{1-t}\cdot\bigl(v_t(x_t)-(x_1-x_0)\bigr)-\frac{1}{1-t}d\bigr|^2\Bigr]
    \end{align}
    equals the Divergence Matching objective $\mathcal{L}_\mathrm{DM}(\theta;t)$ up to a constant independent of $\theta$, for all $t\in[0,1)$ for which the losses are finite. Hence, $\nabla_\theta\mathcal{L}_\mathrm{CDM}(\theta;t)=\nabla_\theta\mathcal{L}_\mathrm{DM}(\theta;t)$.
\end{proposition}

\begin{remark}
    For an approximate pretrained velocity field, \Cref{app:sec:error-bounds} quantifies the resulting likelihood-estimation and population importance-sampling errors under suitable regularity and density-ratio assumptions.
\end{remark}

\begin{remark}
    The CDM objective can also be derived by first expressing the velocity as the posterior expectation over the conditional velocity, and then manipulating the divergence operator to obtain the divergence-operation-free training objective. Details are provided in \Cref{app:sec:cond-trace}.
\end{remark}

\begin{remark}
    While the velocity $v_t$ and the Stein score $\nabla\log p_t$ are connected by \Cref{eq:velocity-score-relation-in-diffusion-main} when $p_0=\mathcal{N}(0,I)$ and $x_t=(1-t)x_0+tx_1$, the mapping between the two becomes nontrivial for general stochastic interpolants, as stated in \Cref{eq:velocity-score-relation-in-stochastic-interpolants}. Hence, we treat $\nabla\log p_t$ as intractable. This can introduce additional variance to the training objective in the $\mathcal{N}(0,I)$ base cases. However, we found that it has little effect on performance in practice
    % , as stated in \tocite
    .
\end{remark}

% \hanlin{I was wondering about directly plugging in the score derived from the velocity into the above formula. I tried it but it does not seem to make much of a difference. It is a bit unclear whether this even reduces variances.}

\subsection{Vectorized Divergence Matching}\label{sec:vectorized-divergence-matching}
% The optimization problem of learning the divergence and the log-density-change part in the likelihood flow map as scalars is generally less constrained than if they were vectors.
Since the divergence is the trace of the velocity Jacobian, predicting its diagonal entries provides a componentwise training signal whose sum recovers the scalar target.
% Because the divergence and the log-density are inherently scalar quantities, learning a model to predict them directly requires more data than if they were vectors.
% \omar{``Because the divergence and the log-density are scalar quantities, learning their parameters is arguably more data-greedy than learning from a vector signal".}
To regress vectors instead of scalars, we propose the following vectorized variants of the divergence matching objectives: (1) \emph{Vectorized Divergence Matching} (DM-v), (2) \emph{Vectorized Hutchinson Divergence Matching} (HDM-v), and (3) \emph{Vectorized Conditional Divergence Matching} (CDM-v). In particular, the effective training target changes from the negative divergence $-\nabla\cdot v_t$ to the negative diagonal entries of the Jacobian $-\mathrm{diag}(\nabla v_t)$. For notational clarity, we use the symbolic notation $\nabla\odot v_t$ to denote $\mathrm{diag}(\nabla v_t)$. {In particular, the divergence predictor is now parameterized to output a vector, i.e. $D_\theta:\mathbb{R}^d\times[0,1]^2\rightarrow\mathbb{R}^d$}:
\begin{align}
    \mathcal{L}_\mathrm{DM\text{-}v}(\theta;t)&=\mathbb{E}_{x_0,x_1}\left[\left\|D_\theta(x_t, t, t)+\nabla\odot v_t(x_t)\right\|^2\right],\label{eq:vectorized-divergence-matching}\\
    \mathcal{L}_\mathrm{HDM\text{-}v}(\theta;t)&=\mathbb{E}_{x_0,x_1, \epsilon}\left[\left\|D_\theta(x_t, t, t)+\epsilon\odot\nabla v_t(x_t)\epsilon\right\|^2\right],\label{eq:vectorized-hutchinson-divergence-matching}\\
    \mathcal{L}_\mathrm{CDM\text{-}v}(\theta;t)&=\mathbb{E}_{x_0,x_1}\biggl[\Bigl\|D_\theta(x_t, t, t)+\frac{x_0}{1-t}\odot\bigl(v_t(x_t)-(x_1-x_0)\bigr)-\frac{1}{1-t}\mathbf{1}_d\Bigr\|^2\biggr].\label{eq:vectorized-conditional-divergence-matching}
\end{align}

\begin{proposition}[Vectorized Divergence Matching]\label{prop:vectorized-divergence-matching}
    Let $(v_t)_t$ and $(p_t)_t$ be defined as in~\Cref{prop:conditional-divergence-matching}, and $D_\theta:\mathbb{R}^d\times[0,1]^2\rightarrow\mathbb{R}^d$ be the parameterized \textbf{vectorized} divergence predictor.
    The vectorized learning objectives defined in \Cref{eq:vectorized-divergence-matching,eq:vectorized-hutchinson-divergence-matching,eq:vectorized-conditional-divergence-matching} are related as
    \begin{align}
        \arg\min_\theta \mathcal{L}_\mathrm{HDM\text{-}v}(\theta;t)&=\arg\min_\theta\mathcal{L}_\mathrm{DM\text{-}v}(\theta;t)\quad\text{and}\quad\nabla_\theta\mathcal{L}_\mathrm{CDM\text{-}v}(\theta;t)=\nabla_\theta\mathcal{L}_\mathrm{DM\text{-}v}(\theta;t),
    \end{align}
    for all $t\in[0,1)$ for which the losses are finite. When the predictor family can represent the target, an optimal $\theta^*=\arg\min_\theta\mathcal{L}_\mathrm{DM\text{-}v}(\theta;t)$ recovers
    \begin{align}
        D_{\theta^*}(x_t, t, t)&=-\nabla\odot v_t(x_t)\quad\text{and}\quad\mathbf{1}_d\cdot D_{\theta^*}(x_t, t, t)=-\nabla\cdot v_t(x_t).
    \end{align}
\end{proposition}
The same vectorization technique is also used by~\citet{ou2025improvingprobabilisticdiffusionmodels} for learning the covariance of the denoising kernel in Diffusion models, and~\citet{yu2025density} for learning the time score.

\paragraph{Connection to Second-Order Score Matching.} {Our CDM-v objective is related to the second-order Score Matching objective~\citep{mengEstimatingHighOrder2021}. At $0<t<1$, recall the relationship between the velocity and score described in \Cref{eq:velocity-score-relation-in-diffusion-main}:}
% To connect CDM-v with the second-order Score Matching, one should recall the connection between the velocity and score described in \Cref{eq:velocity-score-relation-in-diffusion-main}:
\begin{align}
    v_t(x_t)&=\frac{x_t}{t}+\frac{1-t}{t}\nabla\log p_t(x_t)\Longrightarrow \nabla v_t(x_t)=\frac{I}{t}+\frac{1-t}{t}\nabla^2\log p_t(x_t).
\end{align}
With access to $\nabla\log p_t$, $\nabla^2\log p_t$ can be learned with a conditional loss \citep{mengEstimatingHighOrder2021}.
Hence, learning the diagonal entries of $\nabla^2\log p_t(x_t)$ only is equivalent to our vectorized conditional divergence matching up to an independent constant.
We leave the detailed background and discussion in \Cref{app:sec:connection-to-higher-order-score-matching}.

\subsection{SCALLOP}
{Combining the training objectives from \Cref{sec:how-to-train-likelihood-flowmap} with our \emph{Conditional Divergence Matching} objective (\Cref{sec:likelihood-flowmap-conditional-divergence-matching}) and the vectorization technique (\Cref{sec:vectorized-divergence-matching}), we propose \scalloptt, a scalable and efficient pipeline for training a likelihood flow map. The complete training objective is written in \Cref{eq:scallop-loss}.}
% Combining the aforementioned techniques,  \omar{``aforementioned techniques" is unclear given that we speak of many things previously. How about: ``Combining the training objectives from Section 3.2 with our new likelihood distillation objective from Section 3.4 is what we call \scalloptt, a scalable and efficient pipeline for training a likelihood flow map. The complete training objective is written in Eq 22.}. 
Again, we consider a distillation setting following~\cite{ai2026f2d2}, where a pretrained velocity field $v_t^\phi\approx v_t$ is available. When no ambiguity arises, we write $v_t$ for $v_t^\phi$ for notational simplicity. We focus on modeling molecular systems, where specific symmetries must be respected.
% , where specific  rotational and translational symmetries should be accounted
Related works are discussed in~\Cref{sec:introduction} and expanded on in \Cref{app:sec:related_work}.

\paragraph{Parameterization.}
We use the vectorized objective, CDM-v, to learn the log-density-change part of the likelihood flow map. Internally, we parameterize the vectorized map $\Psi_\theta^{\mathrm{vec}}(x_t,t,s)=[x_t;\mathbf{0}_d]+(s-t)f_\theta(x_t,t,s)$, where
$f_\theta(x_t, t, s)=[u_\theta(x_t, t, s);D_\theta(x_t, t, s)]$,
$u_\theta,D_\theta:\mathbb{R}^d\times[0,1]^2\rightarrow\mathbb{R}^{d}$, 
$x_t+(s-t)u_\theta(x_t, t, s)$ gives the generated samples, and $(s-t)D_\theta(x_t, t, s)$ gives the vectorized log-density change. Summing its components recovers the scalar log-density change in $\Psi_\theta$. We employ shared encoding layers for feature extraction, while specifying two readout heads for the samples and log-density-change respectively, analogous to \citet{ai2026f2d2}. As a byproduct of the vectorization trick, we employ the same architecture for these two readout heads. In particular, we follow \citet{rehman2025falconfewstepaccuratelikelihoods} to use the standard Diffusion Transformer block \citep[DiT,][]{Peebles_2023_ICCV} with an additional time embedding head to adapt to the flow map setting. We use 6 shared DiT blocks, and 3 DiT blocks for each of the sample readout head and log-density-change readout head.
% In particular, we modify the standard Diffusion Transformer block \citep[DiT][]{Peebles_2023_ICCV} in two ways to fit the likelihood flow map for molecular systems: (1) we add an additional time embedding head as \cite{rehman2025falconfewstepaccuratelikelihoods} to adapt to the flow map setting; (2) we enforce permutational equivariance by dropping the standard positional embedding in DiT, as well as incorporate the topology of the molecule in the attention mechanism, while \cite{rehman2025falconfewstepaccuratelikelihoods} employed a standard DiT which is not equivariant to permutations and therefore might be overfitting. As the log-density-change is obtained by summing over the permutational-equivariant vector, it is invariant to permutation by design.
% We provide more details on the architectural design in \Cref{app}.
% \tony{I'm not sure how to organize the word here. I'm afraid this will be attacked by FALCON's authors... but it is true that they didn't have perm-equiv as I can't find anything atbout this in their paper.}

\paragraph{Enforcing Soft Symmetries for Molecular Systems.} The generated samples should be equivariant to rotation and translation of the input, namely SE(3)-equivariant, and the predicted log-density-change should be SE(3)-invariant. Following \citet{tan2026scalableequilibriumsamplingsequential,rehman2025falconfewstepaccuratelikelihoods}, we subtract the mean of the data to enforce translational equivariance, while employing random rotation to softly impose equivariance of the rotations. The scalar divergence target is rotation-invariant. We sum the
coordinatewise predictions to estimate it and use rotation augmentation
to encourage this symmetry. As explained in \Cref{app:sec:vectorized-estimators}, the Jacobian
diagonal is not generally an ordinary rotation-equivariant vector, and
augmentation alone does not guarantee exact symmetry.

\paragraph{Learning Objective and Stabilizing Training.} Following \citet{boffi2025buildconsistencymodellearning,ai2026f2d2}, we use learned weights $\exp(-w_\psi(t,s))$ for self-distillation and $\exp(-w_\psi(t,t))$ for flow matching, where $\psi$ denotes a shallow scalar-valued neural network. However, for the divergence matching part, the objective has high variance when $t\rightarrow1$ as studied by \citet{song2021scorebasedgenerativemodelingstochastic,yu2025density} as well as a singularity at $t=1$. To stabilize training, we uniformly sample $t$ from $[0,1-\mathrm{eps}]$ with $\mathrm{eps}=10^{-5}$, as well as employ a weighting $\exp(-w_\psi(t,t))(1-t)$ for the divergence matching part. For the self-distillation part, $(s,t)$ is sampled uniformly on the upper triangle $t<s$ of the square $[0,1-\mathrm{eps}]^2$. Overall, the practical training objective reads
% \begin{multline}
%     \mathcal{L}_\mathrm{\scallop}(\theta)=\mathbb{E}_{t} \Bigl[e^{-w_\psi(t, t)}\bigl(\mathcal{L}_{v}(\theta;t)+(1-t)\mathcal{L}_{\mathrm{CDM\text{-}v}}(\theta;t)\bigr)+w_\psi(t, t)\Bigr]\\+ \mathbb{E}_{s, t}\Bigl[ e^{-w_\psi(t, s)}\mathcal{L}_\mathrm{SD}(\theta;t, s)+w_\psi(t, s)\Bigr].\label{eq:scallop-loss}
% \end{multline}
\begin{multline}
    \mathcal{L}_\mathrm{\scallop}(\theta)=\int_0^1 \Bigl[e^{-w_\psi(t, t)}\bigl(\mathcal{L}_{v}(\theta;t)+(1-t)\mathcal{L}_{\mathrm{CDM\text{-}v}}(\theta;t)\bigr)+w_\psi(t, t)\Bigr]\dd t\\+ \int_0^1\int_0^s\Bigl[ e^{-w_\psi(t, s)}\mathcal{L}_\mathrm{SD}(\theta;t, s)+w_\psi(t, s)\Bigr]\dd t\dd s.\label{eq:scallop-loss}
\end{multline}

\paragraph{Fast Generation and Likelihood Evaluation.} Given a discrete time schedule $0=t_0<...<t_K=1$, inference in \scalloptt is analogous to that in a standard flow map, with a small amount of additional overhead due to calling the log-density-change readout head. In particular, we start by sampling $x_0\sim p_0$ and evaluating $\log p_0(x_0)$. We then iteratively update $x_{t_{i+1}}=x_{t_i}+(t_{i+1}-t_i)u_\theta(x_{t_i}, t_i, t_{i+1})$ and $\log p_{t_{i+1}}(x_{t_{i+1}})=\log p_{t_{i}}(x_{t_{i}})+(t_{i+1}-t_i)\mathbf{1}_d\cdot D_\theta(x_{t_i}, t_i, t_{i+1})$; see \Cref{alg:scallop-inference} for details. It is notable that, unlike methods with approximate invertibility \citep{draxler2024freeformflowsmakearchitecture, rehman2025falconfewstepaccuratelikelihoods}, \scalloptt does not require calculating or estimating the log-determinant of the Jacobian of the generator for likelihood evaluation, but instead requires a single forward pass of the neural network.
Overall, the few-step and high-order-free nature of \scalloptt enables efficient sample generation alongside density evaluation.

\section{Experiments}
\label{sec:experiments}

\newcommand{\err}[1]{\color{gray}{\scriptscriptstyle \pm #1}}
\newcommand{\best}[1]{\textbf{#1}}
\newcommand{\second}[1]{\underline{#1}}
% Add this to your preamble (before \begin{document}):
% \usepackage[export]{adjustbox} 
\begin{table*}[t]
    \centering
    \caption{Results on ALA-$N$ systems, with $N=2,3,4,6$. 
    Results are obtained from three random seeds with $2\times10^5$ samples, reported as \texttt{mean} $\pm$ \texttt{std}.
    The best results are in \textbf{bold}, and the second-best results are \underline{underlined}. Values of SBG-IS, FALCON, and FALCON-A are from \citet{rehman2025falconfewstepaccuratelikelihoods}.}
    \scriptsize
    \label{tab:main_results}
    
    % Adjusted to 0.17 to give the table slightly more room
    \begin{minipage}[t]{0.2\textwidth}
        \centering
        \vspace{0.8cm} % CRITICAL: Locks the top baseline with the right minipage
        \includegraphics[width=0.9\linewidth]{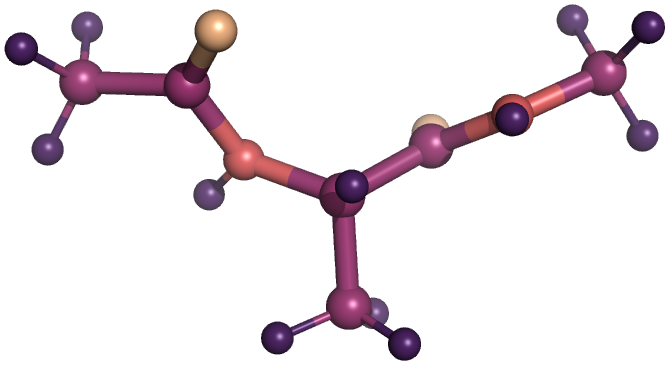}\\[0.2cm]
        \includegraphics[width=0.9\linewidth]{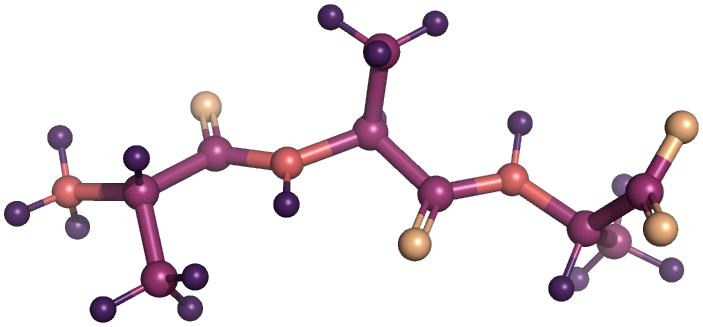}\\[0.2cm]
        \includegraphics[width=0.9\linewidth]{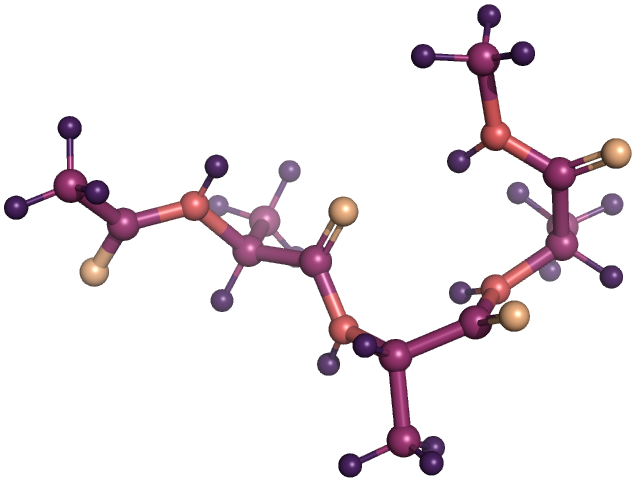}\\[0.2cm]
        \includegraphics[width=0.9\linewidth]{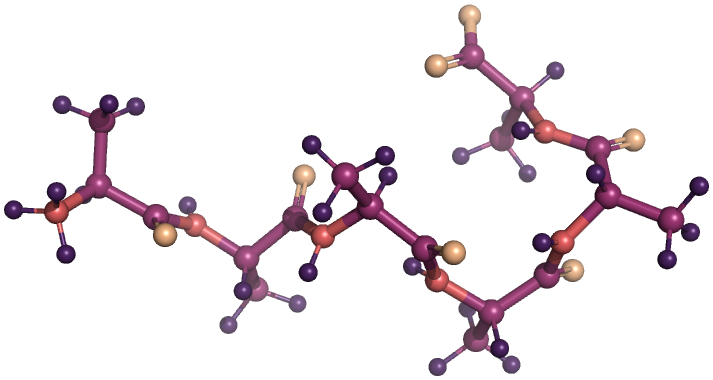}
    \end{minipage}%
    \begin{minipage}[t]{0.75\textwidth}
        \vspace{0pt} % CRITICAL: Locks the top baseline with the left minipage
        \renewcommand{\arraystretch}{1.6} 
        
        % Fit either manuscript width by distributing the remaining column space.
        \setlength{\tabcolsep}{2pt} 
        
        % Increased to 7 columns total
        \begin{tabular*}{\linewidth}{@{\extracolsep{\fill}}c l c c c c c@{}} 
            \cmidrule[0.8pt]{1-7}
            % & & & & & & \\
            \textbf{} & & {\footnotesize SBG-IS} & {\footnotesize FALCON} & {\footnotesize FALCON-A} & {\footnotesize F2D2} & {\footnotesize \textbf{\scallop}} \\
            \cmidrule[0.5pt]{1-7}
            
            % --- ALA Dipeptide ---
            \multirow{4}{*}{\begin{tabular}[c]{@{}c@{}}\textbf{ALA-2} \\ ($d=66$)\end{tabular}}
            & ESS $\uparrow$   & $0.030 \err{0.012}$ & $0.067 \err{0.013}$ & $\best{0.097} \err{0.007}$ & $ 0.066\err{0.001}$ & $\second{0.088}\err{0.001}$ \\
            & $\mathcal{E}$-$\mathcal{W}_2$ $\downarrow$   & $0.873 \err{0.338}$ & $\best{0.225} \err{0.104}$ & $\underline{0.512} \err{0.038}$ & $ 0.958\err{0.114}$ & $0.555\err{0.116}$ \\
            & $\mathbb{T}$-$\mathcal{W}_2$ $\downarrow$   & $0.439 \err{0.129}$ & $0.402 \err{0.021}$ & $\best{0.180} \err{0.005}$ & $ 0.383\err{0.022}$ & $ \second{0.310}\err{0.027}$ \\
            & NFE  & 1 & 4 & 4 & 4 & 4\\
            \cmidrule{1-7}
            
            % --- ALA Tripeptide ---
            \multirow{4}{*}{\begin{tabular}[c]{@{}c@{}}\textbf{ALA-3} \\ ($d=99$)\end{tabular}} 
            & ESS $\uparrow$   & $0.052 \err{0.013}$ & $0.077 \err{0.004}$ & $\best{0.104} \err{0.004}$ & $0.079 \err{0.001}$ & $\second{0.088}\err{0.001}$ \\
            & $\mathcal{E}$-$\mathcal{W}_2$ $\downarrow$   & $\second{0.758} \err{0.506}$ & $\best{0.544} \err{0.013}$ & $1.385 \err{0.182}$ & $2.155 \err{0.033}$ & $1.596\err{0.028}$ \\
            & $\mathbb{T}$-$\mathcal{W}_2$ $\downarrow$   & $0.502 \err{0.016}$ & $0.452 \err{0.011}$ & $0.343 \err{0.004}$ & $\underline{0.219} \err{0.001}$ & $ \best{0.212}\err{0.001}$ \\
            & NFE  & 1 & 8 & 8 & 8 & 8\\
            \cmidrule{1-7}
    
            % --- ALA Tetrapeptide ---
            \multirow{4}{*}{\begin{tabular}[c]{@{}c@{}}\textbf{ALA-4} \\ ($d=126$)\end{tabular}}
            & ESS $\uparrow$   & $0.046 \err{0.014}$ & $0.055 \err{0.003}$ & $\best{0.094} \err{0.007}$ & $0.053 \err{0.001}$ & $\second{0.067} \err{0.001}$ \\
            & $\mathcal{E}$-$\mathcal{W}_2$ $\downarrow$   & $\second{1.068} \err{0.495}$ & $\best{0.686} \err{0.047}$ & $2.929 \err{0.068}$ & $4.209 \err{0.066}$ & $1.545\err{0.033}$ \\
            & $\mathbb{T}$-$\mathcal{W}_2$ $\downarrow$   & $\second{0.969} \err{0.067}$ & $\best{0.858}  \err{0.077}$ & $1.094  \err{0.034}$ & $1.435 \err{0.023}$ & $1.399 \err{0.039}$ \\
            & NFE  & 1 & 8 & 8 & 8 & 8\\
            \cmidrule{1-7}
            
            % --- ALA Hexapeptide ---
            \multirow{4}{*}{\begin{tabular}[c]{@{}c@{}}\textbf{ALA-6} \\ ($d=189$)\end{tabular}}
            & ESS $\uparrow$   & $0.034 \err{0.015}$ & $\second{0.060}  \err{0.017}$ & $\best{0.077} \err{0.007}$ & $0.033\err{0.001}$ & $0.042\err{0.001}$ \\
            & $\mathcal{E}$-$\mathcal{W}_2$ $\downarrow$   & $\second{1.021} \err{0.239}$ & $\best{0.892}   \err{0.311}$ & $1.211 \err{0.105}$ & $10.140 \err{0.065}$ & $5.845 \err{0.101}$ \\
            & $\mathbb{T}$-$\mathcal{W}_2$ $\downarrow$   & $1.431 \err{0.085}$ & $1.256\err{0.132}$ & $1.163  \err{0.112}$ & $\second{1.143} \err{0.039}$ & $\best{1.120} \err{0.014}$ \\
            & NFE  & 1 & 16 & 16 & 16 & 16\\
            \cmidrule[0.8pt]{1-7}
        \end{tabular*}
    \end{minipage}
\end{table*}

\paragraph{Instantiation.} We optimize \Cref{eq:scallop-loss} to train \texttt{SCALLOP}. In particular, we use the LSD loss (\Cref{eq:lsd-loss}) as the self-distillation loss, and leave the other alternatives for future work.

\paragraph{Molecular Systems.} We conduct experiments on several alanine systems of varying lengths, including Alanine Dipeptide (ALA-2), Alanine Tripeptide (ALA-3), Alanine Tetrapeptide (ALA-4), and Alanine Hexapeptide (ALA-6). Data are obtained from simulating Molecular Dynamics in implicit solvent, with the AMBER-14 force field. Details of the simulation are provided in \Cref{app:sec:molecule-exp-details}.
Following \citet{tan2026scalableequilibriumsamplingsequential,rehman2025falconfewstepaccuratelikelihoods}, models are trained on the biased dataset. The goal is to correct the model using self-normalized importance sampling, enabling it to sample equilibrium conformations corresponding to the Boltzmann distribution.

\paragraph{Baselines.} We benchmark \texttt{SCALLOP} against the most promising discrete- and continuous-flow baselines identified by \citet{rehman2025falconfewstepaccuratelikelihoods}: SBG \citep{tan2026scalableequilibriumsamplingsequential} and FALCON \citep{rehman2025falconfewstepaccuratelikelihoods}. Specifically, to better focus on likelihood evaluation, we compare against the SBG with standard SNIS (denoted as SBG-IS). Although an SMC-based variant (SBG-SMC) exists and can be seamlessly integrated with both FALCON and \texttt{SCALLOP},
% (as detailed in~\Cref{app:sec:scallop-smc})
we exclude it from our current benchmark to ensure a fair comparison with SNIS, leaving the exploration of \texttt{SCALLOP}-SMC for future work.

\paragraph{Metrics.} We evaluate \texttt{SCALLOP} and the baselines using standard SNIS. To measure sampling efficiency, we report the Effective Sample Size (ESS). To quantify the effectiveness of the reweighted samples, we compute the 2-Wasserstein distance for both the distributions of energy ($\ewas$) and torsional angle ($\twas$). Specifically, $\ewas$ is sensitive to extreme energy values—often caused by atoms being too close to one another—and thus reflects local structural refinements. Conversely, because torsional angles act as the primary collective variables capturing the most information in alanine systems, $\twas$ measures global structural quality. Further details are provided in Appendix~\ref{app:sec:molecule-exp-details}.

% \subsection{\texttt{SCALLOP} achieves comparable results against baselines}
% The performance of \texttt{SCALLOP} and the other baselines are reported in \Cref{tab:main_results}.

\subsection{SCALLOP as a Boltzmann Generator}
\paragraph{Comparable Performance against SOTA.}
The performance of \texttt{SCALLOP} and the other baselines is reported in \Cref{tab:main_results}. Across different systems, \texttt{SCALLOP} consistently outperforms F2D2 across all considered metrics, while achieving highly competitive results against the state-of-the-art FALCON and FALCON-A. Specifically, \texttt{SCALLOP} excels at capturing the global conformational space, achieving the best $\twas$ on the ALA-3 and ALA-6 systems. Furthermore, it consistently yields a high ESS, indicating substantial overlap with the target distribution and a strong resistance to mode collapse—a conclusion further corroborated by its strong $\twas$ performance.

While FALCON generally achieves a lower $\ewas$, \texttt{SCALLOP} drastically improves upon F2D2, nearly halving the $\ewas$ error in larger systems like ALA-4 and ALA-6. More importantly, higher $\ewas$ values do not undermine the overall validity of the drawn samples, as the Wasserstein-2 distance is notoriously sensitive to energy outliers caused by microscopic structural deviations. The low $\twas$ demonstrates that \texttt{SCALLOP} correctly captures the global torsional angle structure, confirming that the macroscopic metastable states and their statistical weights are accurately recovered. Because the underlying conformational macrostates are correct, the remaining local structural noise that inflates the $\ewas$ can be efficiently mitigated by applying a computationally cheap local relaxation to align the local structures.

\begin{wrapfigure}[9]{r}{0.5\textwidth}
    \vspace{-1.4\baselineskip}
    \centering
    \includegraphics[width=\linewidth]{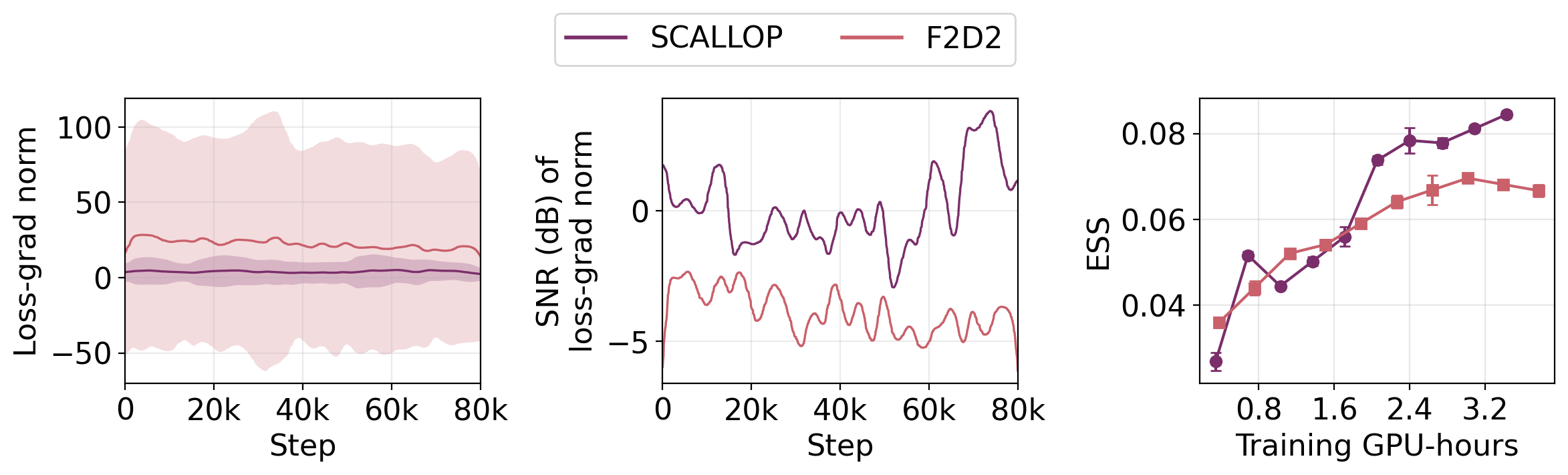}
    \caption{Training behavior comparison between \texttt{SCALLOP} and F2D2 on ALA-2.}
    \label{fig:training_analysis}
\end{wrapfigure}
% \begin{figure}[h]
%     \centering
%     \includegraphics[width=\linewidth]{figs/training_analysis_aldp/training_convergence_row.png}
%     \caption{Training behavior comparison between \texttt{SCALLOP} and F2D2 on ALA-2.}
%     \label{fig:training_analysis}
% \end{figure}
\paragraph{Improved Training Stability against F2D2.} To investigate the behavioral differences during training between \texttt{SCALLOP} and F2D2, we present an analysis on ALA-2. The left panel of \Cref{fig:training_analysis} illustrates the loss-gradient norms throughout training, with variance estimated via a sliding window. To compare the effective training signals, we compute the Signal-to-Noise Ratio (SNR) of the loss gradients, defined as $10\log_{10}(|\mathrm{mean}|/\mathrm{std})$, over the training process, depicted in the middle panel. The consistently higher log SNR exhibited by \texttt{SCALLOP} highlights its enhanced training stability. Furthermore, we report the Effective Sample Size (ESS) as a function of training GPU-hours, demonstrating that \texttt{SCALLOP} ultimately converges to a superior optimum.

\begin{wrapfigure}[15]{r}{0.5\textwidth}
    \vspace{-\baselineskip}
    \centering
    \includegraphics[width=\linewidth]{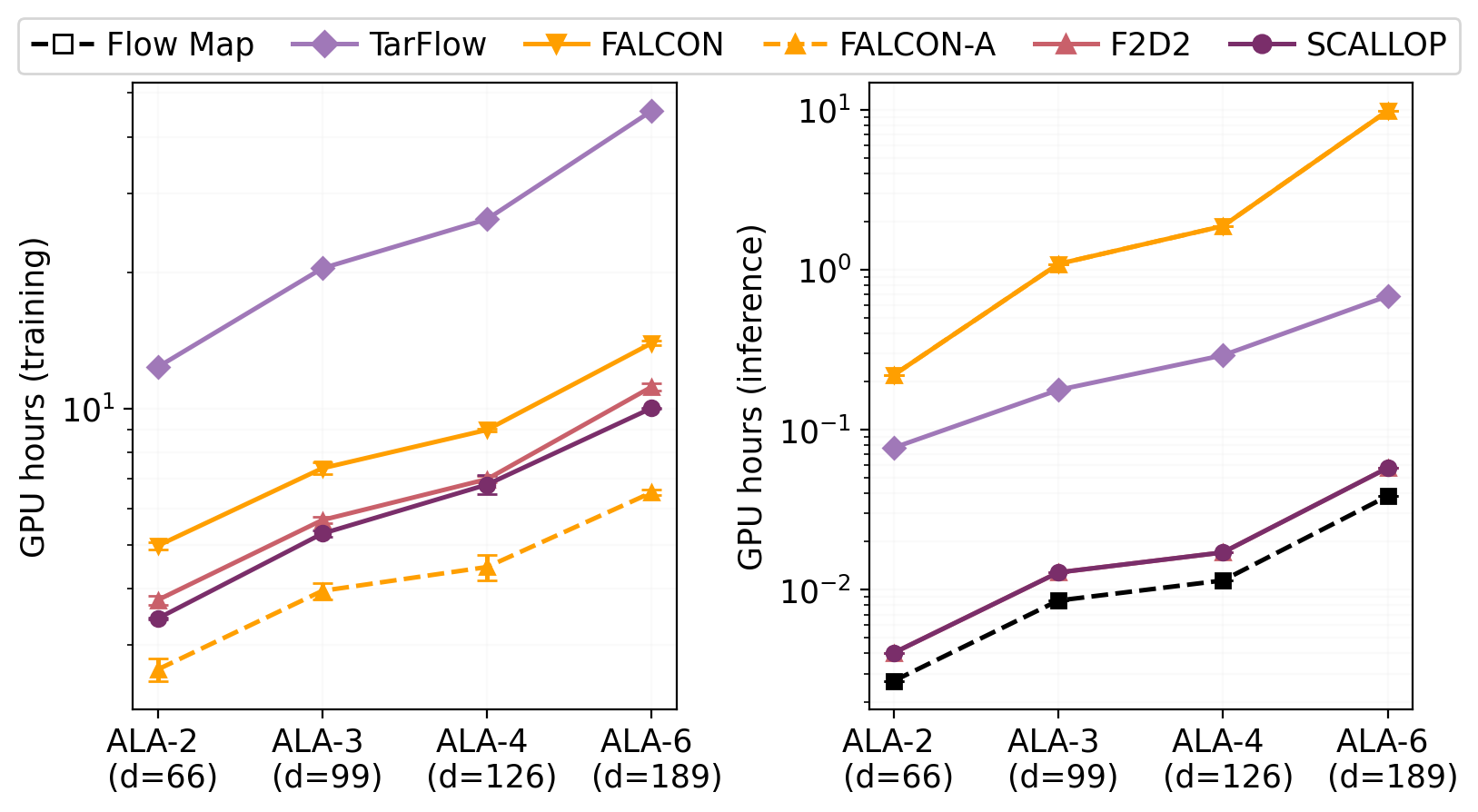}
    \caption{\textbf{(Left)} Training time comparison; \textbf{(Right)} Inference time comparison, for generating $2\times10^5$ samples with their densities. 
    }
    % \vspace{-20em}
    \label{fig:time-comparison}
\end{wrapfigure}
% \begin{figure}[h]
%     \centering
%     \includegraphics[width=\linewidth]{figs/joint_time_benchmark_plot.png}
%     \caption{\textbf{(Left)} Training time comparison; \textbf{(Right)} Inference time comparison, for generating $2\times10^5$ samples with their densities. 
%     }
%     % \vspace{-20em}
%     \label{fig:time-comparison}
% \end{figure}
\paragraph{Efficient Training and Likelihood Evaluation.} We evaluate training and inference times on a single \texttt{NVIDIA RTX PRO 6000} GPU, detailed in \Cref{fig:time-comparison}. The left panel reports training times evaluated at a batch size of $1024$ with the same number of network updates. By avoiding the JVP used in Hutchinson's trace estimator, \texttt{SCALLOP} achieves a roughly $10\%$ speedup compared to F2D2.

The right panel reports the time required to generate $2\times10^5$ samples alongside their densities. Due to the memory-intensive full Jacobian computations, FALCON requires heavily restricted batch sizes $(1000,1000,500,100)$ for ALA-$(2,3,4,6)$, respectively, leading to inference costs that scale steeply with dimensionality. Conversely, \texttt{SCALLOP} and F2D2 estimate likelihoods via appended DiT blocks and possess nearly identical inference costs. This approach bypasses the Jacobian bottleneck, supports large batch sizes ($10^4$ for all systems), and adds modest overhead over a baseline Flow Map that lacks likelihood evaluation.

\begin{table}[t]
\centering
\caption{Evaluation on CelebA-64. The reference flow-matching model achieves BPD $1.75$ with 1024 Euler steps and FID $2.48$ with 200 Euler steps. $|\Delta|$ denotes the absolute deviation from reference BPD and Err denotes the mean absolute error against the reference per-sample BPD estimates. All models are trained with the LSD loss (\Cref{eq:lsd-loss}).}
\label{tab:bpd-fid}
\scriptsize
\setlength{\tabcolsep}{4pt}
\resizebox{\linewidth}{!}{
\begin{tabular}{lcccccccccccc}
\toprule
\multirow{2}{*}{Steps}
& \multicolumn{4}{c}{\footnotesize Flow map}
& \multicolumn{4}{c}{\footnotesize F2D2}
& \multicolumn{4}{c}{\footnotesize \textbf{\scallop}} \\
\cmidrule(lr){2-5} \cmidrule(lr){6-9} \cmidrule(lr){10-13}
& NLL
& $|\Delta|$ $\downarrow$
& Err $\downarrow$
& FID $\downarrow$
& NLL
& $|\Delta|$ $\downarrow$
& Err $\downarrow$
& FID $\downarrow$
& NLL
& $|\Delta|$ $\downarrow$
& Err $\downarrow$
& FID $\downarrow$ \\
\midrule
$1$
& $-69.83$ & $71.58$ & $71.64$ & $12.96$
& $1.64$ & $\best{0.11}$ & $0.53$ & $6.94$
& $1.60$ & $0.15$ & $\best{0.17}$ & $\best{6.91}$ \\
$2$
& $-32.72$ & $34.47$ & $34.46$ & $6.32$
& $1.73$ & $\best{0.02}$ & $0.33$ & $3.86$
& $1.69$ & $0.06$ & $\best{0.12}$ & $\best{3.81}$ \\
$4$
& $-14.89$ & $16.64$ & $16.59$ & $4.04$
& $1.75$ & $\best{0.00}$ & $0.24$ & $2.75$
& $1.75$ & $\best{0.00}$ & $\best{0.12}$ & $\best{2.70}$ \\
$8$
& $-6.78$ & $8.53$ & $8.45$ & $3.33$
& $1.64$ & $0.11$ & $0.18$ & $2.41$
& $1.68$ & $\best{0.07}$ & $\best{0.10}$ & $\best{2.37}$ \\
\bottomrule
\end{tabular}
}
\end{table}
\subsection{Effectiveness and Scalability beyond Molecules}

To test whether the benefits of \scalloptt generalize beyond molecular data, we evaluate it on unconditional image generation using CelebA-64 \citep{liu2015faceattributes}. Following the F2D2 evaluation protocol, we report FID for sample quality and BPD, i.e., negative log-likelihood measured in bits per dimension, for likelihood estimation. We further report the absolute deviation from the 1024-step teacher BPD and the mean absolute per-sample BPD error with respect to the same teacher reference. All metrics are evaluated using 1, 2, 4, and 8 Euler steps for both sampling and likelihood computation. As shown in~\Cref{tab:bpd-fid}, \scalloptt achieves BPD and FID comparable to F2D2 across all steps, while consistently reducing the per-sample BPD error. \Cref{fig:loss-grad-comparison} further shows that \scalloptt has more stable training dynamics, with lower variance in both training loss and gradient norm compared with F2D2 over 60k training steps. Moreover, its loss computation takes only $83.1\%$ of the wall-clock time of F2D2, demonstrating improved computational efficiency. Additional details are provided in~\Cref{app:sec:img-exp-details}.

% \begin{table}[t]
% \centering
% \small
% \caption{Evaluation on CelebA-64. The reference flow-matching model achieves BPD $1.75$ with 1024 Euler steps and FID $2.48$ with 200 Euler steps. $|\Delta|$ denotes the absolute deviation from reference BPD and Err denotes the mean absolute error against the reference per-sample BPD estimates.}
% \vspace{0.5em}
% \label{tab:bpd-fid}
% \setlength{\tabcolsep}{4pt}
% \begin{tabular}{lcccccccc}
% \toprule
% \multirow{2}{*}{Steps}
% & \multicolumn{4}{c}{F2D2}
% & \multicolumn{4}{c}{\textbf{\scallop}} \\
% \cmidrule(lr){2-5} \cmidrule(lr){6-9}
% & NLL
% & $|\Delta|$ $\downarrow$
% & Err $\downarrow$
% & FID $\downarrow$
% & NLL
% & $|\Delta|$ $\downarrow$
% & Err $\downarrow$
% & FID $\downarrow$ \\
% \midrule
% 8
% & 1.64 & 0.11 & 0.18 & 2.41
% & 1.68 & \best{0.07} & \best{0.10} & \best{2.37} \\
% 4
% & 1.75 & \best{0.00} & 0.24 & 2.75
% & 1.75 & \best{0.00} & \best{0.12} & \best{2.70} \\
% 2
% & 1.73 & \best{0.02} & 0.33 & 3.86
% & 1.69 & 0.06 & \best{0.12} & \best{3.81} \\
% 1
% & 1.64 & \best{0.11} & 0.53 & 6.94
% & 1.60 & 0.15 & \best{0.17} & \best{6.91} \\
% \bottomrule
% \end{tabular}
% \vspace{-2em}
% \end{table}

\section{Discussion}

In conclusion, we proposed \textsc{SCAlable LikeLihood distillation of flOw maPs} (\scallop), an efficient procedure for training few-step models to jointly generate samples and their densities. This capability allows \scalloptt to serve as a Boltzmann generator for sampling from unnormalized densities, where it is competitive with the state-of-the-art while being significantly faster to evaluate at inference time. 

\paragraph{Limitations.} Despite its empirical success, \scalloptt has two main limitations. First, exact equivalence between conditional and marginal divergence matching assumes a perfectly learned velocity field. Under the regularity and density-ratio assumptions in \Cref{app:sec:error-bounds}, \Cref{prop:finite-sample-likelihood-error} bounds log-density error along the teacher's continuous-time trajectories, separating teacher error from error in fitting the instantaneous likelihood head. \Cref{thm:IS} bounds the difference in observable expectations between the interpolation and teacher marginals. \Cref{thm:finite-sample-IS} bounds the bias of a population importance-weighted expectation targeting $p_t$ when the proposal density is estimated by a fitted head. These guarantees do not cover finite-sample self-normalization or errors introduced by few-step distillation.
Second, while \scalloptt can generate samples and their densities in $1$ step, it practically uses ${\sim}10$ steps for high-dimensional tasks. 

\newcommand{\RKOYack}{RKOY acknowledges the UK Engineering and Physical Sciences Research Council (EPSRC) grant EP/L016516/1 for the University of Cambridge Centre for Doctoral Training, the Cambridge Centre for Analysis.}
\newcommand{\HYack}{\space HY acknowledges the research environment provided by ELLIS Institute Finland and was supported by the Research Council of Finland Flagship programme: Finnish Center for Artificial Intelligence FCAI.}
\newcommand{\XAack}{}
\newcommand{\YHack}{}
\newcommand{\NMBack}{}
\newcommand{\PRack}{\space PR acknowledges the support of AFRL and DARPA via FA8750-23-2-1015, ONR via N00014-23-1-2368, and NSF via IIS-1955532.}
\newcommand{\JMHLack}{JMHL acknowledges support from EPSRC funding under grant EP/Y028805/1. JMHL also acknowledges support from a Turing AI Fellowship under grant EP/V023756/1.}
\newcommand{\MSack}{}
\newcommand{\BKMack}{}
\newcommand{\OCack}{}
\newcommand{\Computeack}{This work was performed using HPC resources from GENCI–IDRIS (AD011015234R1).
This project acknowledges the resources provided by the Cambridge Service for Data-Driven Discovery (CSD3) operated by the University of Cambridge Research Computing Service (\href{https://www.csd3.cam.ac.uk}{www.csd3.cam.ac.uk}), provided by Dell EMC and Intel using Tier-2 funding from the Engineering and Physical Sciences Research Council (capital grant EP/T022159/1), and DiRAC funding from the Science and Technology Facilities Council (\href{https://www.dirac.ac.uk}{www.dirac.ac.uk}).}

\providecommand{\scallopackheading}{\section*{Acknowledgments}}
\scallopackheading
\RKOYack\space
\HYack\space
\XAack
\YHack
\NMBack
\PRack\space
\JMHLack\space
\MSack
\BKMack
\OCack

\Computeack

% Bibliography
\bibliographystyle{abbrvnat}
\bibliography{references}

\begin{thebibliography}{58}
\providecommand{\natexlab}[1]{#1}
\providecommand{\url}[1]{\texttt{#1}}
\expandafter\ifx\csname urlstyle\endcsname\relax
  \providecommand{\doi}[1]{doi: #1}\else
  \providecommand{\doi}{doi: \begingroup \urlstyle{rm}\Url}\fi

\bibitem[Aggarwal et~al.(2025)Aggarwal, Chen, Boffi, and
  Koes]{aggarwal2025boltzncelearninglikelihoodsboltzmann}
R.~Aggarwal, J.~Chen, N.~M. Boffi, and D.~R. Koes.
\newblock Boltznce: Learning likelihoods for boltzmann generation with
  stochastic interpolants and noise contrastive estimation, 2025.
\newblock URL \url{https://arxiv.org/abs/2507.00846}.

\bibitem[Ai et~al.(2026)Ai, He, Gu, Salakhutdinov, Kolter, Boffi, and
  Simchowitz]{ai2026f2d2}
X.~Ai, Y.~He, A.~Gu, R.~Salakhutdinov, J.~Kolter, N.~Boffi, and M.~Simchowitz.
\newblock Joint distillation for fast likelihood evaluation and sampling in
  flow-based models.
\newblock In \emph{International Conference on Learning Representations}, 2026.

\bibitem[Albergo et~al.(2025)Albergo, Boffi, and
  Vanden-Eijnden]{albergo2025interpolant}
M.~Albergo, N.~M. Boffi, and E.~Vanden-Eijnden.
\newblock Stochastic interpolants: A unifying framework for flows and
  diffusions.
\newblock \emph{Journal of Machine Learning Research}, 26\penalty0
  (209):\penalty0 1--80, 2025.

\bibitem[Barducci et~al.(2008)Barducci, Bussi, and
  Parrinello]{barducci2008well}
A.~Barducci, G.~Bussi, and M.~Parrinello.
\newblock Well-tempered metadynamics: a smoothly converging and tunable
  free-energy method.
\newblock \emph{Physical review letters}, 100\penalty0 (2):\penalty0 020603,
  2008.

\bibitem[Boffi et~al.(2025)Boffi, Albergo, and
  Vanden-Eijnden]{boffi2025buildconsistencymodellearning}
N.~M. Boffi, M.~S. Albergo, and E.~Vanden-Eijnden.
\newblock How to build a consistency model: Learning flow maps via
  self-distillation, 2025.
\newblock URL \url{https://arxiv.org/abs/2505.18825}.

\bibitem[Boffi et~al.(2026)Boffi, Albergo, and
  Vanden-Eijnden]{boffi2026flowmaps}
N.~M. Boffi, M.~S. Albergo, and E.~Vanden-Eijnden.
\newblock How to build a consistency model: Learning flow maps via
  self-distillation.
\newblock In \emph{Annual Conference on Neural Information Processing Systems},
  2026.

\bibitem[Bortoli et~al.(2024)Bortoli, Hutchinson, Wirnsberger, and
  Doucet]{debortoli2024targetscorematching}
V.~D. Bortoli, M.~Hutchinson, P.~Wirnsberger, and A.~Doucet.
\newblock Target score matching, 2024.
\newblock URL \url{https://arxiv.org/abs/2402.08667}.

\bibitem[Chen et~al.(2019)Chen, Rubanova, Bettencourt, and
  Duvenaud]{chen2019neuralordinarydifferentialequations}
R.~T.~Q. Chen, Y.~Rubanova, J.~Bettencourt, and D.~Duvenaud.
\newblock Neural ordinary differential equations, 2019.
\newblock URL \url{https://arxiv.org/abs/1806.07366}.

\bibitem[Choi et~al.(2021)Choi, Meng, Song, and Ermon]{2021Density}
K.~Choi, C.~Meng, Y.~Song, and S.~Ermon.
\newblock Density ratio estimation via infinitesimal classification.
\newblock \emph{arXiv e-prints}, 2021.

\bibitem[Dibak et~al.(2021)Dibak, Klein, Kr{\"a}mer, and
  No{\'e}]{dibak2021temperature}
M.~Dibak, L.~Klein, A.~Kr{\"a}mer, and F.~No{\'e}.
\newblock Temperature steerable flows and {Boltzmann} generators, 2021.
\newblock URL \url{https://arxiv.org/abs/2108.01590}.

\bibitem[Draxler et~al.(2024)Draxler, Sorrenson, Zimmermann, Rousselot, and
  Köthe]{draxler2024freeformflowsmakearchitecture}
F.~Draxler, P.~Sorrenson, L.~Zimmermann, A.~Rousselot, and U.~Köthe.
\newblock Free-form flows: Make any architecture a normalizing flow, 2024.
\newblock URL \url{https://arxiv.org/abs/2310.16624}.

\bibitem[Frans et~al.(2025)Frans, Hafner, Levine, and
  Abbeel]{frans2025stepdiffusionshortcutmodels}
K.~Frans, D.~Hafner, S.~Levine, and P.~Abbeel.
\newblock One step diffusion via shortcut models, 2025.
\newblock URL \url{https://arxiv.org/abs/2410.12557}.

\bibitem[Gao et~al.(2020)Gao, Nijkamp, Kingma, Xu, Dai, and
  Wu]{gao2020flowcontrastive}
R.~Gao, E.~Nijkamp, D.~P. Kingma, Z.~Xu, A.~M. Dai, and Y.~N. Wu.
\newblock Flow contrastive estimation of energy-based models.
\newblock In \emph{2020 {IEEE}/{CVF} conference on computer vision and pattern
  recognition ({CVPR})}, pages 7515--7525, 2020.
\newblock \doi{10.1109/CVPR42600.2020.00754}.

\bibitem[Geng et~al.(2025)Geng, Deng, Bai, Kolter, and
  He]{geng2025meanflowsonestepgenerative}
Z.~Geng, M.~Deng, X.~Bai, J.~Z. Kolter, and K.~He.
\newblock Mean flows for one-step generative modeling, 2025.
\newblock URL \url{https://arxiv.org/abs/2505.13447}.

\bibitem[Grathwohl et~al.(2018)Grathwohl, Chen, Bettencourt, Sutskever, and
  Duvenaud]{grathwohl2018ffjordfreeformcontinuousdynamics}
W.~Grathwohl, R.~T.~Q. Chen, J.~Bettencourt, I.~Sutskever, and D.~Duvenaud.
\newblock Ffjord: Free-form continuous dynamics for scalable reversible
  generative models, 2018.
\newblock URL \url{https://arxiv.org/abs/1810.01367}.

\bibitem[Grubm{\"u}ller(1995)]{grubmuller1995predicting}
H.~Grubm{\"u}ller.
\newblock Predicting slow structural transitions in macromolecular systems:
  Conformational flooding.
\newblock \emph{Physical Review E}, 52\penalty0 (3):\penalty0 2893, 1995.

\bibitem[Guth et~al.(2026)Guth, Kadkhodaie, and
  Simoncelli]{guth2026learningnormalizedimagedensities}
F.~Guth, Z.~Kadkhodaie, and E.~P. Simoncelli.
\newblock Learning normalized image densities via dual score matching, 2026.
\newblock URL \url{https://arxiv.org/abs/2506.05310}.

\bibitem[Gutmann and Hyvärinen(2012)]{gutmann2012nce}
M.~U. Gutmann and A.~Hyvärinen.
\newblock Noise-{Contrastive} {Estimation} of {Unnormalized} {Statistical}
  {Models}, with {Applications} to {Natural} {Image} {Statistics}.
\newblock \emph{Journal of Machine Learning Research}, 13\penalty0
  (11):\penalty0 307--361, 2012.
\newblock URL \url{http://jmlr.org/papers/v13/gutmann12a.html}.

\bibitem[He et~al.(2026)He, Hernández-Lobato, Du, and
  Vargas]{he2026rneplugandplaydiffusioninferencetime}
J.~He, J.~M. Hernández-Lobato, Y.~Du, and F.~Vargas.
\newblock Rne: plug-and-play diffusion inference-time control and energy-based
  training, 2026.
\newblock URL \url{https://arxiv.org/abs/2506.05668}.

\bibitem[Ho et~al.(2020)Ho, Jain, and
  Abbeel]{ho2020denoisingdiffusionprobabilisticmodels}
J.~Ho, A.~Jain, and P.~Abbeel.
\newblock Denoising diffusion probabilistic models, 2020.
\newblock URL \url{https://arxiv.org/abs/2006.11239}.

\bibitem[Hutchinson(1990)]{Hutchinson01011990}
M.~Hutchinson.
\newblock A stochastic estimator of the trace of the influence matrix for
  laplacian smoothing splines.
\newblock \emph{Communications in Statistics - Simulation and Computation},
  19\penalty0 (2):\penalty0 433--450, 1990.
\newblock \doi{10.1080/03610919008812866}.
\newblock URL \url{https://doi.org/10.1080/03610919008812866}.

\bibitem[Hyv{{\"a}}rinen(2005)]{JMLR:v6:hyvarinen05a}
A.~Hyv{{\"a}}rinen.
\newblock Estimation of non-normalized statistical models by score matching.
\newblock \emph{Journal of Machine Learning Research}, 6\penalty0
  (24):\penalty0 695--709, 2005.
\newblock URL \url{http://jmlr.org/papers/v6/hyvarinen05a.html}.

\bibitem[Kish(1957)]{kish1957confidence}
L.~Kish.
\newblock Confidence intervals for clustered samples.
\newblock \emph{American Sociological Review}, 22\penalty0 (2):\penalty0
  154--165, 1957.

\bibitem[Klein et~al.(2023)Klein, Kr\"{a}mer, and Noe]{NEURIPS2023_bc827452}
L.~Klein, A.~Kr\"{a}mer, and F.~Noe.
\newblock Equivariant flow matching.
\newblock In A.~Oh, T.~Naumann, A.~Globerson, K.~Saenko, M.~Hardt, and
  S.~Levine, editors, \emph{Advances in Neural Information Processing Systems},
  volume~36, pages 59886--59910. Curran Associates, Inc., 2023.
\newblock URL
  \url{https://proceedings.neurips.cc/paper_files/paper/2023/file/bc827452450356f9f558f4e4568d553b-Paper-Conference.pdf}.

\bibitem[K{\"o}hler et~al.(2020)K{\"o}hler, Klein, and
  Noe]{pmlr-v119-kohler20a}
J.~K{\"o}hler, L.~Klein, and F.~Noe.
\newblock Equivariant flows: Exact likelihood generative learning for symmetric
  densities.
\newblock In H.~D. III and A.~Singh, editors, \emph{Proceedings of the 37th
  International Conference on Machine Learning}, volume 119 of
  \emph{Proceedings of Machine Learning Research}, pages 5361--5370. PMLR,
  13--18 Jul 2020.
\newblock URL \url{https://proceedings.mlr.press/v119/kohler20a.html}.

\bibitem[K{\"o}hler et~al.(2021)K{\"o}hler, Kr{\"a}mer, and
  No{\'e}]{kohler2021smooth}
J.~K{\"o}hler, A.~Kr{\"a}mer, and F.~No{\'e}.
\newblock Smooth normalizing flows.
\newblock In \emph{Advances in Neural Information Processing Systems},
  volume~34, 2021.
\newblock URL
  \url{https://proceedings.neurips.cc/paper_files/paper/2021/hash/167434fa6219316417cd4160c0c5e7d2-Abstract.html}.

\bibitem[Laio and Parrinello(2002)]{laio2002escaping}
A.~Laio and M.~Parrinello.
\newblock Escaping free-energy minima.
\newblock \emph{Proceedings of the national academy of sciences}, 99\penalty0
  (20):\penalty0 12562--12566, 2002.

\bibitem[Lipman et~al.(2023)Lipman, Chen, Ben-Hamu, Nickel, and
  Le]{lipman2023flow}
Y.~Lipman, R.~T.~Q. Chen, H.~Ben-Hamu, M.~Nickel, and M.~Le.
\newblock Flow matching for generative modeling.
\newblock In \emph{International Conference on Learning Representations}, 2023.

\bibitem[Liu(2022)]{liu2022rectifiedflowmarginalpreserving}
Q.~Liu.
\newblock Rectified flow: A marginal preserving approach to optimal transport,
  2022.
\newblock URL \url{https://arxiv.org/abs/2209.14577}.

\bibitem[Liu et~al.(2025)Liu, Du, Deng, and Zhang]{liu2025hutch}
X.~Liu, H.~Du, W.~Deng, and R.~Zhang.
\newblock Optimal stochastic trace estimation in generative modeling.
\newblock In Y.~Li, S.~Mandt, S.~Agrawal, and E.~Khan, editors,
  \emph{Proceedings of the 28th international conference on artificial
  intelligence and statistics}, volume 258 of \emph{Proceedings of machine
  learning research}, pages 4600--4608. PMLR, May 2025.
\newblock URL \url{https://proceedings.mlr.press/v258/liu25k.html}.

\bibitem[Liu et~al.(2015)Liu, Luo, Wang, and Tang]{liu2015faceattributes}
Z.~Liu, P.~Luo, X.~Wang, and X.~Tang.
\newblock Deep learning face attributes in the wild.
\newblock In \emph{Proceedings of International Conference on Computer Vision
  (ICCV)}, December 2015.

\bibitem[Lu and Song(2025)]{lu2025simplifying}
C.~Lu and Y.~Song.
\newblock Simplifying, stabilizing and scaling continuous-time consistency
  models.
\newblock In \emph{The Thirteenth International Conference on Learning
  Representations}, 2025.
\newblock URL \url{https://openreview.net/forum?id=LyJi5ugyJx}.

\bibitem[Meng et~al.(2021)Meng, Song, Li, and
  Ermon]{mengEstimatingHighOrder2021}
C.~Meng, Y.~Song, W.~Li, and S.~Ermon.
\newblock Estimating {High} {Order} {Gradients} of the {Data} {Distribution} by
  {Denoising}.
\newblock In M.~Ranzato, A.~Beygelzimer, Y.~Dauphin, P.~S. Liang, and J.~W.
  Vaughan, editors, \emph{Advances in {Neural} {Information} {Processing}
  {Systems}}, volume~34, pages 25359--25369. Curran Associates, Inc., 2021.

\bibitem[Meyer et~al.(2021)Meyer, Musco, Musco, and
  Woodruff]{meyer2021hutchinson}
R.~A. Meyer, C.~Musco, C.~Musco, and D.~P. Woodruff.
\newblock Hutch++: Optimal stochastic trace estimation.
\newblock In \emph{2021 Symposium on Simplicity in Algorithms (SOSA)}, pages
  142--155. Society for Industrial and Applied Mathematics, 2021.
\newblock \doi{10.1137/1.9781611976496.16}.

\bibitem[Nam et~al.(2025)Nam, M{\'a}t{\'e}, Toshev, Kaniselvan,
  G{\'o}mez-Bombarelli, Chen, Wood, Liu, and Miller]{nam2025enhancing}
J.~Nam, B.~M{\'a}t{\'e}, A.~P. Toshev, M.~Kaniselvan, R.~G{\'o}mez-Bombarelli,
  R.~T. Chen, B.~Wood, G.-H. Liu, and B.~K. Miller.
\newblock Enhancing diffusion-based sampling with molecular collective
  variables.
\newblock \emph{arXiv preprint arXiv:2510.11923}, 2025.

\bibitem[Noé et~al.(2019)Noé, Olsson, Köhler, and
  Wu]{noé2019boltzmanngeneratorssampling}
F.~Noé, S.~Olsson, J.~Köhler, and H.~Wu.
\newblock Boltzmann generators: {Sampling} equilibrium states of many-body
  systems with deep learning.
\newblock \emph{Science}, 365\penalty0 (6457):\penalty0 eaaw1147, 2019.
\newblock \doi{10.1126/science.aaw1147}.
\newblock tex.eprint: https://www.science.org/doi/pdf/10.1126/science.aaw1147.

\bibitem[Ou et~al.(2025)Ou, Zhang, Zhang, Xiao, Li, and
  Barber]{ou2025improvingprobabilisticdiffusionmodels}
Z.~Ou, M.~Zhang, A.~Zhang, T.~Z. Xiao, Y.~Li, and D.~Barber.
\newblock Improving probabilistic diffusion models with optimal diagonal
  covariance matching, 2025.
\newblock URL \url{https://arxiv.org/abs/2406.10808}.

\bibitem[OuYang et~al.(2026{\natexlab{a}})OuYang, Grenioux, and
  Hernández-Lobato]{ouyang2026diffusiveclassificationlosslearning}
R.~OuYang, L.~Grenioux, and J.~M. Hernández-Lobato.
\newblock A diffusive classification loss for learning energy-based generative
  models, 2026{\natexlab{a}}.
\newblock URL \url{https://arxiv.org/abs/2601.21025}.

\bibitem[OuYang et~al.(2026{\natexlab{b}})OuYang, Qiang, and
  Hernández-Lobato]{ouyang2026bnemboltzmannsamplerbased}
R.~OuYang, B.~Qiang, and J.~M. Hernández-Lobato.
\newblock Bnem: A boltzmann sampler based on bootstrapped noised energy
  matching, 2026{\natexlab{b}}.
\newblock URL \url{https://arxiv.org/abs/2409.09787}.

\bibitem[Papamakarios et~al.(2021)Papamakarios, Nalisnick, Rezende, Mohamed,
  and Lakshminarayanan]{papamakarios2021normalizing}
G.~Papamakarios, E.~Nalisnick, D.~J. Rezende, S.~Mohamed, and
  B.~Lakshminarayanan.
\newblock Normalizing flows for probabilistic modeling and inference.
\newblock \emph{Journal of Machine Learning Research}, 22\penalty0
  (57):\penalty0 1--64, 2021.

\bibitem[Peebles and Xie(2023)]{Peebles_2023_ICCV}
W.~Peebles and S.~Xie.
\newblock Scalable diffusion models with transformers.
\newblock In \emph{Proceedings of the IEEE/CVF International Conference on
  Computer Vision (ICCV)}, pages 4195--4205, October 2023.

\bibitem[Plainer et~al.(2026)Plainer, Wu, Klein, Günnemann, and
  Noé]{plainer2026consistentsamplingsimulationmolecular}
M.~Plainer, H.~Wu, L.~Klein, S.~Günnemann, and F.~Noé.
\newblock Consistent sampling and simulation: Molecular dynamics with
  energy-based diffusion models, 2026.
\newblock URL \url{https://arxiv.org/abs/2506.17139}.

\bibitem[Rehman et~al.(2026)Rehman, Akhound-Sadegh, Gazizov, Bengio, and
  Tong]{rehman2025falconfewstepaccuratelikelihoods}
D.~Rehman, T.~Akhound-Sadegh, A.~Gazizov, Y.~Bengio, and A.~Tong.
\newblock {FALCON}: {Few}-step accurate likelihoods for continuous flows.
\newblock In \emph{The fourteenth international conference on learning
  representations}, 2026.

\bibitem[Rezende and Mohamed(2015)]{rezende2015variational}
D.~Rezende and S.~Mohamed.
\newblock Variational inference with normalizing flows.
\newblock In \emph{International conference on machine learning}, pages
  1530--1538. PMLR, 2015.

\bibitem[Rissanen et~al.(2025)Rissanen, OuYang, He, Chen, Heinonen, Solin, and
  Hernández-Lobato]{rissanen2025progressivetemperingsamplerdiffusion}
S.~Rissanen, R.~OuYang, J.~He, W.~Chen, M.~Heinonen, A.~Solin, and J.~M.
  Hernández-Lobato.
\newblock Progressive tempering sampler with diffusion, 2025.
\newblock URL \url{https://arxiv.org/abs/2506.05231}.

\bibitem[Schebek et~al.(2025)Schebek, Noé, and
  Rogal]{schebek2025scalableboltzmanngeneratorsequilibrium}
M.~Schebek, F.~Noé, and J.~Rogal.
\newblock Scalable boltzmann generators for equilibrium sampling of large-scale
  materials, 2025.
\newblock URL \url{https://arxiv.org/abs/2509.25486}.

\bibitem[Song and Dhariwal(2024)]{song2024improved}
Y.~Song and P.~Dhariwal.
\newblock Improved techniques for training consistency models.
\newblock In \emph{The Twelfth International Conference on Learning
  Representations}, 2024.
\newblock URL \url{https://openreview.net/forum?id=WNzy9bRDvG}.

\bibitem[Song et~al.(2021)Song, Sohl-Dickstein, Kingma, Kumar, Ermon, and
  Poole]{song2021scorebasedgenerativemodelingstochastic}
Y.~Song, J.~Sohl-Dickstein, D.~P. Kingma, A.~Kumar, S.~Ermon, and B.~Poole.
\newblock Score-based generative modeling through stochastic differential
  equations, 2021.
\newblock URL \url{https://arxiv.org/abs/2011.13456}.

\bibitem[Song et~al.(2023)Song, Dhariwal, Chen, and
  Sutskever]{song2023consistency}
Y.~Song, P.~Dhariwal, M.~Chen, and I.~Sutskever.
\newblock Consistency models.
\newblock In \emph{International Conference on Machine Learning}, pages
  32211--32252. PMLR, 2023.

\bibitem[Tan et~al.(2026{\natexlab{a}})Tan, Bose, Lin, Klein, Bronstein, and
  Tong]{tan2026scalableequilibriumsamplingsequential}
C.~B. Tan, A.~J. Bose, C.~Lin, L.~Klein, M.~M. Bronstein, and A.~Tong.
\newblock Scalable equilibrium sampling with sequential boltzmann generators,
  2026{\natexlab{a}}.
\newblock URL \url{https://arxiv.org/abs/2502.18462}.

\bibitem[Tan et~al.(2026{\natexlab{b}})Tan, Hassan, Klein, Syed, Beaini,
  Bronstein, Tong, and
  Neklyudov]{tan2026amortizedsamplingtransferablenormalizing}
C.~B. Tan, M.~Hassan, L.~Klein, S.~Syed, D.~Beaini, M.~M. Bronstein, A.~Tong,
  and K.~Neklyudov.
\newblock Amortized sampling with transferable normalizing flows,
  2026{\natexlab{b}}.
\newblock URL \url{https://arxiv.org/abs/2508.18175}.

\bibitem[Tong et~al.(2024)Tong, Fatras, Malkin, Huguet, Zhang, Rector-Brooks,
  Wolf, and Bengio]{tong2024condflow}
A.~Tong, K.~Fatras, N.~Malkin, G.~Huguet, Y.~Zhang, J.~Rector-Brooks, G.~Wolf,
  and Y.~Bengio.
\newblock Improving and generalizing flow-based generative models with
  minibatch optimal transport.
\newblock \emph{Transactions on Machine Learning Research}, 2024.
\newblock ISSN 2835-8856.
\newblock Expert Certification.

\bibitem[Torrie and Valleau(1977)]{torrie1977nonphysical}
G.~M. Torrie and J.~P. Valleau.
\newblock Nonphysical sampling distributions in monte carlo free-energy
  estimation: Umbrella sampling.
\newblock \emph{Journal of computational physics}, 23\penalty0 (2):\penalty0
  187--199, 1977.

\bibitem[Vincent(2011)]{vincent2011dsm}
P.~Vincent.
\newblock A {Connection} {Between} {Score} {Matching} and {Denoising}
  {Autoencoders}.
\newblock \emph{Neural Computation}, 23\penalty0 (7):\penalty0 1661--1674,
  2011.
\newblock \doi{10.1162/NECO_a_00142}.

\bibitem[Xie et~al.(2026)Xie, Winkler, Sun, Lewis, Foster, Luna, Hempel,
  Gastegger, Chen, Zaporozhets, et~al.]{xie2026enhanced}
Y.~Xie, L.~Winkler, L.~Sun, S.~Lewis, A.~E. Foster, J.~J. Luna, T.~Hempel,
  M.~Gastegger, Y.~Chen, I.~Zaporozhets, et~al.
\newblock Enhanced diffusion sampling: Efficient rare event sampling and free
  energy calculation with diffusion models.
\newblock \emph{arXiv preprint arXiv:2602.16634}, 2026.

\bibitem[Yu et~al.(2025)Yu, Klami, Hyvarinen, Korba, and Chehab]{yu2025density}
H.~Yu, A.~Klami, A.~Hyvarinen, A.~Korba, and O.~Chehab.
\newblock Density ratio estimation with conditional probability paths.
\newblock In \emph{International Conference on Machine Learning}, 2025.

\bibitem[Yu et~al.(2026)Yu, OuYang, Kaushik, Klami, Gutmann, and
  Chehab]{yu2026learningenergybasedmodelsstochastic}
H.~Yu, R.~OuYang, P.~Kaushik, A.~Klami, M.~U. Gutmann, and O.~Chehab.
\newblock Learning energy-based models from stochastic interpolants using
  spatiotemporal differences, 2026.
\newblock URL \url{https://arxiv.org/abs/2605.26850}.

\bibitem[Zhai et~al.(2025)Zhai, Zhang, Nakkiran, Berthelot, Gu, Zheng, Chen,
  Bautista, Jaitly, and Susskind]{zhai2025normalizingflowscapablegenerative}
S.~Zhai, R.~Zhang, P.~Nakkiran, D.~Berthelot, J.~Gu, H.~Zheng, T.~Chen, M.~A.
  Bautista, N.~Jaitly, and J.~Susskind.
\newblock Normalizing flows are capable generative models, 2025.
\newblock URL \url{https://arxiv.org/abs/2412.06329}.

\end{thebibliography}

% Appendix
\clearpage
\beginappendix
\startcontents[app]
\printcontents[app]{l}{1}{\setcounter{tocdepth}{2}}
% \appendix, \section*{Appendix}, and appendix TOC (\printcontents[app]) are in main_arxiv.tex
\paragraph*{Notation and Terminology.}

Throughout this appendix, we operate within the framework of generative modeling via continuous-time paths. Here, $x_0$ and $x_1$ denote the endpoint data variables, $x_t$ represents the intermediate state at time $t \in [0,1]$, and $z$ is the latent noise variable driving the process. Unless otherwise specified, conditional expectations of the form $\mathbb{E}[\cdot|x_t]$ are taken over the joint posterior distribution $p(x_0,x_1|x_t)$ of the endpoints given the intermediate state. At times with $\gamma_t>0$, the noise $z$ is uniquely determined once $(x_0,x_1,x_t)$ are known, so expectations involving the noise (such as $\mathbb{E}[z|x_t]$) are implicitly evaluated under this same posterior.

We use ``estimator'' for both a conditional training target and its exact
population conditional mean, distinguishing them when needed. Mini-batches
of joint interpolation samples give Monte Carlo estimates of the training
loss; they do not directly evaluate the posterior mean at a fixed $x_t$.
The population identities below describe the target of conditional
regression, rather than a claim that a finite trained network is unbiased.

\section{Conditional Estimators for Flow-based Model's Likelihood: Proofs and Connections}

In this section, we provide conditional estimators for estimating the model density, i.e. likelihoods, of samples generated by the following ODE in $\mathbb{R}^d$

\begin{align}
    \dd x_t=v_t(x_t)\dd t, \label{eq:generation-ode}
\end{align}

which starts from the base distribution and induces the laws $(p_t^v)_t$, with the aim of matching the target $p_1$.
While many solutions exist, we focus on the one that induces the same laws, $(p_t)_t$ as the following \emph{Stochastic Interpolant}

\begin{align}
    x_t=\alpha_t x_1+\beta_t x_0+\gamma_t z,\label{eq:stochastic-interpolants-def}
\end{align}

where $(x_0,x_1)$ are sampled from any coupling distribution $\pi$ with marginals $p_0,p_1$, respectively, $\alpha,\beta,\gamma$ are non-negative functions satisfying certain conditions, and $z$ follows a standard Gaussian distribution $\mathcal{N}(0,I)$, independently of $(x_0,x_1)$. All conditional Gaussian score identities below are restricted to times with $\gamma_t>0$ and assume the differentiability and integrability needed to interchange differentiation and integration.

\begin{remark}
    By choosing $p_0=\delta_0$, i.e. $X_0\equiv0$, the above formulation recovers the standard setting in \emph{Diffusion Models} (DM) and \emph{Flow Matching} (FM).
\end{remark}

\begin{remark}
    While the velocity field generating a given marginal probability path $p_t$ is not unique, the canonical {marginal velocity field} induced by the stochastic interpolant in \Cref{eq:stochastic-interpolants-def}, $(v_t^*)_{t\in[0,1]}$, is uniquely defined via conditional expectation \citep{albergo2025interpolant}:
    \begin{align}
        v_t^*(x_t) &= \mathbb{E}[\dot\alpha_t x_1 + \dot\beta_t x_0 + \dot\gamma_t z \mid x_t] \label{eq:velocity-def-in-stochastic-interpolants} \\
        &= \mathbb{E}[\dot\alpha_t x_1 + \dot\beta_t x_0 \mid x_t] + \dot\gamma_t\mathbb{E}[z \mid x_t]. \nonumber
    \end{align}
    Defining $f_t(x_t) := \mathbb{E}[\dot\alpha_t x_1 + \dot\beta_t x_0 \mid x_t]$ and applying Tweedie's formula to the Gaussian noise term, we can express the marginal velocity as:
    \begin{align}
        v_t^*(x_t) = f_t(x_t) - \dot\gamma_t\gamma_t\nabla\log p_t(x_t). \label{eq:velocity-score-relation-in-stochastic-interpolants}
    \end{align}
    
    An imperfect velocity $v_t\ne v_t^*$ can change the marginal path,
    but need not do so: if $\nabla\!\cdot[p_t(v_t-v_t^*)]=0$ at all times,
    the same marginals solve both continuity equations. For example, for
    radial $p_t$ in dimension at least two, a field
    $\delta_t(x)=a(\|x\|^2)Jx$, with $J^\top=-J$ and smooth compactly
    supported $a$, is tangential to the density level sets and has zero
    divergence. It can be nonzero while preserving $p_t$.
    The general mismatch is quantified in \Cref{app:sec:error-bounds}.
    
\end{remark}

In the following content, we aim to estimate the likelihood induced by \Cref{eq:generation-ode}, which is \begin{align}
    \log p_t^v(x_t)&=\log p_0(x_0)+\int_0^t{\dd}_u\log p^v_u(x_u)\dd u, \quad\text{where }\\
    {\dd}_u\log p^v_u(x_u)&=-\nabla\cdot v_u(x_u)=-\mathrm{tr}\bigl(\nabla v_u(x_u)\bigr),
\end{align}

and $(x_u)_{u\in[0,t]}$ follow the dynamics in \Cref{eq:generation-ode}. For clarity, we present the estimators and connections based on the following assumption:

\begin{assumption}\label{assump:perfect-v}
    We assume $v_t=v_t^*$, which implies $p_t^v=p_t^{v^*}$ and $\nabla\log p_t^v=\nabla\log p_t^{v^*}$.
\end{assumption}

Under \Cref{assump:perfect-v}, we denote $p_t$ as the marginals for clarity.

\subsection{Desiderata for a Consistent Objective}

At fixed $t$, let $Y=f_t(x_t|\xi)$ be a scalar or vector conditional target,
$m_t(x_t)=\mathbb E[Y\mid x_t]$ its population conditional mean, and
$D_\theta(x_t,t)$ a predictor. Assume the target and sampling law are fixed
independently of $\theta$ and the squared losses are finite. Conditional
expectation gives the loss decomposition
\citep{tong2024condflow,yu2025density}
\begin{align}
    \mathbb E\|D_\theta-Y\|^2
    =\mathbb E\|D_\theta-m_t\|^2+\mathbb E\|Y-m_t\|^2.
    \label{eq:conditional-loss-variance}
\end{align}
To see this, expand $D_\theta-Y=(D_\theta-m_t)+(m_t-Y)$; the cross term
has expectation zero because $\mathbb E[m_t-Y\mid x_t]=0$.
The final term is independent of $\theta$, so conditional and marginal
regression have the same minimizers and, when differentiation under the
expectation is justified, the same predictor gradients. The identity also
holds after averaging over $t$ with a fixed nonnegative weight $\lambda(t)$,
provided the weighted losses are finite. Thus it suffices to construct a
tractable conditional target whose mean is the desired marginal quantity.

\subsection{Derivation through Total Derivative}\label{app:sec:cond-total-derivative}

For clarity, we define the following notation throughout the paper:

\begin{align}
        \tilde v_t(x_t|x_0,x_1):&= \dot\alpha_tx_1+\dot\beta_tx_0+\dot\gamma_tz,\quad{\text{where }}z=\frac{x_t-(\alpha_tx_1+\beta_tx_0)}{\gamma_t},
\end{align}

and $v_t(x_t)=\mathbb{E}[\tilde v_t(x_t|x_0,x_1)|x_t]$.

To estimate $\dd_t\log p_t(x_t)$, one could expand it through total derivative

\begin{align}
    \dd_t\log p_t(x_t)&=\partial_t\log p_t(x_t)+\nabla\log p_t(x_t)\cdot\dd_tx_t\\
    &= \partial_t\log p_t(x_t)+\nabla\log p_t(x_t)\cdot v_t(x_t),
\end{align}

where both $\nabla\log p_t(x_t)$ and $\partial_t\log p_t(x_t)$ can be expressed as conditional expectations \citep{song2021scorebasedgenerativemodelingstochastic,yu2025density}, respectively, as follows:
\begin{align}
    \nabla\log p_t(x_t) &= \mathbb{E}\bigl[\nabla_{x_t}\log p(x_t|x_0,x_1)|x_t\bigr]\\
    &= -\frac{1}{\gamma_t}\mathbb{E}\bigl[z|x_t\bigr],\\
    \partial_t\log p_t(x_t) &= \mathbb{E}\bigl[\partial_t\log p(x_t|x_0,x_1)|x_t\bigr]\\
    &= \frac{1}{\gamma_t}\mathbb{E}\biggl[-d{\dot{\gamma_t}} + z\cdot\left( \dot{\alpha}_t x_{1} + \dot{\beta}_t x_{0}+{\dot{\gamma_t}} z \right)\Big|x_t\biggr]\\
    &= \frac{1}{\gamma_t}\mathbb{E}\biggl[-d{\dot{\gamma_t}} + z\cdot\tilde v_t(x_t|x_0,x_1)\Big|x_t\biggr].
\end{align}

Therefore, the total derivative can be expressed as a conditional expectation

\begin{align}
    \dd_t\log p_t(x_t)&= -\mathbb{E}\biggl[\frac{z}{\gamma_t}\cdot\Bigl(v_t(x_t)-\tilde v_t(x_t|x_0,x_1)\Bigr)+\frac{\dot\gamma_t}{\gamma_t}d\Big|x_t\biggr],\label{eq:cond-total-derivative-estimator}
\end{align}

which can be learned efficiently without any high-order derivatives given access to $v_t$.

\subsection{Derivation through the Continuity Equation}\label{app:sec:cond-total-derivative-through-continuity-equation}

Alternatively, \Cref{eq:chain-rule-for-logp-main} can be derived from the (log) continuity equation

{
    \begin{align}
        \partial_t p_t(x_t)+\nabla\cdot(v_t(x_t)p_t(x_t))=0
        \Leftrightarrow
        \partial_{t}\log p_{t}(x_{t}) = -\nabla\log p_{t}(x_{t}) \cdot v_{t}(x_{t}) - \nabla \cdot v_{t}(x_{t}).
    \end{align}}

\begin{minipage}[c]{0.65\textwidth}
        By the instantaneous change-of-variable \citep{chen2019neuralordinarydifferentialequations}, $\dd_t\log p_t(x_t)=-\nabla\cdot v_t(x_t)$, we recover \Cref{eq:chain-rule-for-logp-main}. Thus, the chain rule and continuity equation yield the instantaneous change-of-variables formula, providing an alternative proof to \citet{chen2019neuralordinarydifferentialequations}. These relations are visualized in the diagram on the right-hand side.

    \end{minipage}%
\hspace{0em}%
\begin{minipage}[c]{0.3\textwidth}
        \begin{center}

        \resizebox{\linewidth}{!}{
            \begin{tikzpicture}

                \node (top) at (0, 1.5) {$\nabla\log p_t(x_t)\cdot v_t(x_t) + \partial_t\log p_t(x_t)$};
                \node (left) at (-1.5, 0) {$\dd_t \log p_t(x_t)$};
                \node (right) at (1.5, 0) {$-\nabla\cdot v_t(x_t)$};

                \node (legend) at (0, -1.25) {
                    \begin{tabular}{rl}
                        $\color{finn}=\joinrel=$& only holds for $\dd x_t=v_t(x_t)\dd t$ \\
                        $\color{orange}=\joinrel=$& holds for any $x_t\in\mathbb{R}^d$
                    \end{tabular}
                };

                \path (left) -- (top)
                    node[midway, sloped] {\color{finn}$=\joinrel=\joinrel=$}
                    node[midway, above left, yshift=-4pt, xshift=-4pt, font=\footnotesize] {Chain Rule};

                \path (left) -- (right)
                    node[midway] {\color{finn}$=\joinrel=\joinrel=$}
                    node[midway, below=3pt, font=\footnotesize] {Inst. Change of Var.};

                \path (right) -- (top)
                    node[midway, sloped] {\color{orange}$\mathbf{=\joinrel=\joinrel=}$}
                    node[midway, above right, yshift=-4pt, xshift=4pt, font=\footnotesize] {Continuity Eq.};

            \end{tikzpicture}
        }

        \end{center}
    \end{minipage}

\subsection{Derivation through Manipulating the Trace Operator}\label{app:sec:cond-trace}

On the other hand, we could then expand the Jacobian of the velocity, $J_v(x_t)=\nabla v_t(x_t)$, as follows

\begin{align}
    \nabla v_t(x_t)&=\nabla_{x_t}\int p(x_0,x_1|x_t)\tilde v_t(x_t|x_0,x_1)dx_0dx_1\notag\\
    &\overset{(i)}{=}\int \biggl(\tilde v_t(x_t|x_0,x_1)\nabla_{x_t}p(x_0,x_1|x_t)^\top+p(x_0,x_1|x_t)\nabla_{x_t}\tilde v_t(x_t|x_0,x_1)\biggr)dx_0dx_1\notag\\
    &\overset{(ii)}{=}\int p(x_0,x_1|x_t)\biggl(\tilde v_t(x_t|x_0,x_1)\nabla_{x_t}\log p(x_0,x_1|x_t)^\top+\nabla_{x_t}\tilde v_t(x_t|x_0,x_1)\biggr)dx_0dx_1\notag\\
    &\overset{(iii)}{=}\int p(x_0,x_1|x_t)\biggl(\tilde v_t(x_t|x_0,x_1)\nabla_{x_t}\log p(x_t|x_0,x_1)^\top\notag\\
    &\qquad\qquad\qquad\qquad-\tilde v_t(x_t|x_0,x_1)\nabla_{x_t}\log p_t(x_t)^\top+\nabla_{x_t}\tilde v_t(x_t|x_0,x_1)\biggr)dx_0dx_1,
\end{align}

where $(i)$ is by $\nabla (pf)=f(\nabla p)^\top+p\nabla f$ for any vector-valued function $f$; $(ii)$ is according to $\nabla p=p\nabla\log p$; and $(iii)$ is based on $\nabla_{x_t}\log p(x_0,x_1|x_t)=\nabla_{x_t}\log p(x_t|x_0,x_1)-\nabla\log p_t(x_t)$. By plugging

\begin{align}
    \nabla_{x_t}\log p(x_t|x_0,x_1)&=\frac{\alpha_tx_1+\beta_tx_0-x_t}{\gamma_t^2}=-\frac{z}{\gamma_t},\\
    \nabla_{x_t}\tilde v_t(x_t|x_0,x_1)&=\nabla_{x_t}\biggl(\dot\alpha_tx_1+\dot\beta_tx_0+\dot\gamma_t\frac{x_t-\alpha_tx_1-\beta_tx_0}{\gamma_t}\biggr)=\frac{\dot\gamma_t}{\gamma_t}I,
\end{align}

one can simply write the Jacobian as a conditional expectation:

\begin{align}
    \nabla v_t(x_t)=\mathbb{E}\bigl[\tilde J_v(x_t|x_0,x_1)|x_t\bigr],
\end{align}

where

\begin{align}
    \tilde J_v(x_t|x_0,x_1) &= -\bigl(\dot\alpha_tx_1+\dot\beta_tx_0+\dot\gamma_tz\bigr)\biggl(\nabla\log p_t(x_t)+\frac{z}{\gamma_t}\biggr)^\top+\frac{\dot\gamma_t}{\gamma_t}I\\
    &= -\tilde v_t(x_t|x_0,x_1)\biggl(\nabla\log p_t(x_t)+\frac{z}{\gamma_t}\biggr)^\top+\frac{\dot\gamma_t}{\gamma_t}I.
\end{align}

The divergence is then obtained by applying a trace operator:

\begin{align}
    \nabla\cdot v_t(x_t)=\mathrm{tr}\biggl(\mathbb{E}\bigl[\tilde J_v(x_t|x_0,x_1)|x_t\bigr]\biggr)=\mathbb{E}\biggl[\mathrm{tr}\Bigl(\tilde J_v(x_t|x_0,x_1)\Bigr)|x_t\biggr],
\end{align}

where

\begin{align}
    \mathrm{tr}\Bigl(\tilde J_v(x_t|x_0,x_1)\Bigr) &= -\bigl(\dot\alpha_tx_1+\dot\beta_tx_0+\dot\gamma_tz\bigr)\cdot\biggl(\nabla\log p_t(x_t)+\frac{z}{\gamma_t}\biggr)+\frac{\dot\gamma_t}{\gamma_t}d\\
    &= -\tilde v_t(x_t|x_0,x_1)\cdot\biggl(\nabla\log p_t(x_t)+\frac{z}{\gamma_t}\biggr)+\frac{\dot\gamma_t}{\gamma_t}d.
\end{align}

\paragraph{Example 1.} When $p_0=\delta_0$, i.e. $X_0\equiv0$, the formulation reduces to a standard Diffusion model or Flow matching. By further choosing the optimal path, i.e. $x_t=tx_1+(1-t)z$, and noting that

\begin{align}
    v_t(x_t)=\frac{x_t}{t} +\frac{1-t}{t}\nabla\log p_t(x_t).\label{eq:velocity-score-relation-in-diffusion-with-optimal-path}
\end{align}

One could simply write the above conditional Jacobian as

\begin{align}
    \tilde J_v(x_t|x_0,x_1) &= -\bigl(x_1-z\bigr)\biggl(\frac{tv_t(x_t)-x_t}{1-t}+\frac{z}{1-t}\biggr)^\top-\frac{1}{1-t}I\\
    &= -\bigl(x_1-z\bigr)\biggl(\frac{tv_t(x_t)-t(x_1-z)}{1-t}\biggr)^\top-\frac{1}{1-t}I\\
    &=-\frac{t}{1-t}\bigl(x_1-z\bigr)\Big(v_t(x_t)-\bigl(x_1-z\bigr)\Big)^\top -\frac{1}{1-t}I.
\end{align}

\subsection{Connection between Two Estimators, and Beyond}

We now demonstrate that the two estimators for $\dd_t\log p_t(x_t)$ introduced in \Cref{app:sec:cond-total-derivative,app:sec:cond-trace} are mathematically equivalent, differing only by a zero-mean term. Recall the two formulations:

\begin{align}
    \dd_t\log p_t(x_t) &= -\mathbb{E}\biggl[\frac{z}{\gamma_t}\cdot\Bigl(v_t(x_t)-\tilde v_t(x_t|x_0,x_1)\Bigr)+\frac{\dot\gamma_t}{\gamma_t}d\Big|x_t\biggr]\label{eq:cond-total-derivative-estimator-restate}\\
    &= -\mathbb{E}\biggl[-\tilde v_t(x_t|x_0,x_1)\cdot\biggl(\nabla\log p_t(x_t)+\frac{z}{\gamma_t}\biggr)+\frac{\dot\gamma_t}{\gamma_t}d\Big|x_t\biggr]\label{eq:cond-trace-estimator-restate}.
\end{align}

Notice that for any vector-valued function $h_t(x_t)$,
according to \Cref{eq:velocity-def-in-stochastic-interpolants}, we have

\begin{align}
    \mathbb{E}\Bigl[h_t(x_t)\cdot\bigl(v_t(x_t)-\tilde v_t(x_t|x_0,x_1)\bigr)\Big|x_t\Bigr] &= h_t(x_t)\cdot\Bigl(v_t(x_t)-\mathbb{E}\bigl[\tilde v_t(x_t|x_0,x_1)|x_t\bigr]\Bigr)\notag\\
    &= h_t(x_t)\cdot\Bigl(v_t(x_t)-v_t(x_t)\Bigr)=0.\label{eq:control-variate-with-h}
\end{align}

Then \Cref{eq:cond-total-derivative-estimator-restate} can be written as

\begin{align}
    &-\mathbb{E}\biggl[\biggl(\frac{z}{\gamma_t}+\nabla\log p_t(x_t)\biggr)\cdot\Bigl(v_t(x_t)-\tilde v_t(x_t|x_0,x_1)\Bigr)+\frac{\dot\gamma_t}{\gamma_t}d\Big|x_t\biggr]\\
    =&-\mathbb{E}\biggl[-\tilde v_t(x_t|x_0,x_1)\cdot\biggl(\frac{z}{\gamma_t}+\nabla\log p_t(x_t)\biggr)+\frac{\dot\gamma_t}{\gamma_t}d\Big|x_t\biggr]\notag\\
    &\qquad\qquad\qquad\qquad\qquad\qquad\qquad\qquad-v_t(x_t)\cdot\biggl(\underbrace{\frac{1}{\gamma_t}\mathbb{E}[z|x_t]+\nabla\log p_t(x_t)}_{\overset{(i)}{=}-\nabla\log p_t(x_t)+\nabla\log p_t(x_t)=0}\biggr),
\end{align}

where $(i)$ is by Tweedie's formula.

\subsection{Vectorized Estimators}\label{app:sec:vectorized-estimators}

We first conclude that the vectorized estimators obtained by replacing the inner-product $\cdot$ in \Cref{eq:cond-total-derivative-estimator-restate,eq:cond-trace-estimator-restate} with an elementwise product $\odot$ are unbiased estimators for $\mathrm{diag}(-\nabla v_t(x_t))$.

The conclusion can be simply obtained from the derivation of the conditional trace estimator, \Cref{eq:cond-trace-estimator-restate}, where

\begin{align}
    \mathrm{diag}\bigl(ab^\top\bigr)=a\odot b\quad\text{and}\quad \mathrm{tr}\bigl(ab^\top\bigr)=a\cdot b.
\end{align}

The conclusion also applies to the conditional total derivative estimator, \Cref{eq:cond-total-derivative-estimator-restate}, since the difference between \Cref{eq:cond-total-derivative-estimator-restate,eq:cond-trace-estimator-restate} is a term with zero expectation as stated in their connection.

\paragraph{Transformation of the vectorized target.}

If the velocity is rotation-equivariant, its Jacobian satisfies
$J_v(Qx)=QJ_v(x)Q^\top$ for an orthogonal block rotation $Q$.
The trace is invariant, but the diagonal generally does not transform as
an ordinary vector: $\mathrm{diag}(QJQ^\top)\ne Q\,\mathrm{diag}(J)$.
For instance, $J=I$ has unchanged diagonal while $Q\mathbf1_d$ need not
equal $\mathbf1_d$. Thus the likelihood head learns coordinatewise
diagonal targets whose sum estimates an invariant scalar. Rotation
augmentation encourages the desired learned symmetry, but does not by
itself guarantee exact equivariance or invariance.

\subsection{Connection with Higher-Order Score Matching}\label{app:sec:connection-to-higher-order-score-matching}

We connect the conditional estimators with the second-order Tweedie identity.
For $x\sim p_{\mathrm{data}}$ and
$q_\sigma(\tilde x\mid x)=\mathcal N(\tilde x;x,\sigma^2I)$ with
$\sigma>0$, let $\tilde p$ be the marginal density of $\tilde x$ and write
$\otimes^2x=xx^\top$. We assume the differentiability and moment conditions
needed for the conditional moments below.

According to \Cref{eq:velocity-score-relation-in-stochastic-interpolants} and under \Cref{assump:perfect-v}, the Jacobian of velocity can be written as

\begin{align}
    \nabla v_t(x_t)=\nabla f_t(x_t)-\dot\gamma_t\gamma_t\nabla^2\log p_t(x_t).
\end{align}

The second-order Tweedie identity \citep{mengEstimatingHighOrder2021} gives

\begin{align}
    &\mathbb{E}[\otimes^2 x|\tilde x]=\sigma^2I + \sigma^4\nabla^2\log\tilde p(\tilde x) + \otimes^2\bigl(\tilde x+\sigma^2\nabla\log\tilde p(\tilde x)\bigr)\\
    \Longrightarrow &\nabla^2\log\tilde p(\tilde x)=\frac{1}{\sigma^4}\Bigl(\mathbb{E}[\otimes^2 x|\tilde x]-\sigma^2I-\otimes^2\bigl(\tilde x+\sigma^2\nabla\log\tilde p(\tilde x)\bigr)\Bigr).
\end{align}

Let $x=\alpha_tx_1+\beta_tx_0$. Then we have

\begin{align}
    \tilde p(\tilde x)=\int \pi(x_0,x_1)\delta(\alpha_tx_1+\beta_tx_0-x)\mathcal{N}(\tilde x;x,\gamma_t^2I)dxdx_0dx_1=p_t(\tilde x).
\end{align}

Hence, by applying the above second-order Tweedie's formula, the Hessian of $\log p_t(x_t)$ can be written as

\begin{align}
    \nabla^2\log p_t(x_t) &= \frac{1}{\gamma_t^4}\Bigl(\mathbb{E}[\otimes^2\bigl(\alpha_t x_1+\beta_t x_0\bigr)|x_t]-\gamma_t^2I-\otimes^2\bigl( x_t+\gamma_t^2\nabla\log p_t(x_t)\bigr)\Bigr)
\end{align}

For clarity, we focus on the example of Diffusion models, i.e. $p_0=\delta_0$. In such a case, the velocity and score can be related as

\begin{align}
      v_t(x_t)=\frac{\dot\alpha_t}{\alpha_t}x_t+\biggl(\frac{\dot\alpha_t}{\alpha_t}-\frac{\dot\gamma_t}{\gamma_t}\biggr)\gamma_t^2\nabla\log p_t(x_t),\label{eq:velocity-score-relation-in-diffusion-general}
\end{align}

and therefore

\begin{align}
     \nabla v_t(x_t)=\frac{\dot\alpha_t}{\alpha_t}I+\biggl(\frac{\dot\alpha_t}{\alpha_t}-\frac{\dot\gamma_t}{\gamma_t}\biggr)\gamma_t^2\nabla^2\log p_t(x_t).
\end{align}

For clarity, we denote $c_t=\dot\alpha_t/\alpha_t-\dot\gamma_t/\gamma_t$.
The Hessian can be simplified as follows:

\begin{align}
    \nabla^2\log p_t(x_t) &= \frac{1}{\gamma_t^4}\Bigl(\mathbb{E}[\alpha_t^2 x_1 x_1^\top|x_t]-\gamma_t^2I-\otimes^2\bigl( x_t+\gamma_t^2\nabla\log p_t(x_t)\bigr)\Bigr).
\end{align}

Since $\dd_t\log p_t(x_t)$ is the negative trace of the Jacobian of $v_t$, we focus on the diagonal entry of the Hessian:

\begin{align}
    \mathrm{diag}\bigl(\nabla^2\log p_t(x_t)\bigr) &= \frac{1}{\gamma_t^4}\Bigl(\mathbb{E}[\alpha_t^2\odot^2 x_1|x_t]-\gamma_t^2\mathbf{1}_d-\odot^2 x_t \notag\\
    &\qquad\qquad- 2\gamma_t^2x_t\odot \nabla\log p_t(x_t) - \gamma_t^4\odot^2 \nabla\log p_t(x_t)\Bigr),
\end{align}

where we denote $\odot^2$ as $2$-fold elementwise product for clarity. And notice that

\begin{align}
    \mathbb{E}[\alpha_t^2\odot^2 x_1|x_t] &=\mathbb{E}[\odot^2(x_t-\gamma_t z)|x_t]\\
    &=\odot^2x_t - 2\gamma_tx_t\odot\mathbb{E}[z|x_t]+\gamma_t^2\mathbb{E}[\odot^2 z|x_t]\\
    &\overset{(i)}{=}\odot^2x_t + 2\gamma_t^2x_t\odot\nabla\log p_t(x_t)+\gamma_t^2\mathbb{E}[\odot^2 z|x_t],
\end{align}

where $(i)$ is by Tweedie's formula. The diagonal of the Hessian can be simplified as

\begin{align}
    \mathrm{diag}\bigl(\nabla^2\log p_t(x_t)\bigr) &= \mathbb{E}\biggl[\odot^2\biggl(\frac{1}{\gamma_t}z\biggr)\Big|x_t\biggr] - \odot^2\nabla\log p_t(x_t)-\frac{1}{\gamma_t^2}\mathbf{1}_d.
\end{align}

Further notice that

\begin{align}
    \mathbb{E}\biggl[\odot^2\biggl(\frac{1}{\gamma_t}z\biggr)\Big|x_t\biggr] &= \mathbb{E}\biggl[\odot^2\biggl(\frac{1}{\gamma_t}z+\nabla\log p_t(x_t)\biggr)\Big|x_t\biggr] \notag\\
    &\quad- 2\mathbb{E}\biggl[\frac{1}{\gamma_t}z\Big|x_t\biggr]\odot \nabla\log p_t(x_t)-\odot^2\nabla\log p_t(x_t)\\
    &\overset{(i)}{=}\mathbb{E}\biggl[\odot^2\biggl(\frac{1}{\gamma_t}z+\nabla\log p_t(x_t)\biggr)\Big|x_t\biggr]+\odot^2\nabla\log p_t(x_t),
\end{align}

where $(i)$ is again by Tweedie's formula. We have the simplification

\begin{align}
     \mathrm{diag}\bigl(\nabla^2\log p_t(x_t)\bigr) &= \mathbb{E}\biggl[\odot^2\biggl(\frac{1}{\gamma_t}z+\nabla\log p_t(x_t)\biggr)\Big|x_t\biggr]-\frac{1}{\gamma_t^2}\mathbf{1}_d\\
     &= \mathbb{E}\biggl[\frac{z}{\gamma_t}\odot\biggl(\frac{z}{\gamma_t}+\nabla\log p_t(x_t)\biggr)\Big|x_t\biggr]-\frac{1}{\gamma_t^2}\mathbf{1}_d.
\end{align}

Therefore,

\begin{align}
    \mathrm{diag}\bigl(\nabla v_t(x_t)\bigr)&=\frac{\dot\alpha_t}{\alpha_t}\mathbf{1}_d+c_t\gamma_t^2
    \mathbb{E}\biggl[\frac{z}{\gamma_t}\odot\biggl(\frac{z}{\gamma_t}+\nabla\log p_t(x_t)\biggr)\Big|x_t\biggr]-c_t\mathbf{1}_d\\
    &\overset{(i)}{=}\mathbb{E}\biggl[\frac{z}{\gamma_t}\odot\biggl(c_t\gamma_t z+v_t(x_t)-\frac{\dot\alpha_t}{\alpha_t}x_t\biggr)\Big|x_t\biggr]+\frac{\dot\gamma_t}{\gamma_t}\mathbf{1}_d\\
    &\overset{(ii)}{=}\frac{1}{\gamma_t}\mathbb{E}\biggl[z\odot\Bigl(v_t(x_t)-\tilde v_t(x_t|x_0,x_1)\Bigr)\Big|x_t\biggr]+\frac{\dot\gamma_t}{\gamma_t}\mathbf{1}_d,
\end{align}

where $(i)$ follows from the relationship between velocity and score stated in \Cref{eq:velocity-score-relation-in-diffusion-general} and $(ii)$ is by

\begin{align}
    \frac{\dot\alpha_t}{\alpha_t}x_t-c_t\gamma_tz&=\frac{\dot\alpha_t}{\alpha_t}(\alpha_tx_1+\cancel{\gamma_tz})-\biggl(\cancel{\frac{\dot\alpha_t}{\alpha_t}}-{\frac{\dot\gamma_t}{\gamma_t}}\biggr)\gamma_tz= \dot\alpha_tx_1+\dot\gamma_tz=\tilde v_t(x_t|x_0,x_1).
\end{align}

Remarkably, it recovers the estimator derived from the conditional total derivative, \Cref{eq:cond-total-derivative-estimator}, in the diffusion setting.

\paragraph{Example.} When choosing the optimal path, i.e. $\alpha_t=t$ and $\gamma_t=1-t$, the diagonal of Jacobian results in

\begin{align}
    \mathrm{diag}\bigl(\nabla v_t(x_t)\bigr){=}\frac{1}{1-t}\mathbb{E}\biggl[z\odot\Bigl(v_t(x_t)-(x_1-z)\Bigr)\Big|x_t\biggr]-\frac{1}{1-t}\mathbf{1}_d,
\end{align}

\subsection{Error Quantification}\label{app:sec:error-bounds}

\subsubsection{Notation and Assumptions}\label{app:sec:error-bounds:notation}

We work on $\mathbb{R}^d$. Let $v_t^*$ denote the ground-truth velocity field
generating the probability path $(p_t)_{t\in[0,1]}$ through the continuity
equation $\partial_t p_t+\nabla\!\cdot(p_t v_t^*)=0$, and let $v_t$ denote the
velocity learned by the model, with marginal $p_t^v$ satisfying
$\partial_t p_t^v+\nabla\!\cdot(p_t^v v_t)=0$. We assume throughout that the
densities are strictly positive and that the densities and velocity fields are
sufficiently regular for the continuity equations and the characteristic
identities below to hold pointwise. In particular, the learned ODE admits a
global diffeomorphic flow $\phi_t$ such that
$\phi_t{}_\#p_0=p_t^v$. Define the velocity error

\begin{align}
    \delta_t(x):=v_t(x)-v_t^*(x).
\end{align}

We take $p_0=\mathcal N(0,I)$ and adopt the parameterization $x_t=\alpha_t x_1+\gamma_t z$, where
$x_1\sim p_1$ and $z\sim\mathcal{N}(0,I)$ are independent. The schedules satisfy
$\alpha_0=0$, $\gamma_0=1$, $\alpha_1=1$, $\gamma_1=0$,
$\alpha_t>0$ for $t\in(0,1]$, and $\gamma_t>0$ for $t\in[0,1)$.
We define the following quantities:

\begin{align}
    e_t^2:=\mathbb{E}_{p_t}\!\bigl[\|\delta_t(x_t)\|^2\bigr],\qquad
    B_t:=\|\delta_t\|_{L^\infty(\mathbb{R}^d)},\qquad
    S_t^2:=\mathbb{E}_{p_t}\!\bigl[\|\mathcal{D}_t\delta_t(x_t)\|^2\bigr],
\end{align}

where $\|\cdot\|$ is the Euclidean norm and $L^\infty$ denotes the essential
supremum (with the Euclidean norm for vector-valued functions), and

\begin{align}
    \mathcal{D}_t f := \Delta f+\nabla\log p_t\cdot\nabla f
    \;=\;p_t^{-1}\nabla\!\cdot\!\bigl(p_t\,\nabla f\bigr)\label{eq:langevin-generator}
\end{align}

is the generator of the overdamped Langevin diffusion with invariant measure
$p_t$, applied component-wise to vector fields. We make the following
assumptions:

\begin{assumption}\label{ass:bounded}
    $B_t<\infty$ for $t\in[0,1]$.
\end{assumption}

\begin{assumption}\label{ass:hess}
    For every $t\in[0,1]$, the components of $\delta_t$ lie in the domain of
    $\mathcal D_t$, $S_t<\infty$, and the weighted integration-by-parts
    identity
    \begin{align}
        \mathbb E_{p_t}\!\bigl[\|\nabla\delta_t(x_t)\|_F^2\bigr]
        =
        -\mathbb E_{p_t}\!\bigl[\delta_t(x_t)\cdot\mathcal D_t\delta_t(x_t)\bigr]
        \label{eq:weighted-ibp-assumption}
    \end{align}
    holds. A sufficient condition is that $\delta_t$ is twice continuously
    differentiable, all displayed integrals are absolutely convergent, and
    the boundary term in the divergence theorem vanishes.
\end{assumption}

\begin{assumption}\label{ass:bounded-4-moment}
    The density $p_1$ is continuously differentiable, the integration by
    parts used in the target-score identity below has no boundary term, and
    its fourth-order score moment is finite:
    \begin{align}
        J_4(p_1):=\mathbb{E}_{p_1}[\|\nabla\log p_1(x_1)\|^4]<\infty.
    \end{align}
\end{assumption}

We further define the following score factor for notational convenience:

\begin{align}
    G_t
    &:=
    \min\biggl(
        \frac{\sqrt{d(d+2)}}{\gamma_t^2},
        \frac{\sqrt{J_4(p_1)}}{\alpha_t^2}
    \biggr),
    \label{eq:score-factor-G}
\end{align}

with the extended-value convention that any term with zero denominator is
interpreted as $+\infty$. In particular,
$G_t<\infty$, $G_0=\sqrt{d(d+2)}$, and
$G_1=\sqrt{J_4(p_1)}$.

\paragraph{Outline of the proof.}
We first bound the local weighted-divergence error and propagate it along
the learned flow, including the residual of a fitted likelihood head.
The ideal integrated surrogate equals the interpolation density exactly,
so its population importance-sampling estimand reduces to an expectation
under the teacher's sampling law. We then bound the teacher's observable
bias and the population importance-sampling error of a fitted head, and
state their dependence on assumed learning rates.

\subsubsection{Pointwise Error}

Along trajectories of the learned ODE $\dd x_t=v_t(x_t)\dd t$, the
instantaneous change-of-variables formula gives

\begin{align}
    \dd_t\log p_t^v(x_t)
    =\partial_t\log p_t^v(x_t)+\nabla\log p_t^v(x_t)\cdot v_t(x_t)
    =-\nabla\!\cdot v_t(x_t).
\end{align}

Because both $\nabla\log p_t^v$ and $\partial_t\log p_t^v$ are intractable, we replace them by $\nabla\log p_t$ and $\partial_t\log p_t$ respectively to obtain the estimator

\begin{align}
    \widetilde{\dd_t\log p_t^v}(x_t)
    :=\partial_t\log p_t(x_t)+\nabla\log p_t(x_t)\cdot v_t(x_t),
    \label{eq:dlogpv-dt-estimator}
\end{align}

where both $\nabla\log p_t$ and $\partial_t\log p_t$ could be estimated by a posterior expectation over the dataset \citep{song2021scorebasedgenerativemodelingstochastic,yu2025density}.

We further introduce the material derivative of the data density in the
direction of the ground-truth velocity field. The superscript $v^*$ below
specifies this direction and does not denote a density different from $p_t$:

\begin{align}
    \dd_t\log p_t^{v^*}(x)
    &:=
    \partial_t\log p_t(x)+\nabla\log p_t(x)\cdot v_t^*(x)\\
    &=
    -\nabla\!\cdot v_t^*(x),
    \label{eq:dlogpvstar-dt}
\end{align}

where the last equality follows from the continuity equation
$\partial_t\log p_t+\nabla\log p_t\cdot v_t^*=-\nabla\!\cdot v_t^*$.

Define the pointwise error of the estimator as

\begin{align}
    \mathcal{E}_t(x_t)
    &:=
    \widetilde{\dd_t\log p_t^v}(x_t)-\dd_t\log p_t^v(x_t).
\end{align}

Adding and subtracting $\dd_t\log p_t^{v^*}(x_t)$, we can decompose the error as

\begin{align}
    \mathcal{E}_t(x_t)
    &=
    \underbrace{
    \widetilde{\dd_t\log p_t^v}(x_t)
    -
    \dd_t\log p_t^{v^*}(x_t)
    }_{\text{score--velocity mismatch}}
    +
    \underbrace{
    \dd_t\log p_t^{v^*}(x_t)
    -
    \dd_t\log p_t^v(x_t)
    }_{\text{model-flow mismatch}}.
    \label{eq:Et-decomposition}
\end{align}

The first term is given by

\begin{align}
    \widetilde{\dd_t\log p_t^v}(x_t)
    -
    \dd_t\log p_t^{v^*}(x_t)
    &=
    \partial_t\log p_t(x_t)+\nabla\log p_t(x_t)\cdot v_t(x_t) \notag\\
    &\quad -
    \partial_t\log p_t(x_t)-\nabla\log p_t(x_t)\cdot v_t^*(x_t)\\
    &=
    \nabla\log p_t(x_t)\cdot\delta_t(x_t).
\end{align}

The second term is given by

\begin{align}
    \dd_t\log p_t^{v^*}(x_t)
    -
    \dd_t\log p_t^v(x_t)
    &=
    -\nabla\!\cdot v_t^*(x_t)-\bigl(-\nabla\!\cdot v_t(x_t)\bigr)\\
    &=
    \nabla\!\cdot\delta_t(x_t).
\end{align}

Therefore,

\begin{align}
    \mathcal{E}_t(x_t)
    =
    \nabla\log p_t(x_t)\cdot\delta_t(x_t)
    +
    \nabla\!\cdot\delta_t(x_t).
\end{align}

Equivalently, the error is the weighted divergence

\begin{align}
    \mathcal E_t=p_t^{-1}\nabla\!\cdot(p_t\delta_t).
    \label{eq:weighted-divergence-error}
\end{align}

A nonzero velocity error can therefore have zero scalar CDM error. If
$\nabla\!\cdot(p_t\delta_t)=0$ throughout the path, then $p_t$ also solves
the learned continuity equation; uniqueness with the common initial
condition gives $p_t^v=p_t$.

\begin{lemma}[Mean-square error of the estimator]\label{lem:Et-bound}
Under \Cref{ass:bounded,ass:hess,ass:bounded-4-moment} and the flow and
interpolation regularity stated above, define
\begin{align}
    \eta_t:=\|\mathcal E_t\|_{L^2(p_t)},\qquad
    K_t:=\sqrt{e_t}\bigl(\sqrt{S_t}+\sqrt{B_tG_t}\bigr).
    \label{eq:sharper-local-error}
\end{align}
Then, for $t\in[0,1]$,
\begin{align}
    \mathbb E_{p_t}[\mathcal E_t^2]
    =\eta_t^2\le K_t^2\le 2e_t(S_t+B_tG_t).
\end{align}
\end{lemma}

\begin{proof}
We first establish the score and velocity moment estimates, then control
the full weighted divergence without bounding its two summands separately.
For $t<1$, differentiation of the Gaussian convolution gives
\begin{align}
    \nabla\log p_t(x_t)=-\gamma_t^{-1}\mathbb E[z\mid x_t],
    \qquad
    \mathbb E_{p_t}\|\nabla\log p_t\|^4
    \le \frac{\mathbb E\|z\|^4}{\gamma_t^4}
    =\frac{d(d+2)}{\gamma_t^4}.
\end{align}
For $0<t<1$, integration by parts against $p_1$ instead gives the target
score identity \citep{debortoli2024targetscorematching},
\begin{align}
    \nabla\log p_t(x_t)
    =\alpha_t^{-1}\mathbb E[\nabla\log p_1(x_1)\mid x_t],
    \qquad
    \mathbb E_{p_t}\|\nabla\log p_t\|^4
    \le\frac{J_4(p_1)}{\alpha_t^4}.
\end{align}
The latter bound holds directly at $t=1$. Conditional Jensen's inequality
justifies both moment bounds. Thus, with the endpoint convention in
\Cref{eq:score-factor-G},
\begin{align}
    \bigl(\mathbb E_{p_t}\|\nabla\log p_t\|^4\bigr)^{1/2}\le G_t,
    \qquad
    \bigl(\mathbb E_{p_t}\|\delta_t\|^4\bigr)^{1/2}\le B_te_t.
    \label{eq:score-velocity-fourth-moments}
\end{align}
Suppress $t$ and write $g=\nabla\log p$. The weighted integration-by-parts
identity and Cauchy--Schwarz imply
\begin{align}
    A^2:=\mathbb E_p\|\nabla\delta\|_F^2
    &=-\mathbb E_p[\delta\cdot\mathcal D\delta]\le eS,\\
    C^2:=\mathbb E_p[\|g\|^2\|\delta\|^2]
    &\le\bigl(\mathbb E_p\|g\|^4\bigr)^{1/2}
           \bigl(\mathbb E_p\|\delta\|^4\bigr)^{1/2}\le GBe.
\end{align}
Set $F=\sqrt p\,\delta$ and let $\|\cdot\|_2$ denote the unweighted
$L^2(\mathrm dx)$ norm. Then $F\in H^1(\mathbb R^d;\mathbb R^d)$:
$\|F\|_2=e$ and the two terms in its weak derivative are square-integrable
by the preceding estimates. The product rule yields
\begin{align}
    \sqrt p\,\mathcal E
    &=\nabla\!\cdot F+\tfrac12\sqrt p\,g\cdot\delta,\\
    \nabla F&=\sqrt p\,\nabla\delta+\tfrac12\sqrt p\,\delta g^\top.
\end{align}
For an $H^1$ vector field,
$\|\nabla\!\cdot F\|_2\le\|\nabla F\|_{2,F}$, where the second norm uses
the Frobenius norm. Indeed, Plancherel's identity and
$|\xi\cdot\widehat F(\xi)|^2\le\|\xi\|^2\|\widehat F(\xi)\|^2$
prove this integrated inequality. Applying the triangle inequality twice gives
\begin{align}
    \eta=\|\sqrt p\,\mathcal E\|_2
    \le\|\nabla F\|_{2,F}+\tfrac12 C
    \le A+C\le\sqrt{eS}+\sqrt{eBG}=K.
\end{align}
Squaring and using $(a+b)^2\le2a^2+2b^2$ completes the proof.
\end{proof}

\subsubsection{Error of Likelihood Estimation}

Recall the estimator $\widetilde{\dd_t\log p_t^v}$ from
\Cref{eq:dlogpv-dt-estimator}. We turn it into a surrogate log-density by
integrating it along the learned trajectory $\phi$ (where $x_t=\phi_t(x_0)$,
$x_0\sim p_0$), matching the exact log-density at the same initial condition:

\begin{align}
    \log\tilde p_t^v\bigl(\phi_t(x_0)\bigr)
    &:=\log p_0(x_0)+\int_0^t\!\widetilde{\dd_s\log p_s^v}\bigl(\phi_s(x_0)\bigr)\,ds,\\
    \log p_t^v\bigl(\phi_t(x_0)\bigr)
    &=\log p_0(x_0)+\int_0^t\!\dd_s\log p_s^v\bigl(\phi_s(x_0)\bigr)\,ds .
\end{align}

The first integrand is exactly the derivative of the interpolation
log-density along the learned trajectory:

\begin{align}
    \widetilde{\dd_s\log p_s^v}(\phi_s(x_0))
    =\frac{\mathrm d}{\mathrm ds}\log p_s(\phi_s(x_0)).
\end{align}

Since $p_0^v=p_0$ and $\phi_t$ is a global diffeomorphism, the chain rule
and the common initial condition give the identities

\begin{align}
    \tilde p_t^v=p_t,\qquad
    \Phi_t:=\log\tilde p_t^v-\log p_t^v=\log\frac{p_t}{p_t^v}.
    \label{eq:population-cdm-density-identity}
\end{align}

Thus ideal population CDM integrated along an approximate teacher tracks
the interpolation density, which need not equal the teacher's sampling density.
Subtracting the two trajectory integrals gives

\begin{align}
    \Phi_t(\phi_t(x_0))
    =\int_0^t\mathcal E_s(\phi_s(x_0))\,\mathrm ds.
    \label{eq:Phi-integral-rep}
\end{align}

All statements are conditional on the learned functions. Write
$m_t:=\widetilde{\dd_t\log p_t^v}$ and let $D_t^\theta$ be a fitted scalar
instantaneous head. Define its residual, training error, and total local error by
\begin{align}
    r_{D,t}(x)&:=D_t^\theta(x)-m_t(x),\qquad
    \varepsilon_{D,t}^2:=\mathbb E_{p_t^v}|r_{D,t}|^2,
    \label{eq:D-training-error}\\
    \mathcal E_t^\theta&:=D_t^\theta-\dd_t\log p_t^v
    =r_{D,t}+\mathcal E_t.
    \label{eq:theta-error-decomposition}
\end{align}
Its integrated log-density estimate and error are
\begin{align}
    \log\tilde p_t^{v,\theta}(\phi_t(x_0))
    &:=\log p_0(x_0)+\int_0^tD_s^\theta(\phi_s(x_0))\,\mathrm ds,\\
    \Phi_t^\theta&:=\log\tilde p_t^{v,\theta}-\log p_t^v.
\end{align}
The positive function $\tilde p_t^{v,\theta}$ need not integrate to one.

\begin{proposition}[Likelihood error along the learned flow]
\label{prop:finite-sample-likelihood-error}
Under the assumptions of \Cref{lem:Et-bound}, let
$M_s^*:=\|p_s^v/p_s\|_{L^\infty(p_s)}$ be finite for almost every
$s\in[0,t]$. For a fitted head with finite integrals below,
\begin{align}
    \|\Phi_t^\theta\|_{L^2(p_t^v)}
    &\le\int_0^t\bigl(\varepsilon_{D,s}+\sqrt{M_s^*}\,\eta_s\bigr)\,\mathrm ds
    \notag\\
    &\le A_t^\theta
    :=\int_0^t\bigl(\varepsilon_{D,s}+\sqrt{M_s^*}\,K_s\bigr)\,\mathrm ds,
    \label{eq:fitted-head-minkowski}\\
    \mathbb E_{p_t^v}|\Phi_t^\theta|^2
    &\le t\int_0^t\left[2\varepsilon_{D,s}^2
                  +4M_s^*e_s(S_s+B_sG_s)\right]\,\mathrm ds.
    \label{eq:finite-sample-likelihood-error}
\end{align}
For the ideal head $D_s^\theta=m_s$, the residual vanishes and
$\tilde p_t^{v,\theta}=\tilde p_t^v=p_t$. In this case,
\begin{align}
    \|\Phi_t\|_{L^2(p_t^v)}
    &\le\int_0^t\sqrt{M_s^*}\,\eta_s\,\mathrm ds
    \le\int_0^t\sqrt{M_s^*}\,K_s\,\mathrm ds,\label{eq:phi-minkowski}\\
    \mathbb E_{p_t^v}[\Phi_t^2]
    &\le t\int_0^t M_s^*K_s^2\,\mathrm ds
    \le2t\int_0^t M_s^*e_s(S_s+B_sG_s)\,\mathrm ds.
\end{align}
\end{proposition}
\begin{proof}
The exact decomposition along a learned trajectory is
\begin{align}
    \Phi_t^\theta(\phi_t(x_0))
    &=\log\frac{p_t}{p_t^v}(\phi_t(x_0))
      +\int_0^t r_{D,s}(\phi_s(x_0))\,\mathrm ds\\
    &=\int_0^t\mathcal E_s^\theta(\phi_s(x_0))\,\mathrm ds.
    \label{eq:fitted-density-decomposition}
\end{align}
The pushforward relation $\phi_s{}_\#p_0=p_s^v$ and Minkowski's integral
inequality give
\begin{align}
    \|\Phi_t^\theta\|_{L^2(p_t^v)}
    \le\int_0^t\|\mathcal E_s^\theta\|_{L^2(p_s^v)}\,\mathrm ds.
\end{align}
By the triangle inequality and change of measure,
\begin{align}
    \|\mathcal E_s^\theta\|_{L^2(p_s^v)}
    &\le\varepsilon_{D,s}+\sqrt{M_s^*}\,\eta_s
    \le\varepsilon_{D,s}+\sqrt{M_s^*}\,K_s,
    \label{eq:theta-error-split}\\
    \mathbb E_{p_s^v}|\mathcal E_s^\theta|^2
    &\le2\varepsilon_{D,s}^2+4M_s^*e_s(S_s+B_sG_s).
    \label{eq:theta-local-error-bound}
\end{align}
Cauchy--Schwarz in time proves the mean-square bound. Setting
$r_{D,s}=0$ gives the ideal-head norm bound; applying Cauchy--Schwarz
directly to that bound gives the stated ideal-head mean-square estimate.
\end{proof}

\subsubsection{Observable Error of the Teacher Flow}

For $h:\mathbb R^d\to\mathbb R^k$, the population
importance-sampling estimand is

\begin{align}
    H_t:=\mathbb E_{p_t}h,\qquad
    \tilde H_t:=\mathbb E_{p_t^v}\left[\frac{p_t}{\tilde p_t^v}h\right]
    =\mathbb E_{p_t^v}h.
    \label{eq:ideal-is-unit-weights}
\end{align}

The weight is identically one by
\Cref{eq:population-cdm-density-identity}. Accordingly, the following result
bounds observable bias of the approximate teacher using a square-root-density
argument.

\begin{theorem}[Observable error of the teacher flow]\label{thm:IS}
Under the assumptions of \Cref{lem:Et-bound}, suppose
$\int_0^tK_s\,\mathrm ds<\infty$. Then
\begin{align}
    \|\sqrt{p_t}-\sqrt{p_t^v}\|_{L^2(\mathrm dx)}
    \le\tfrac12\int_0^t\eta_s\,\mathrm ds
    \le\tfrac12\int_0^tK_s\,\mathrm ds.
    \label{eq:hellinger-teacher-bound}
\end{align}
If $\sigma_{h,t}^2:=\mathbb E_{p_t}\|h-H_t\|^2<\infty$ and the terminal
ratio $M_t^*<\infty$, then
\begin{align}
    \|H_t-\tilde H_t\|
    \le\tfrac12\sigma_{h,t}(1+\sqrt{M_t^*})
         \int_0^tK_s\,\mathrm ds.
    \label{eq:teacher-observable-error}
\end{align}
Alternatively, for bounded $h$, no density-ratio bound is needed:
\begin{align}
    \|H_t-\tilde H_t\|
    \le\|h\|_\infty\int_0^tK_s\,\mathrm ds.
    \label{eq:bounded-observable-error}
\end{align}
All three conclusions apply at $t=1$ under the same endpoint regularity.
\end{theorem}

\begin{proof}
Let $J_s(x_0)=|\det\nabla\phi_s(x_0)|$ and set
$u_s(x_0)=\sqrt{p_s(\phi_s(x_0))J_s(x_0)}$. The continuity equation gives
\begin{align}
    \partial_su_s=\tfrac12\mathcal E_s(\phi_s(x_0))u_s,
    \qquad
    \sqrt{p_s^v(\phi_s(x_0))J_s(x_0)}=\sqrt{p_0(x_0)}.
\end{align}
By change of variables,
$\|\partial_su_s\|_{L^2(\mathrm dx_0)}=\eta_s/2$.
Integrating from $u_0=\sqrt{p_0}$ and applying Minkowski proves
\Cref{eq:hellinger-teacher-bound}.
Since both densities integrate to one, subtracting $H_t$ from $h$ leaves
the difference of expectations unchanged. Factor
$p_t-p_t^v=(\sqrt{p_t}-\sqrt{p_t^v})(\sqrt{p_t}+\sqrt{p_t^v})$
and apply Cauchy--Schwarz. The inequality
$(\sqrt{p_t}+\sqrt{p_t^v})^2\le(1+\sqrt{M_t^*})^2p_t$
proves \Cref{eq:teacher-observable-error}. For bounded $h$, use instead
$\|\sqrt{p_t}+\sqrt{p_t^v}\|_2\le2$ to obtain
\Cref{eq:bounded-observable-error}.
\end{proof}

Density-ratio assumptions are substantive: even a nonzero, arbitrarily
small translation of a Gaussian makes both supremum ratios infinite.
The bounded-observable result above still applies.

\subsubsection{Error of a Fitted Likelihood Head}\label{app:sec:finite-sample-training-error}

For the vectorized implementation, the scalar instantaneous head is
$D_t^\theta(x)=\mathbf1_d^\top D_\theta^{\mathrm{vec}}(x,t,t)$.

Training typically samples $p_t$, whereas $\varepsilon_{D,t}$ is measured
under $p_t^v$. To connect them, let $Y_t^{\mathrm{vec}}$ be the vector CDM
target with the same fixed teacher, and let
$m_t^{\mathrm{vec}}(x)=\mathbb E[Y_t^{\mathrm{vec}}\mid x_t=x]$.
Its components sum to $m_t$. Whenever the squared losses are finite,

\begin{align}
    R_t&:=\mathbb E_{p_t}\|D_\theta^{\mathrm{vec}}(x_t,t,t)
                         -m_t^{\mathrm{vec}}(x_t)\|^2,\\
    \varepsilon_{D,t}^2&\le dM_t^*R_t.
    \label{eq:vector-risk-transfer}
\end{align}

Here $R_t$ is the excess vector CDM risk over its unrestricted population
optimum, by the conditional-variance identity. The inequality follows from
$|\mathbf1_d^\top a|^2\le d\|a\|^2$ and density transfer. The factor $d$
comes from summing the componentwise head errors. An empirical training loss alone does not bound $R_t$
without a generalization argument.

\begin{corollary}[Population importance-sampling error with a fitted head]
\label{thm:finite-sample-IS}
For $\mathbb E_{p_t}\|h\|^2<\infty$, define
\begin{align}
    \tilde H_t^\theta:=\mathbb E_{p_t^v}
                \left[\frac{p_t}{\tilde p_t^{v,\theta}}h\right],\\
    M_t:=\|p_t/p_t^v\|_{L^\infty(p_t^v)},\qquad
    \kappa_t^\theta:=\|e^{\max(-\Phi_t^\theta,0)}\|_{L^\infty(p_t)}.
\end{align}
Under the assumptions of \Cref{prop:finite-sample-likelihood-error}, with
$M_t,\kappa_t^\theta,A_t^\theta<\infty$,
\begin{align}
    \|H_t-\tilde H_t^\theta\|
    \le\sqrt{\mathbb E_{p_t}\|h\|^2}\,
         \kappa_t^\theta\sqrt{M_t}\,A_t^\theta.
    \label{eq:finite-sample-IS-error}
\end{align}
This population bound also applies at $t=1$; it does not include
finite-particle Monte Carlo error.
\end{corollary}
\begin{proof}
Changing measure gives
$H_t-\tilde H_t^\theta=\mathbb E_{p_t}[h(1-e^{-\Phi_t^\theta})]$.
Using $|1-e^{-u}|\le |u|e^{\max(-u,0)}$, Cauchy--Schwarz, and density
transfer yields
\begin{align}
    \|H_t-\tilde H_t^\theta\|
    \le\sqrt{\mathbb E_{p_t}\|h\|^2}\,
       \kappa_t^\theta\sqrt{M_t}\,\|\Phi_t^\theta\|_{L^2(p_t^v)}.
\end{align}
Apply \Cref{prop:finite-sample-likelihood-error}.
\end{proof}

We next propagate assumed population learning rates; we do not derive these rates from the empirical training objective.

\begin{assumption}[Finite-sample learning rates]
\label{ass:finite-sample-rates}
Denote the number of samples used to train the velocity field $v_t$ as $n_v$,
and denote the number of samples used to train the total-score estimator $D_t^\theta$ as $n_D$. Assume that uniformly over $s\in[0,t]$,
\begin{align}
    \varepsilon_{D,s}(n_D)
    &\leq C_D n_D^{-\beta_D},\\
    e_{s,n_v}
    &\leq C_v n_v^{-\beta_v},\\
    B_{s,n_v}
    &\leq C_B,\\
    S_{s,n_v}
    &\leq C_S,
\end{align}
where $C_D,C_v,C_B,C_S>0,\beta_D,\beta_v>0$ are constants. Further assume that $G_s$ and $M_s^*$ are uniformly bounded
on $[0,t]$, i.e.
\begin{align}
    G_s\leq C_G,
    \qquad
    M_s^*\leq C_M,
    \qquad
    s\in[0,t].
\end{align}
For the importance-sampling conclusion, also assume, uniformly in $n_D$ and
$n_v$,
\begin{align}
    \mathbb E_{p_t}\|h(x_t)\|^2\leq C_h^2,\qquad
    M_t\leq C_{\mathrm{ratio}},\qquad
    \kappa_t^\theta\leq C_\kappa.
\end{align}
All constants are independent of $n_D$ and $n_v$. If the learning-rate
bounds are probabilistic over the training samples, all conclusions below
hold on the same event on which these uniform bounds hold.
\end{assumption}

\begin{corollary}[Sample-size dependence]
\label{cor:sample-size-dependence}
Under the assumptions of \Cref{prop:finite-sample-likelihood-error,thm:finite-sample-IS} and the uniform learning-rate bounds in \Cref{ass:finite-sample-rates}, the local deployed-estimator error
satisfies
\begin{align}
    \mathbb{E}_{p_s^v}
    \left[
    \bigl(\mathcal E_s^\theta(x_s)\bigr)^2
    \right]
    =
    O(n_D^{-2\beta_D})
    +
    O(n_v^{-\beta_v}),
\end{align}
uniformly over $s\in[0,t]$. Consequently, the likelihood-estimation error
satisfies
\begin{align}
    \mathbb{E}_{p_t^v}
    \left[
    \bigl(\Phi_t^\theta(x_t)\bigr)^2
    \right]
    =
    O(t^2 n_D^{-2\beta_D})
    +
    O(t^2 n_v^{-\beta_v}),
\end{align}
and the resulting importance-sampling error satisfies
\begin{align}
    \|H_t-\tilde H_t^\theta\|
    =
    O(t n_D^{-\beta_D})
    +
    O(t n_v^{-\beta_v/2}).
\end{align}
In particular, if both stages achieve root-sample rates,
$\beta_D=\beta_v=1/2$, then
\begin{align}
    \|H_t-\tilde H_t^\theta\|
    =
    O(t n_D^{-1/2})
    +
    O(t n_v^{-1/4}).
\end{align}
At $t=1$, this becomes
\begin{align}
    \|H_1-\tilde H_1^\theta\|
    =
    O(n_D^{-1/2})
    +
    O(n_v^{-1/4}).
\end{align}
\end{corollary}

\begin{proof}
Substituting \Cref{ass:finite-sample-rates} into
\Cref{eq:theta-local-error-bound} gives, uniformly over $s\in[0,t]$,
\begin{align}
    \mathbb E_{p_s^v}|\mathcal E_s^\theta|^2
    \le 2C_D^2n_D^{-2\beta_D}
       +4C_MC_v(C_S+C_BC_G)n_v^{-\beta_v}.
\end{align}
Integrating this bound in \Cref{prop:finite-sample-likelihood-error}
gives the stated mean-square likelihood rate. Equivalently,
$A_t^\theta=O(tn_D^{-\beta_D})+O(tn_v^{-\beta_v/2})$; substituting this
into \Cref{thm:finite-sample-IS} gives the importance-sampling rate.
Set $\beta_D=\beta_v=1/2$ for the final statements.
\end{proof}

\section{Expanded Discussion of Related Works}
\label{app:sec:related_work}

The following are three dominant approaches used in computational chemistry for training neural networks that support efficient sampling \textit{and} likelihood evaluation.
Once such neural networks are trained, they can be used as a \textit{proposal} distribution, known in computational chemistry as a Boltzmann generator, for self-normalized importance sampling of a Boltzmann distribution.

\paragraph{Discrete and Continuous Normalizing Flows}
% \paragraph{Boltzmann Generators with Exact Likelihood.}
Boltzmann Generators (BGs) are tractable proposal distributions used in importance sampling, allowing both efficient sampling and likelihood evaluation. They were first applied to molecular systems by \cite{noé2019boltzmanngeneratorssampling}, who parameterized them using neural networks in the form of Normalizing Flows (NFs).
Subsequent work has evolved along two main directions. The first replaces NFs with their continuous counterpart, Continuous Normalizing Flows \citep[CNFs][]{chen2019neuralordinarydifferentialequations} that can be parameterized using more flexible model architectures. For instance, \cite{pmlr-v119-kohler20a,NEURIPS2023_bc827452} develop equivariant CNFs tailored to molecular systems. However, this approach does not scale well: computing the log-likelihood requires integrating the divergence of the velocity field (as in~\eqref{eq:augmented-pf-ode}), which becomes prohibitively expensive in high dimensions.
The second direction instead focuses on improving the expressivity and sampling efficiency of standard NFs. Recent work, such as TarFlow \citep{zhai2025normalizingflowscapablegenerative}, introduces scalable transformer-based architectures, achieving state-of-the-art performance. Building on this line, \cite{tan2026scalableequilibriumsamplingsequential,tan2026amortizedsamplingtransferablenormalizing,schebek2025scalableboltzmanngeneratorsequilibrium} pushed the boundaries of NFs for atomic systems. In particular, \cite{tan2026scalableequilibriumsamplingsequential} also proposes an inference-time scaling strategy to further improve sampling efficiency.

\paragraph{Energy-based Diffusion Models.} Training diffusion models with energy parameterization acts as a bridge to alleviate the computational overhead of solving the divergence integration in CNFs, where the learned energies define Boltzmann distributions on different noise levels. Recent works focus on improving the training of energy-based diffusion models beyond Denoising Score Matching. \cite{aggarwal2025boltzncelearninglikelihoodsboltzmann,ouyang2026diffusiveclassificationlosslearning} incorporate Noise-Contrastive Estimation \citep{gutmann2012nce}, where \cite{aggarwal2025boltzncelearninglikelihoodsboltzmann} treats the learned energy of clean data as an approximate normalized likelihood and \cite{ouyang2026diffusiveclassificationlosslearning} leverages the learned intermediate energies for further inference-time control. \cite{he2026rneplugandplaydiffusioninferencetime} improves the training by imposing a Bayes' rule constraint, with an application of the EBMs for free energy estimation. \cite{plainer2026consistentsamplingsimulationmolecular} regularizes the EBMs to enforce the Fokker-Planck Equation, and uses the learned EBMs as a force field. \cite{yu2026learningenergybasedmodelsstochastic} proposes a unified framework for learning EBMs. In contrast to the standard data-driven setting, \cite{ouyang2026bnemboltzmannsamplerbased} proposes an energy matching objective for learning a Boltzmann neural sampler when data are unavailable, relying instead on the unnormalized target density.

\paragraph{Flow Maps with Approximate Invertibility.} 
% \paragraph{Boltzmann Generators with Approximated Invertibility.} 
NFs enable exact likelihood evaluation through their invertibility, which also limits their flexibility. To enable the use of more flexible network architectures, Free-form Flows \citep{draxler2024freeformflowsmakearchitecture} propose to softly enforce invertibility through an additional reconstruction loss. To extend this approach to recent few-step generative models, FALCON \citep{rehman2025falconfewstepaccuratelikelihoods} proposes to make flow maps invertible by additionally imposing reconstruction losses on the jump between two times $t<s$. FALCON enables fast yet accurate likelihood estimation with flow maps and shows its effectiveness and scalability on molecular systems. However, FALCON requires the exact log-determinant of the Jacobian of the flow map for likelihood evaluation; over $T$ steps, this is computationally expensive, resulting in $\mathcal{O}(Td)$ function evaluations and an additional $\mathcal{O}(Td^3)$ cost.

\section{Experimental Details for Molecules}\label{app:sec:molecule-exp-details}

The molecular experiments are conducted on one \texttt{NVIDIA RTX PRO 6000} GPU.

\subsection{Details of Training}
\paragraph{Model Architecture.}
Following \cite{rehman2025falconfewstepaccuratelikelihoods}, we adopt the same Diffusion Transformer (DiT) backbone for both \scalloptt and F2D2, with additional time conditioning adapted for flow maps, across all benchmarks. We use the vectorized divergence matching objectives stated in \Cref{sec:vectorized-divergence-matching} for both \scalloptt and F2D2. To obtain a parameter count similar to FALCON, we use 6 DiT blocks as shared parameters, and 3 additional DiT blocks for the log-density-change readout head.
The details of the backbone configuration are in \Cref{tab:model_arch_details} below.
\begin{table}[h]
    \centering
    \caption{Overview of \scalloptt and F2D2 configurations across benchmarks.}
    \label{tab:model_arch_details}
    \setlength{\tabcolsep}{5pt}
    \begin{tabular}{lcccccc}
        \toprule
        System & Hid. Size & Shared Blocks & log-density Blocks & Heads & Cond. Dim. & Param. (M) \\
        \midrule
        ALA-2     & 192 & 6 &3 & 6 & 64 & 6.25 \\
        ALA-3           & 192 &  6 &3 & 6 & 64 & 6.25 \\
        ALA-4  & 192 &  6 &3 & 6 & 64 & 6.25 \\
        ALA-6          & 192 &  6 &3 & 6 & 64 & 6.25 \\
        \bottomrule
    \end{tabular}
\end{table}

\paragraph{Hyperparameters for Training.}
We train all models with an exponential moving average (EMA) on
the weights using a decay rate of 0.999. We use the AdamW optimizer with learning rate $10^{-4}$,
gradient-norm clip $10$, and weight decay $10^{-4}$. We employ a square-root annealing learning rate schedule with
a warm-up phase covering $2.5$\% of the training iterations. All training runs use a batch size of $1000$, with $4\times10^6$ updates in total.

\paragraph{Data Augmentation.} To preserve SE(3) symmetries, we employ center-of-mass adjustment and standardize the data using the mean and standard deviation estimated from the training set. We further apply a random rotation to the data every time a sample is drawn, to impose a soft constraint on rotational equivariance.

\subsection{Details of Inference}

\paragraph{Hyperparameters.}
In all evaluations, we draw $2\times10^5$ samples from the model and employ SNIS for resampling. In SNIS, logit values are clipped at 0.002, as in standard baselines \citep{tan2026scalableequilibriumsamplingsequential,tan2026amortizedsamplingtransferablenormalizing,rehman2025falconfewstepaccuratelikelihoods}. We use batch size $10^4$ for \scalloptt, F2D2, and SBG-IS during inference to maximize GPU utilization, for all benchmarks. However, FALCON only fits with significantly smaller batch sizes due to the need for exact Jacobian calculation. On a single \texttt{NVIDIA RTX PRO 6000} GPU, FALCON only fits batch sizes (1000, 1000, 500, 100) for ALA-(2, 3, 4, 6), respectively. We employ a linear schedule for all generation in \scalloptt and F2D2.

\paragraph{Algorithms.} We provide the pseudocode for inference of \scalloptt and F2D2 in \Cref{alg:scallop-inference} below:
\begin{algorithm}[H]
\caption{Inference for \texttt{SCALLOP} and F2D2}
\label{alg:scallop-inference}
\begin{algorithmic}[1] % [1] adds line numbers
\State \textbf{Input:} base distribution $p_0$, trained likelihood flow map $f_\theta=[u_\theta;D_\theta]$, time schedule $\{t_i\}_{i=0}^N$
\State \textbf{Output:} Samples $x_N \sim p_1^\theta$ with associated estimated log-likelihoods $\ell_N$

\State $x_0 \sim p_0$
\State $\ell_0 \leftarrow \log p_0(x_0)$

\For{$i = 1$ \textbf{to} $N$}
    \State $(t, s) \leftarrow (t_{i-1}, t_i)$
    \State $x_i \leftarrow x_{i-1} + (s-t)u_\theta(x_{i-1}, t, s)$
    
    \If{$D_\theta(x_{i-1}, t, s) \in \mathbb{R}$}
        \State $\ell_i \leftarrow \ell_{i-1} + (s-t)D_\theta(x_{i-1}, t, s)$
    \ElsIf{$D_\theta(x_{i-1}, t, s) \in \mathbb{R}^d$}
        \State $\ell_i \leftarrow \ell_{i-1} + (s-t)D_\theta(x_{i-1}, t, s) \cdot \mathbf{1}_d$\Comment{for vectorized divergence}
    \EndIf
\EndFor

\State \Return $x_N, \ell_N$
\end{algorithmic}
\end{algorithm}

\subsection{Details of Alanine Systems}
For all the alanine systems, including ALA-(2, 3, 4, 6), we follow the same setting as \cite{tan2026scalableequilibriumsamplingsequential,rehman2025falconfewstepaccuratelikelihoods} to split the data into training, validation, and test sets, with sizes $10^5$, $2\times10^4$, and $10^4$. Data are obtained from MD for 1 $\mu$s, with a timestep of 1 fs, at 300 K, 310 K, 300 K, 310 K for ALA-(2, 3, 4, 6) respectively, provided by the official repository of \cite{tan2026scalableequilibriumsamplingsequential}. The force fields for ALA-(2, 3, 4, 6) are Amber ff99SBildn, Amber 14, and Amber ff99SBildn, respectively.

\paragraph{Data provenance and licenses.}
We obtain the molecular trajectories from the \href{https://huggingface.co/datasets/transferable-samplers/sequential-boltzmann-generators-data}{Sequential Boltzmann Generators data release}, whose dataset card specifies the MIT License. The release credits the ALA-2 (Ace-A-Nme) trajectories to \citet{kohler2021smooth}, the ALA-4 (Ace-AAA-Nme) trajectories to \citet{dibak2021temperature}, and the ALA-3 (AAA) and ALA-6 (AAAAAA) trajectories to \citet{tan2026scalableequilibriumsamplingsequential}. The core \href{https://github.com/transferable-samplers/transferable-samplers}{Transferable Samplers codebase} is also MIT-licensed; third-party components retain the separate licenses identified in its \href{https://github.com/transferable-samplers/transferable-samplers/blob/main/NOTICE}{NOTICE file} and individual file headers.

\subsection{Details of Metrics}

\paragraph{Effective Sample Size.} ESS \citep{kish1957confidence} quantifies the resampling efficiency of SNIS. Given $N$ generated samples $x_{1:N}$ with associated unnormalized importance weights $w_{1:N}$, the ESS is defined as:
\begin{align}
    \mathrm{ESS}(w_{1:N})=\frac{(\sum_{n=1}^N w_n)^2}{N\sum_{n=1}^N w_n^2}\in[0,1],
\end{align}
where a higher ESS generally reflects better SNIS performance. Therefore, ESS can be treated as a proxy for evaluating how accurate the estimated log densities of samples are, given that the model can generate high-quality samples.

\paragraph{2-Wasserstein Distance.} The 2-Wasserstein distance ($\mathcal{W}_2$) is a measure of difference between two probability distributions $p$ and $q$ defined on $\mathbb{R}^m$,
defined as follows
\begin{align}
    \mathcal{W}_2(p, q)^2=\min_{\pi\in\Pi(p,q)}\int_{\mathbb{R}^m\times\mathbb{R}^m}c(x, y)\dd\pi(x, y),
\end{align}
where $\Pi$ is the set of admissible couplings and $c:\mathbb{R}^m\times\mathbb{R}^m\rightarrow\mathbb{R}$ is the cost function.
\begin{itemize}[leftmargin=*]
    \item 2-Wasserstein Distance of energy ($\ewas$) measures the energy histogram difference between generated samples and reference ones, which admits $m=1$ and uses $c(x, y)=|x-y|^2$.
    \item 2-Wasserstein Distance of torus ($\twas$) measures the torsional-angle histogram difference between generated samples and reference ones. For an ALA-N system, there are $2(N-1)$ torsional angles in total. Let $\mathrm{Dihedral}(x)=(\phi_1,\psi_1,...,\phi_{N-1},\psi_{N-1})\in[-\pi,\pi]^{2(N-1)}$ be the torsional angles of a sample $x$.     
    To account for the fact that these torsional angles lie on the torus $[-\pi,\pi]^{2(N-1)}$ with periodic boundaries, we use:
    \begin{align}
        c(x, y)= \sum_{i=1}^{2(N-1)}\left[(\mathrm{Dihedral}(x)_i-\mathrm{Dihedral}(y)_i+\pi)\space\mathrm{mod}\space2\pi-\pi\right]^2
    \end{align}
    as the cost function.
\end{itemize}

\section{Additional Experimental Results for Molecules}

\newcommand{\imgw}{0.18\textwidth} % one column width; tweak if overfull
\begin{figure}[t]
  \centering
  \setlength{\tabcolsep}{3pt}
  \renewcommand{\arraystretch}{1.0}
  \begin{tabular}{@{}c@{\hspace{3pt}}c@{\hspace{3pt}}c@{\hspace{3pt}}c@{\hspace{3pt}}c@{}}
    % ----- row 1: ALDP only at (1,1) -----
    \includegraphics[width=\imgw]{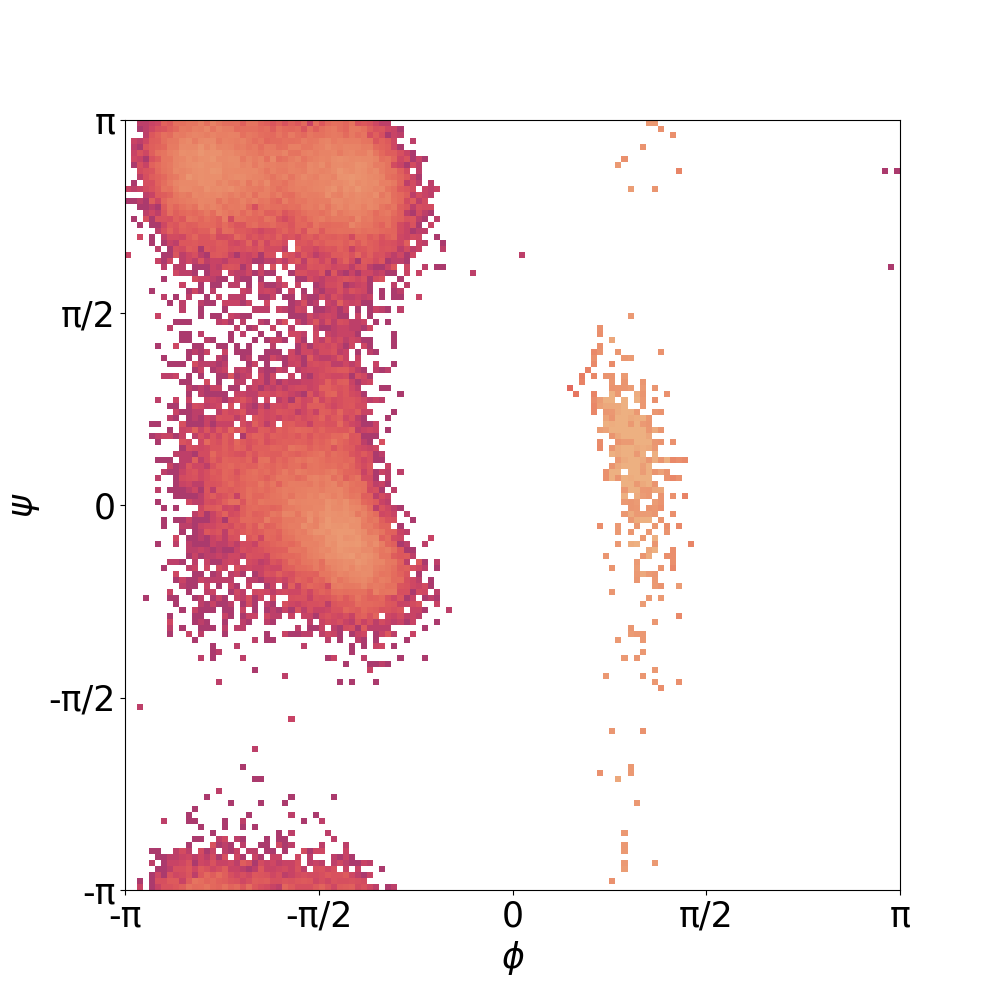} &
    & & & \\
    \includegraphics[width=\imgw]{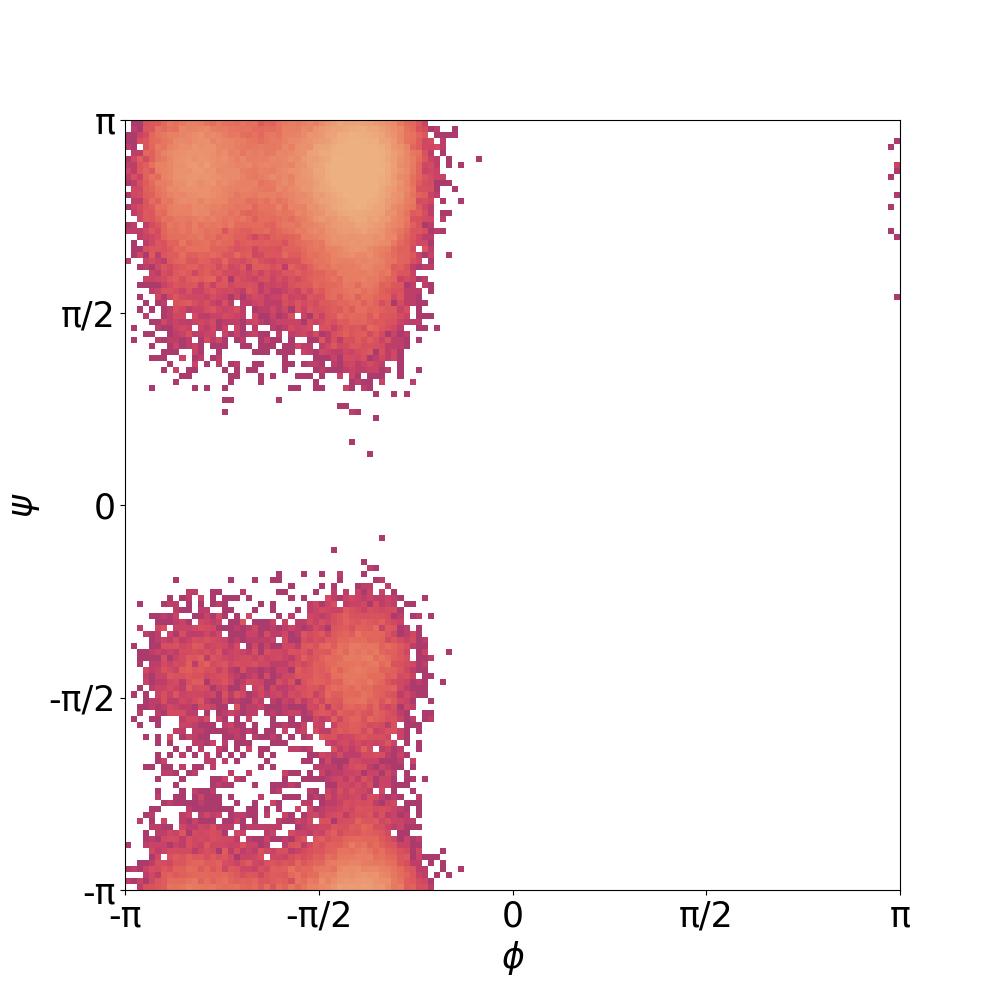} &
    \includegraphics[width=\imgw]{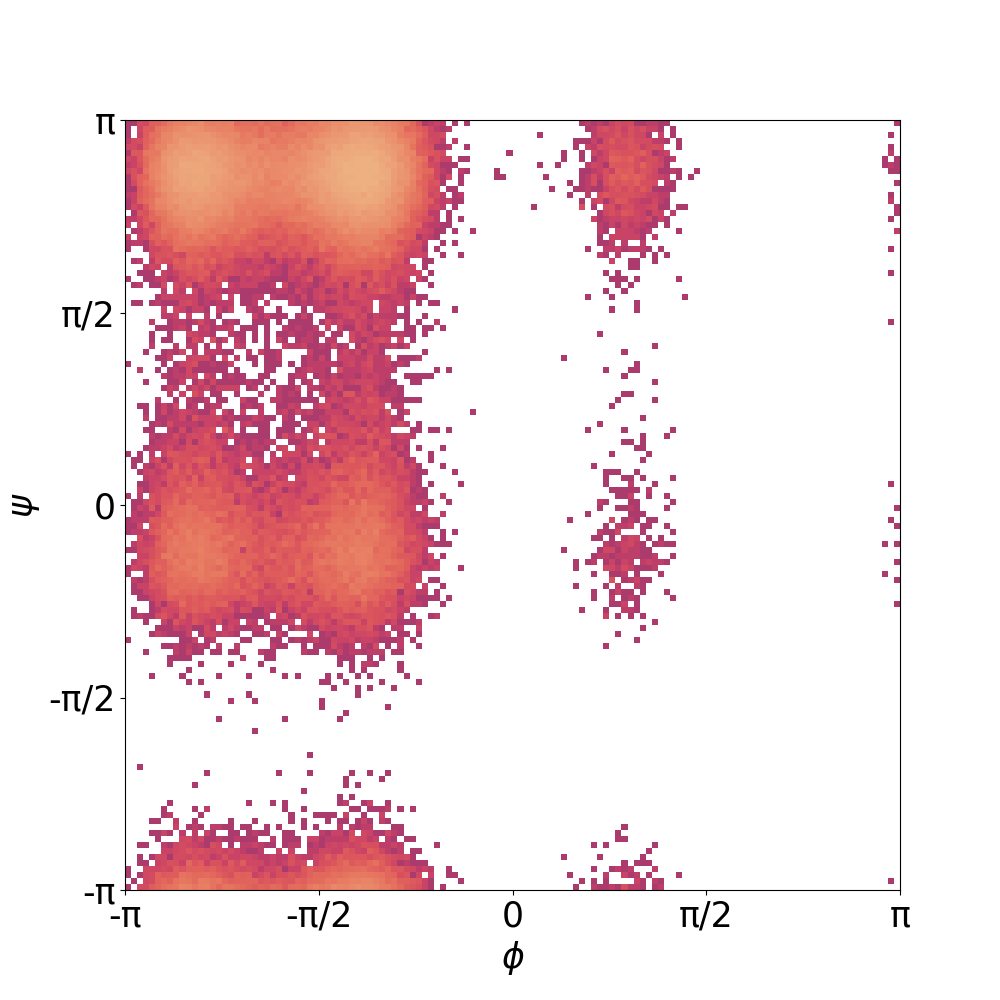} & & &
    \\
    \includegraphics[width=\imgw]{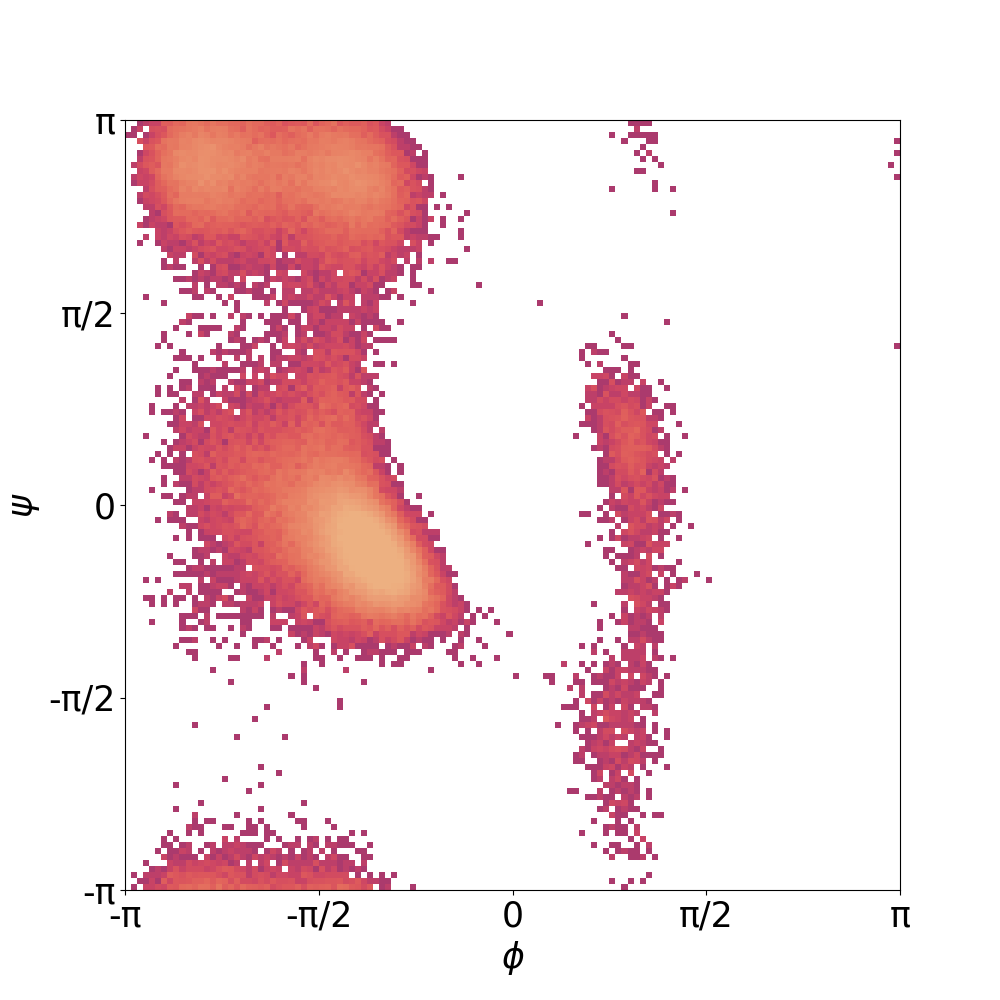} &
    \includegraphics[width=\imgw]{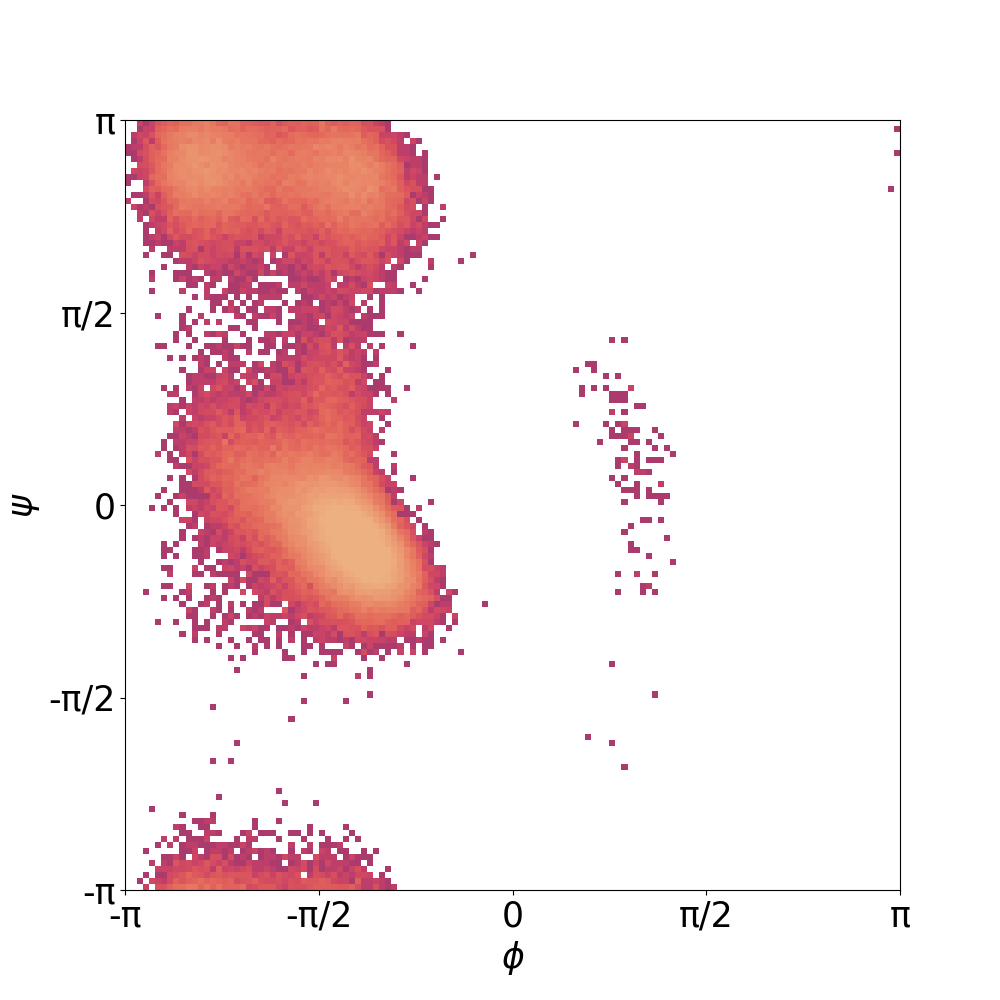} &
    \includegraphics[width=\imgw]{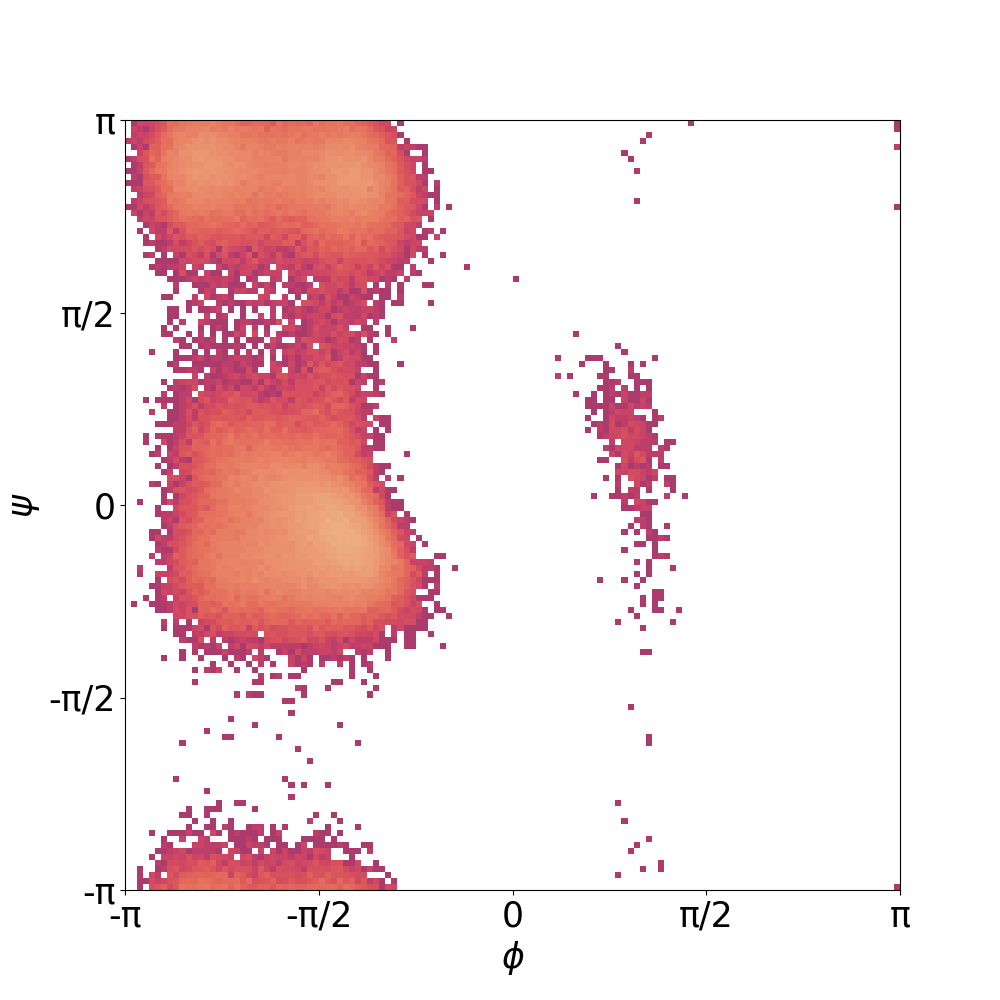} &
    &
    \\
    % ----- row 4: ala6 energy at (4,1); rest free -----
    \includegraphics[width=\imgw]{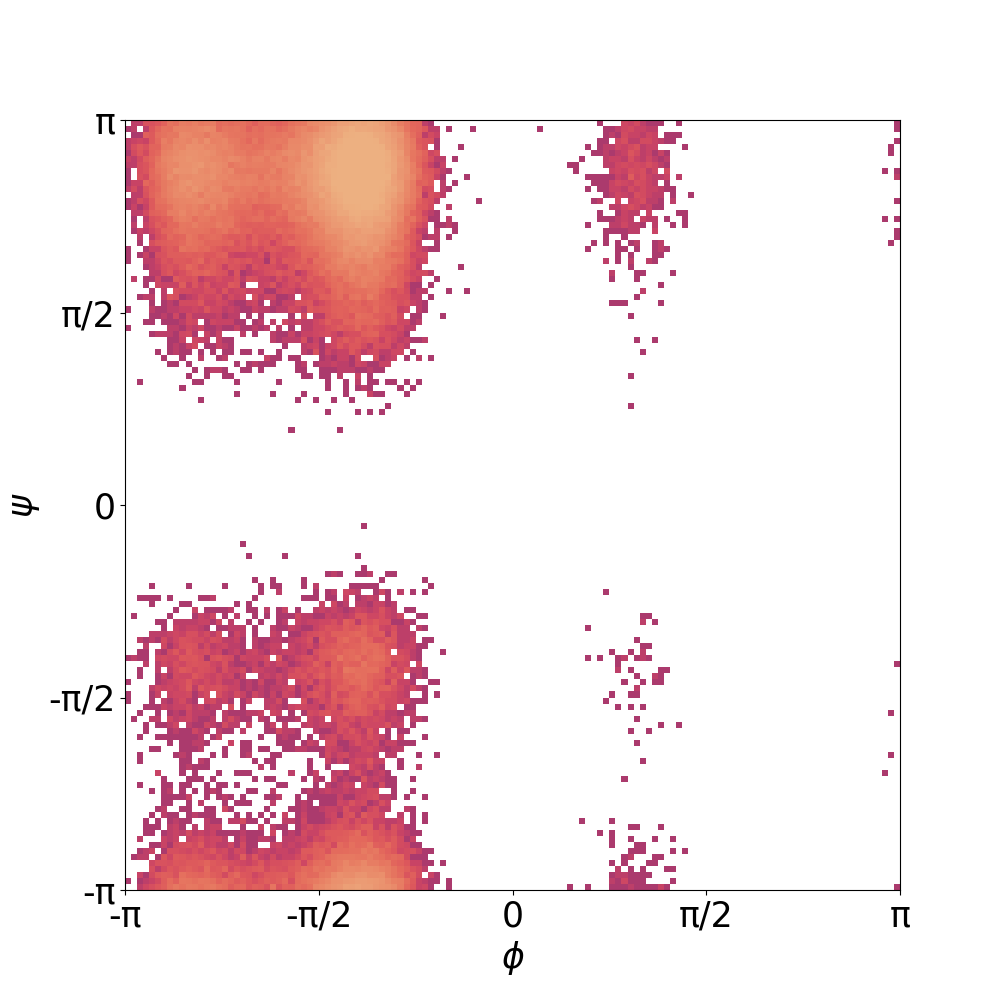} &
    \includegraphics[width=\imgw]{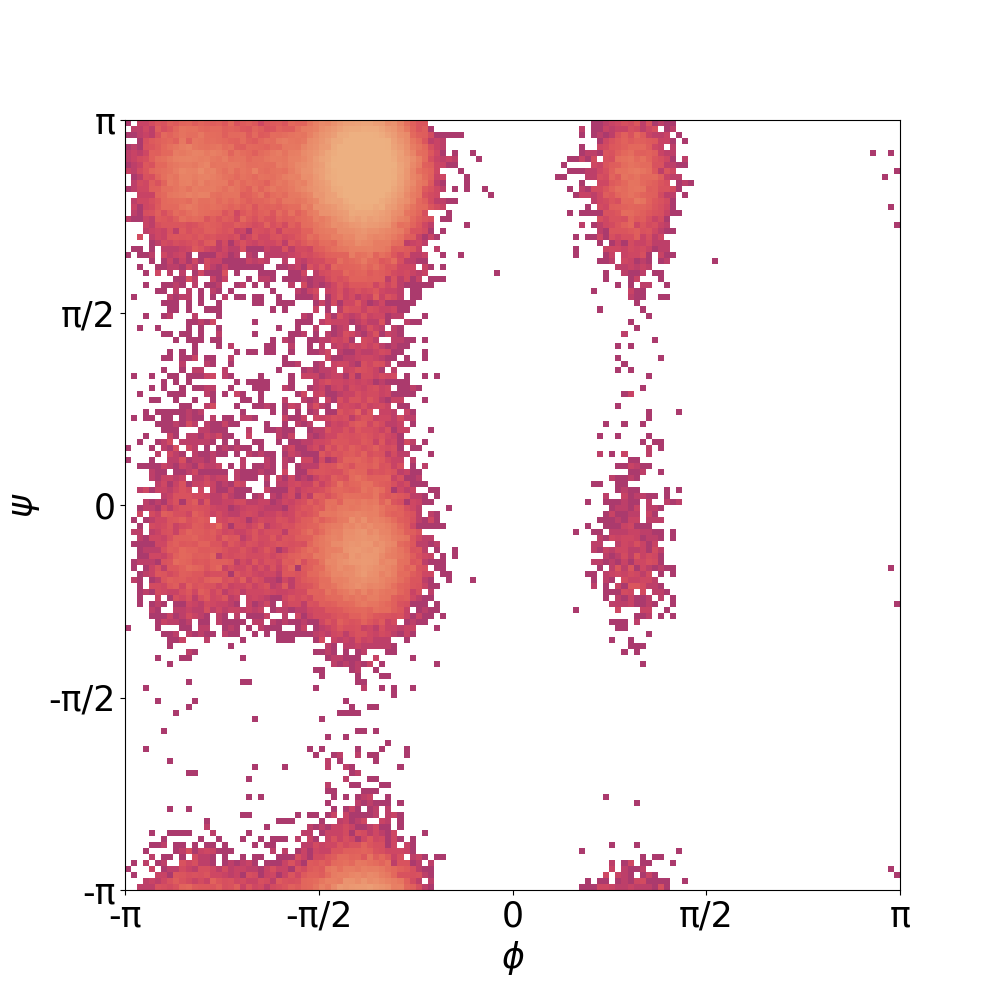}
    & \includegraphics[width=\imgw]{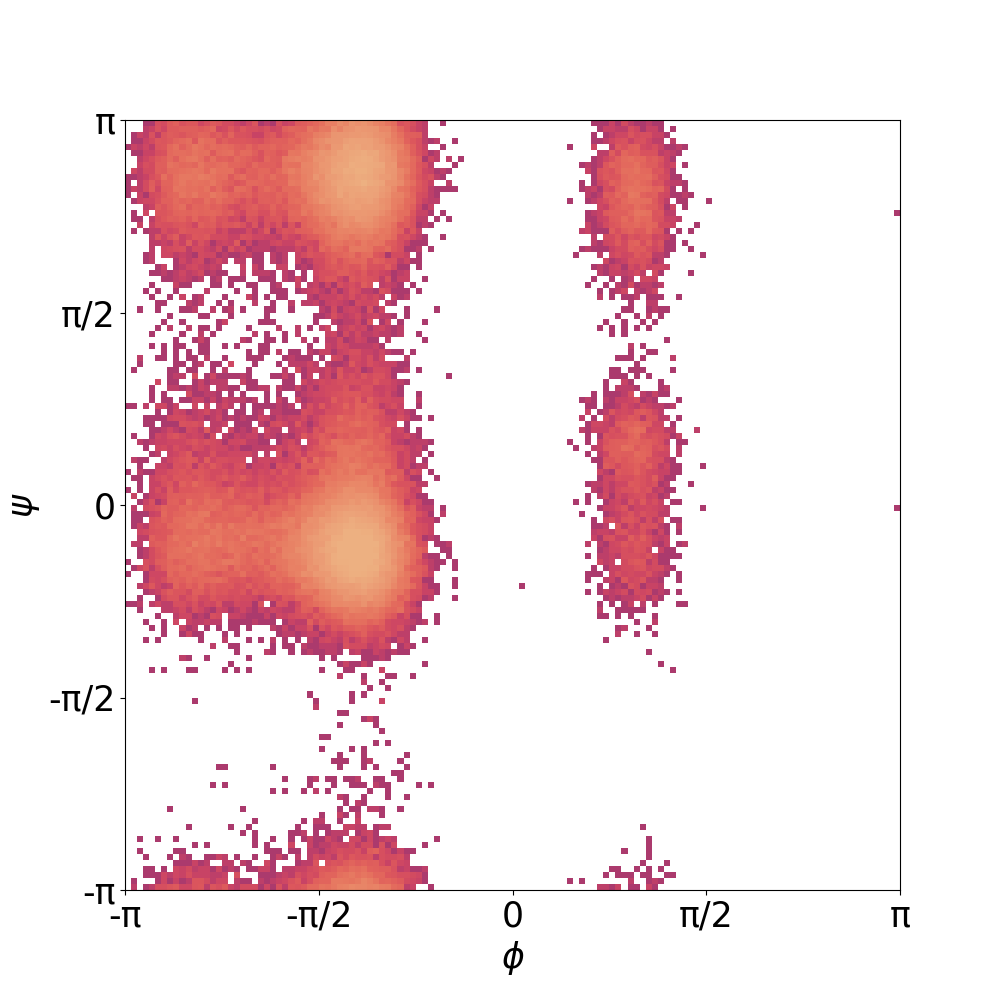}&\includegraphics[width=\imgw]{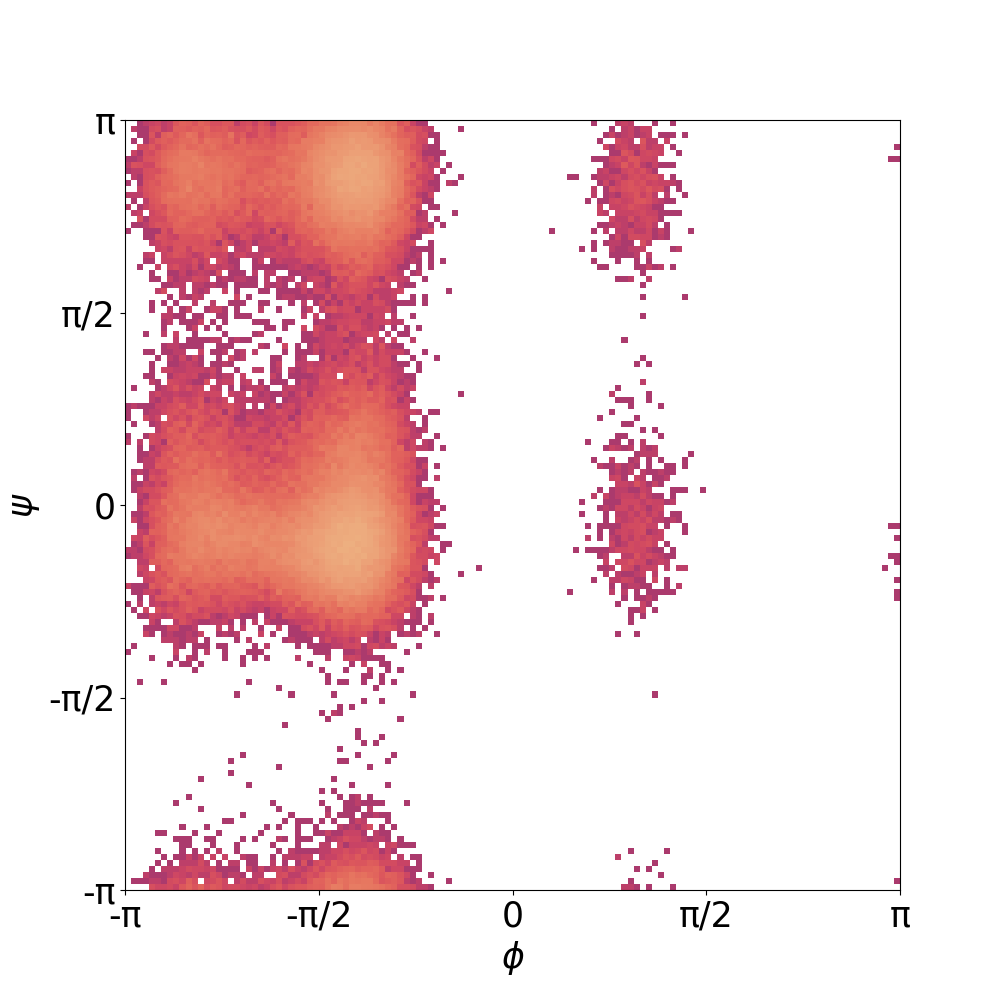} & \includegraphics[width=\imgw]{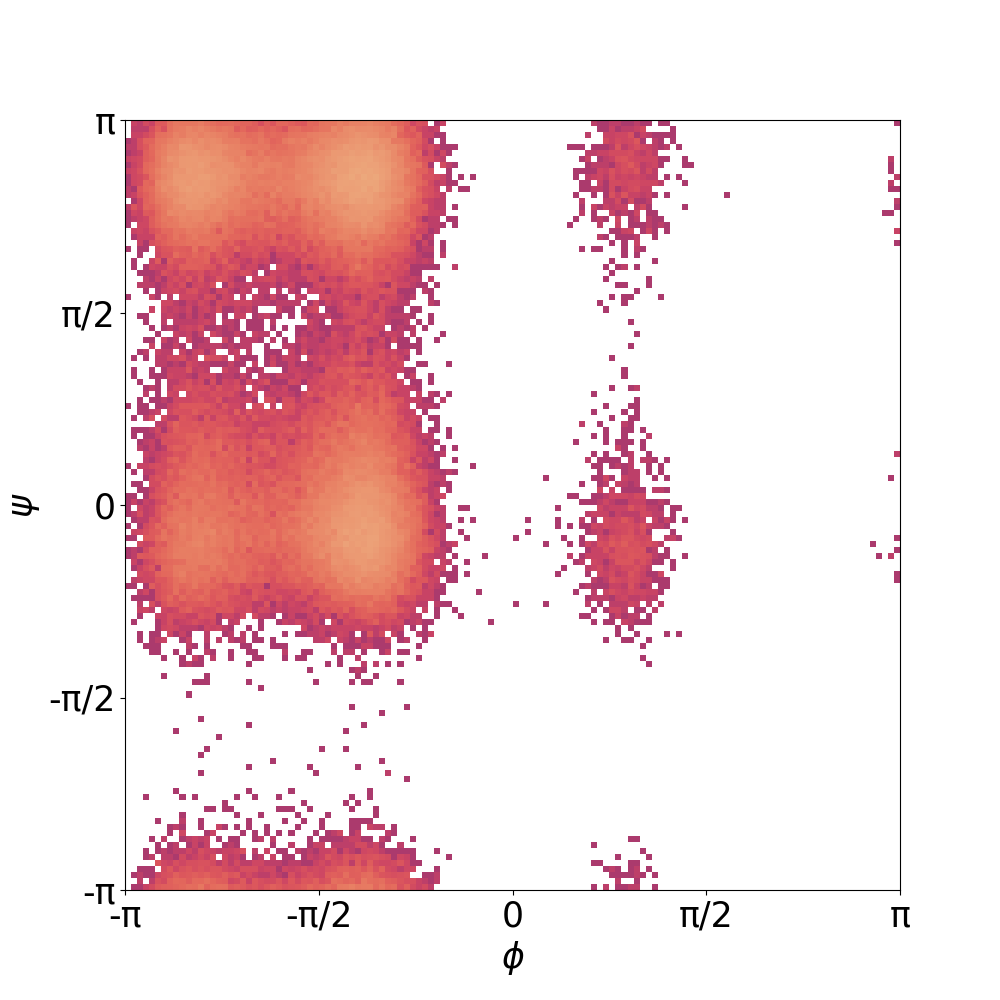}\\
  \end{tabular}
  \caption{Ramachandran plots for ALA-$N$ systems, from ALA-2 (TOP) to ALA-6 (BOTTOM). Torsional angle indices are ordered from LEFT to RIGHT.}
  \label{fig:all-diheral}
\end{figure}

\begin{figure}[t]
  \centering
  \setlength{\tabcolsep}{3pt}
  \renewcommand{\arraystretch}{1.0}
  \begin{tabular}{@{}c@{\hspace{3pt}}c@{\hspace{3pt}}c@{\hspace{3pt}}c@{\hspace{3pt}}c@{}}
    % ----- row 1: ALDP only at (1,1) -----
    \includegraphics[width=\imgw]{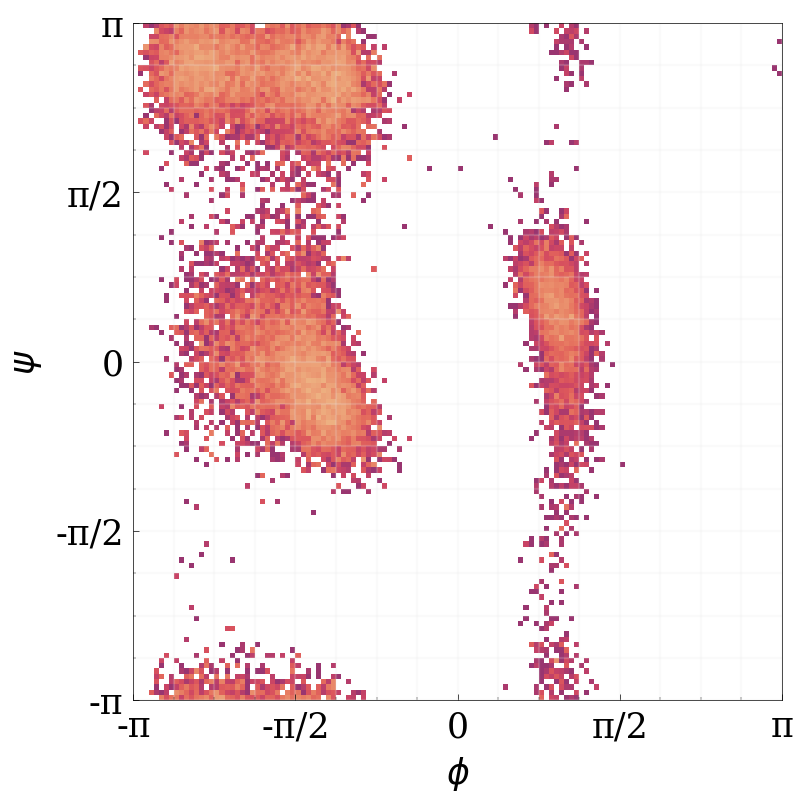} &
    & & & \\
    \includegraphics[width=\imgw]{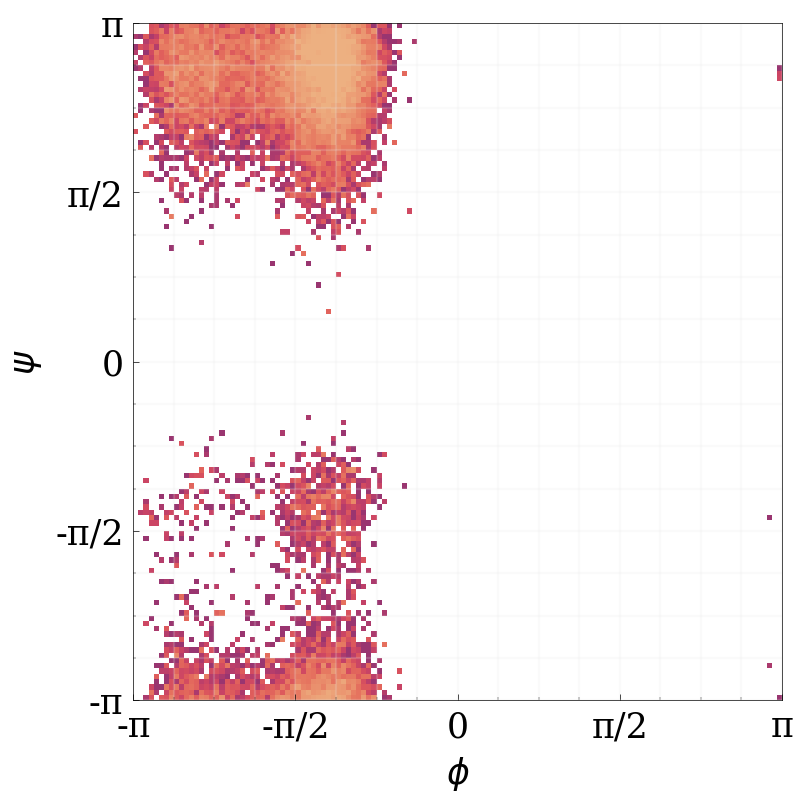} &
    \includegraphics[width=\imgw]{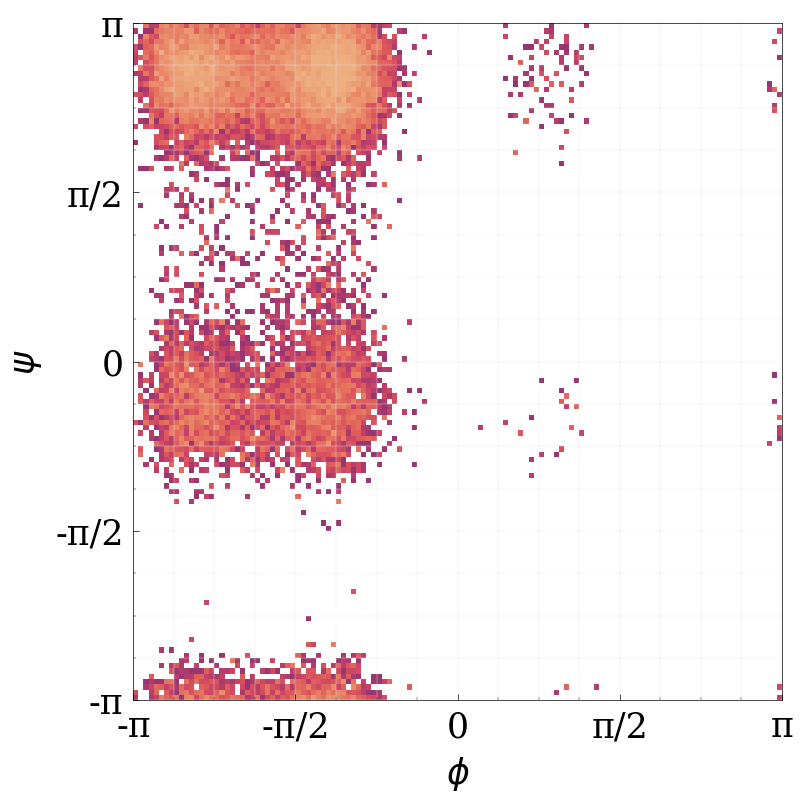} & & &
    \\
    \includegraphics[width=\imgw]{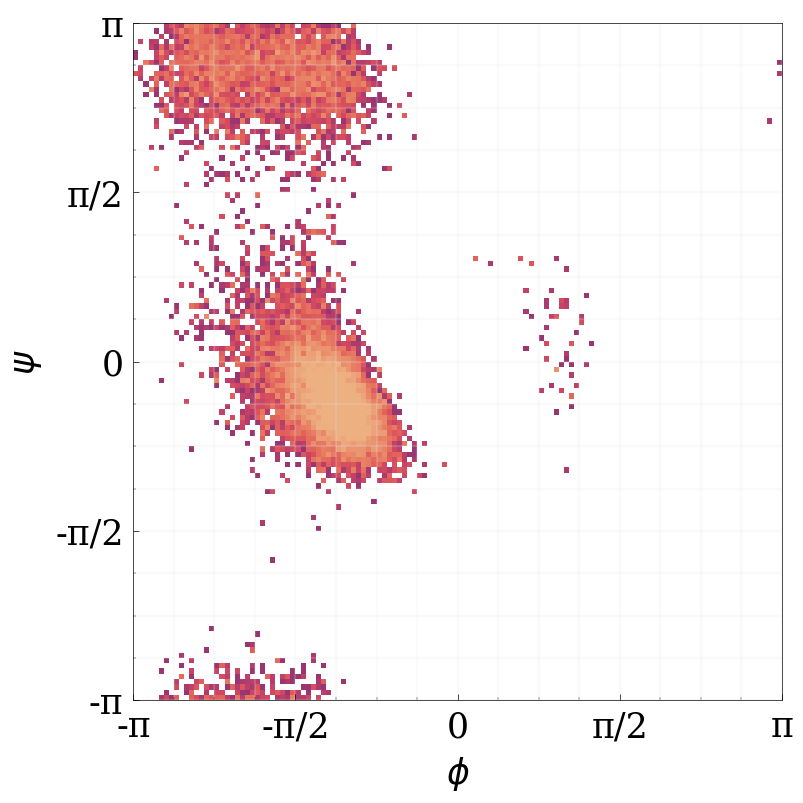} &
    \includegraphics[width=\imgw]{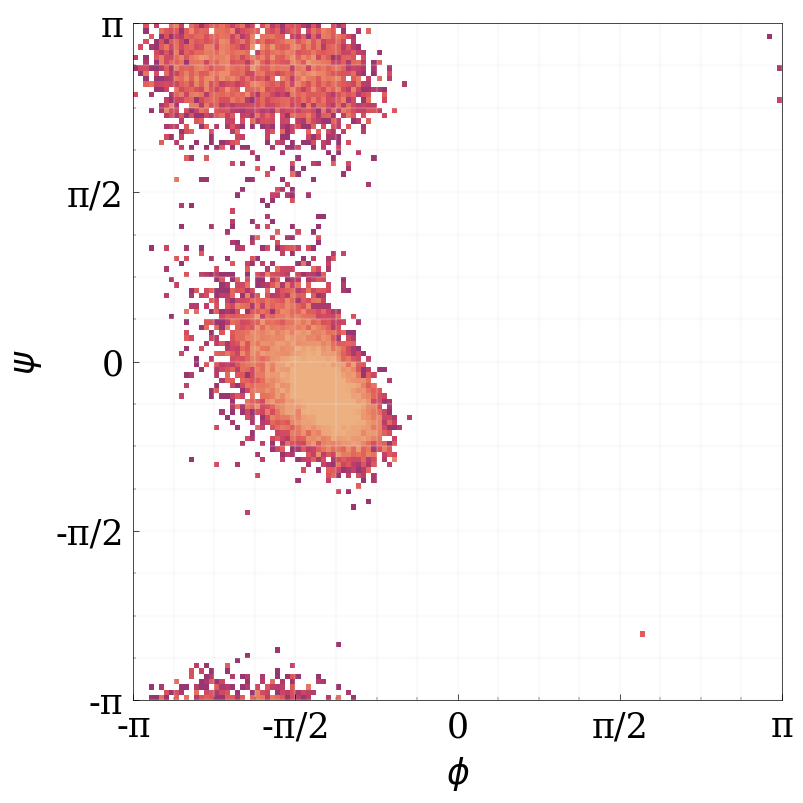} &
    \includegraphics[width=\imgw]{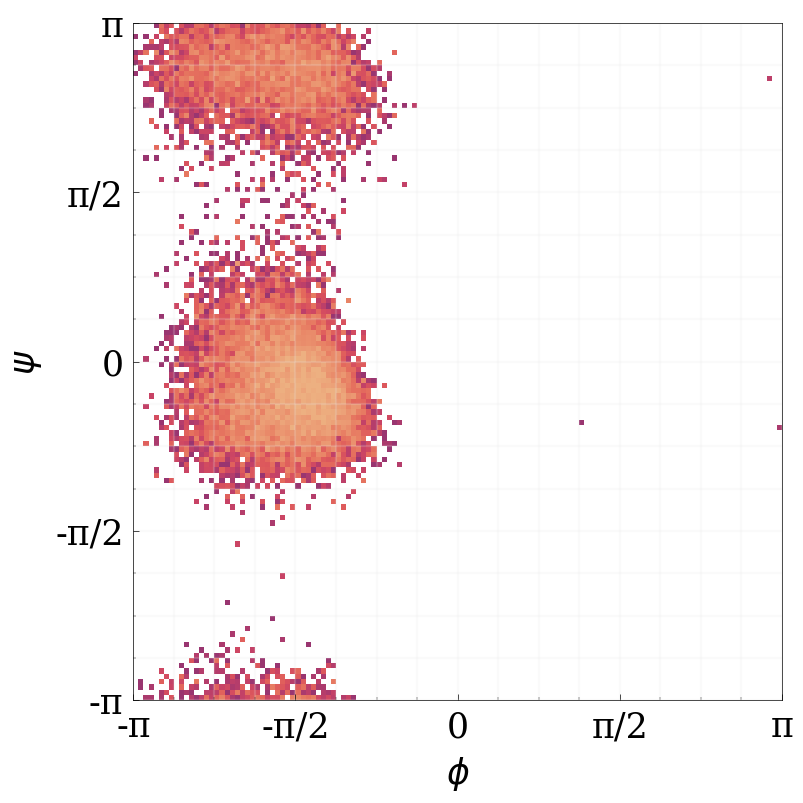} &
    &
    \\
    % ----- row 4: ala6 energy at (4,1); rest free -----
    \includegraphics[width=\imgw]{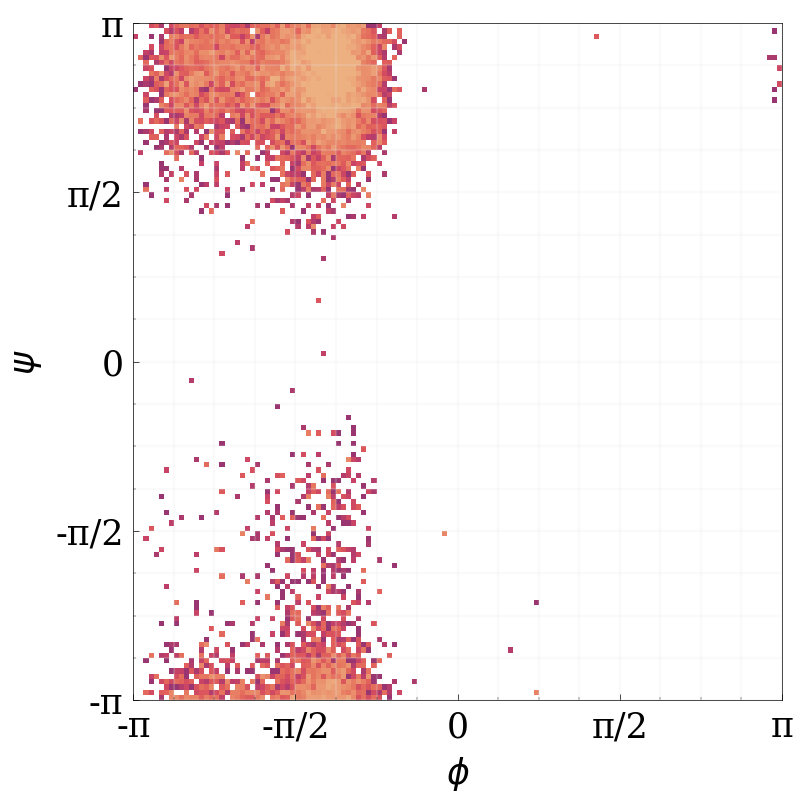} &
    \includegraphics[width=\imgw]{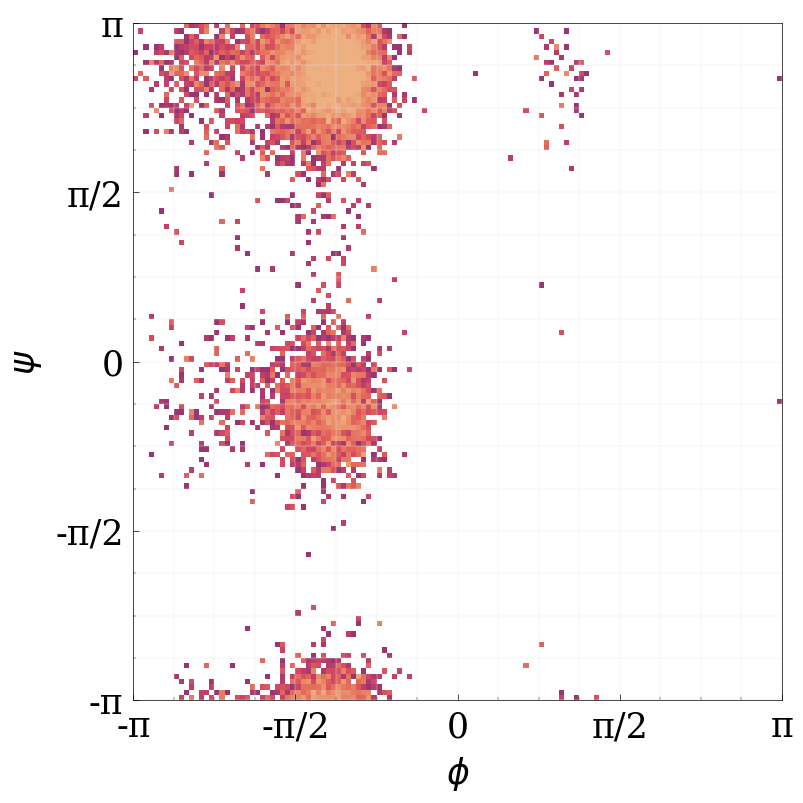}
    & \includegraphics[width=\imgw]{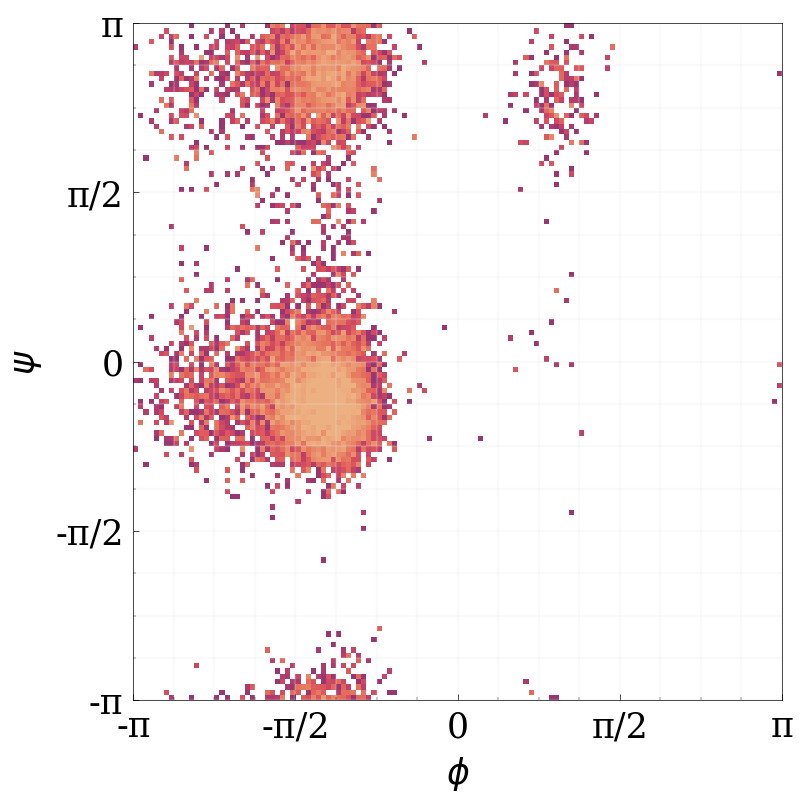}&\includegraphics[width=\imgw]{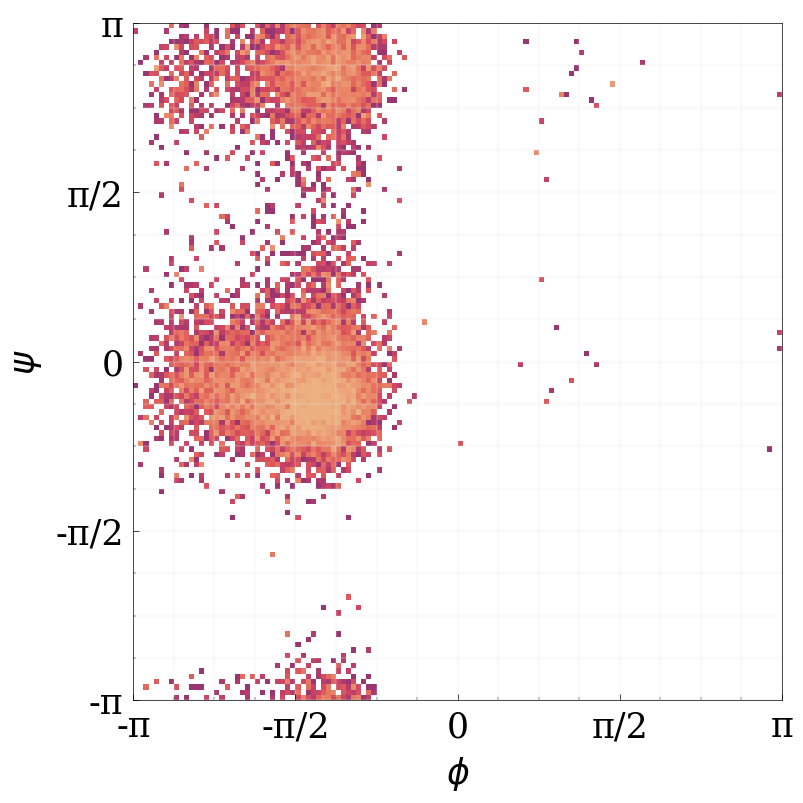} & \includegraphics[width=\imgw]{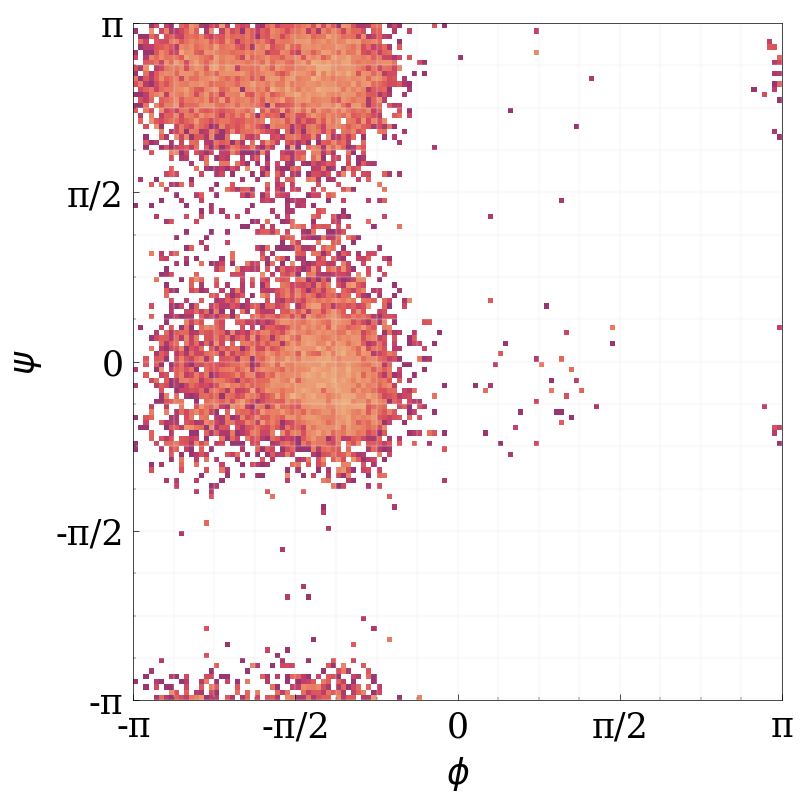}\\
  \end{tabular}
  \caption{Ramachandran plots for ALA-$N$ data generated by \scalloptt, from ALA-2 (TOP) to ALA-6 (BOTTOM). Torsional angle indices are ordered from LEFT to RIGHT.}
  \label{fig:scallop-diheral}
\end{figure}

\begin{figure}[t]
  \centering
  \setlength{\tabcolsep}{3pt}
  \renewcommand{\arraystretch}{1.0}
  \begin{tabular}{@{}c@{\hspace{3pt}}c@{\hspace{3pt}}c@{\hspace{3pt}}c@{\hspace{3pt}}c@{}}
    % ----- row 1: ALDP only at (1,1) -----
    \includegraphics[width=\imgw]{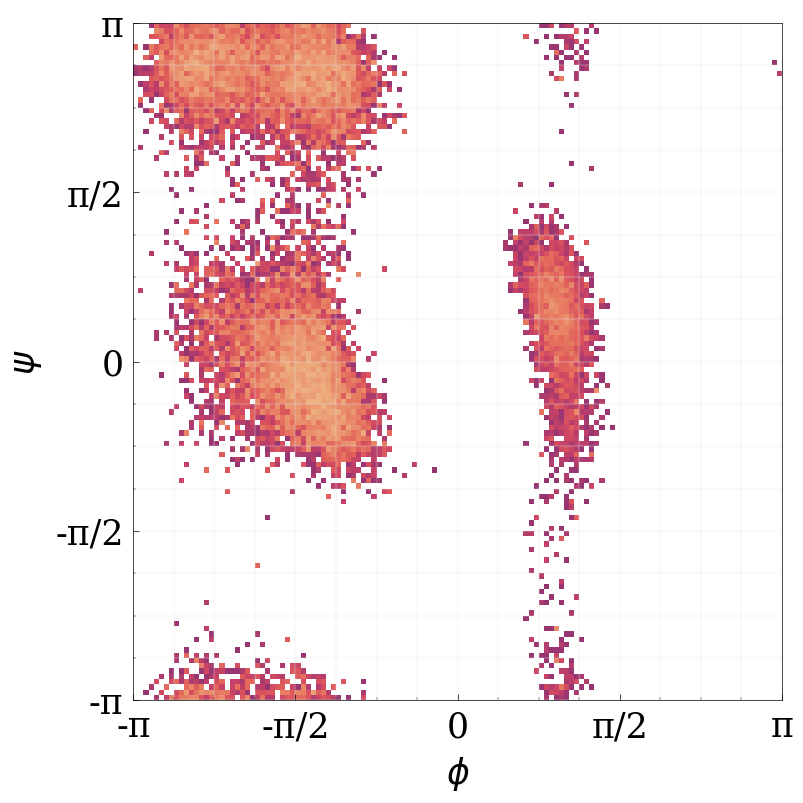} &
    & & & \\
    \includegraphics[width=\imgw]{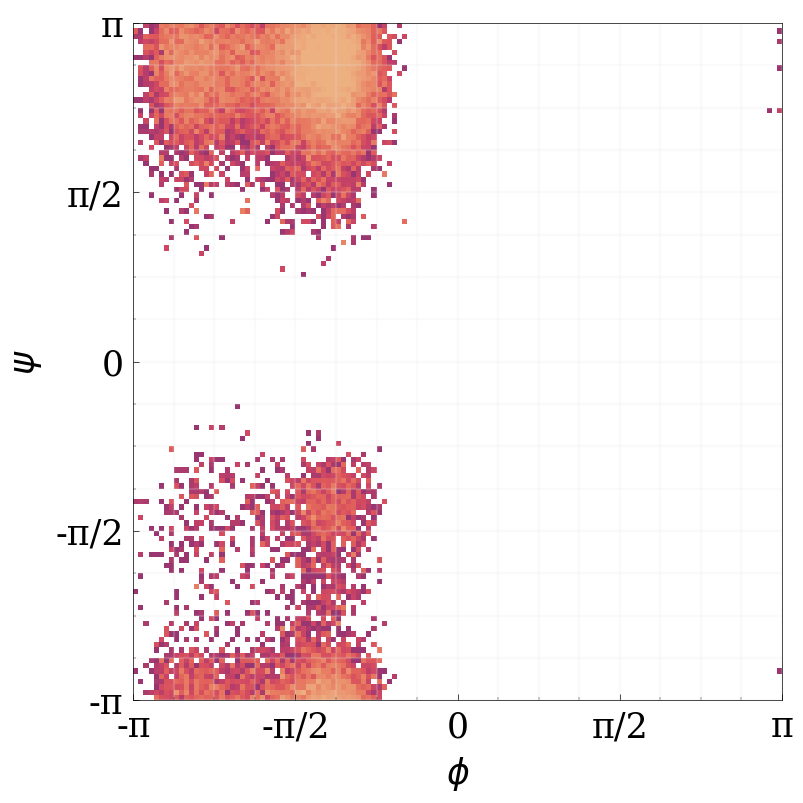} &
    \includegraphics[width=\imgw]{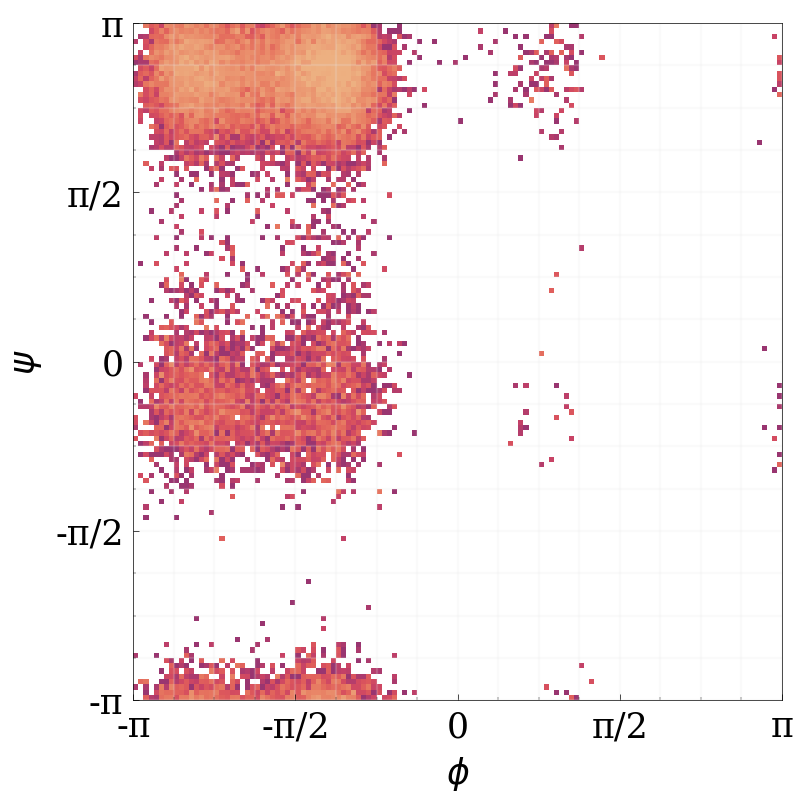} & & &
    \\
    \includegraphics[width=\imgw]{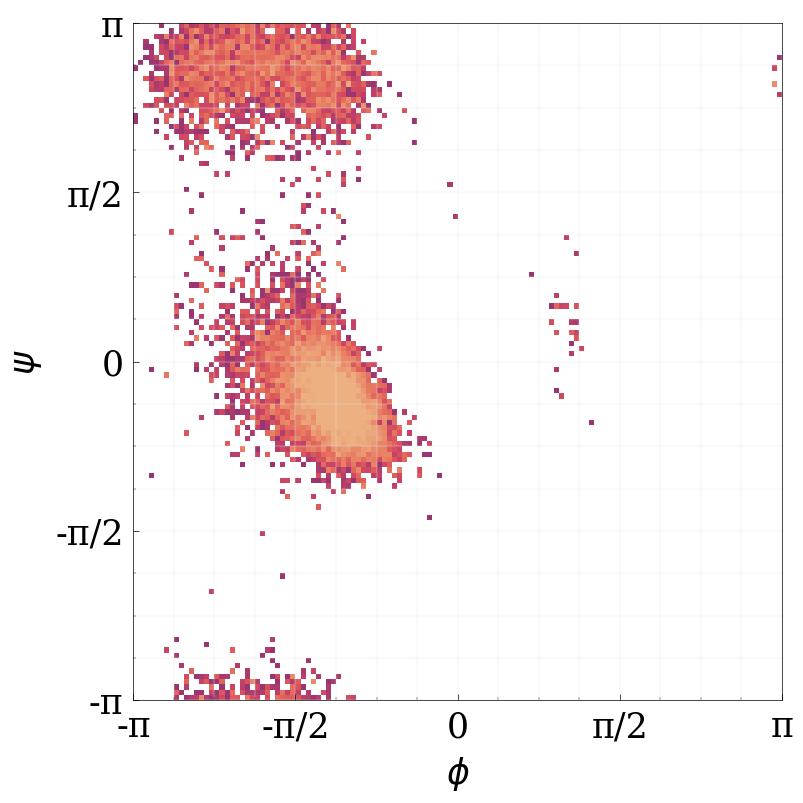} &
    \includegraphics[width=\imgw]{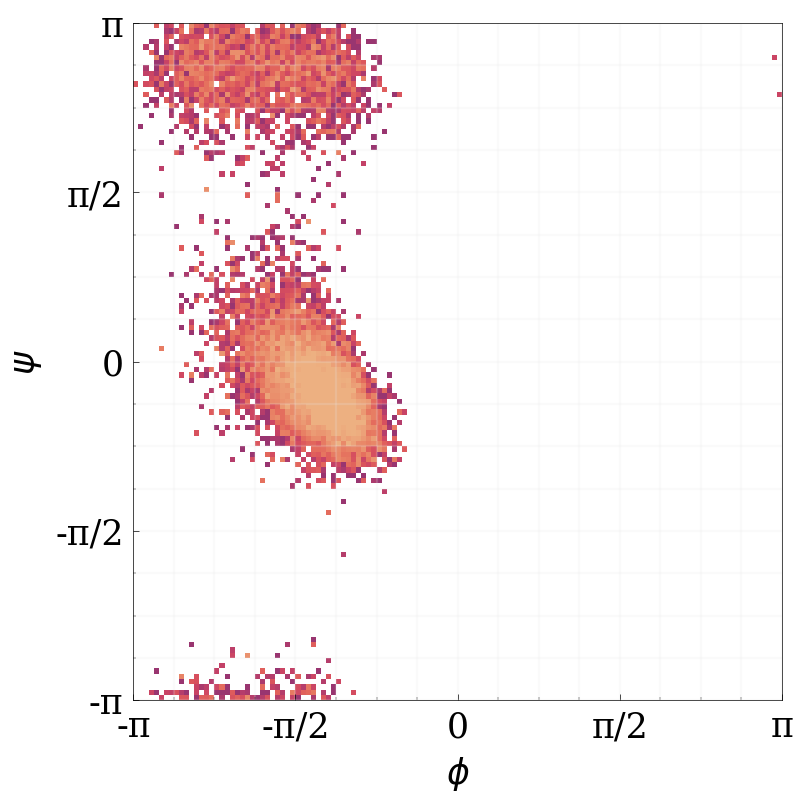} &
    \includegraphics[width=\imgw]{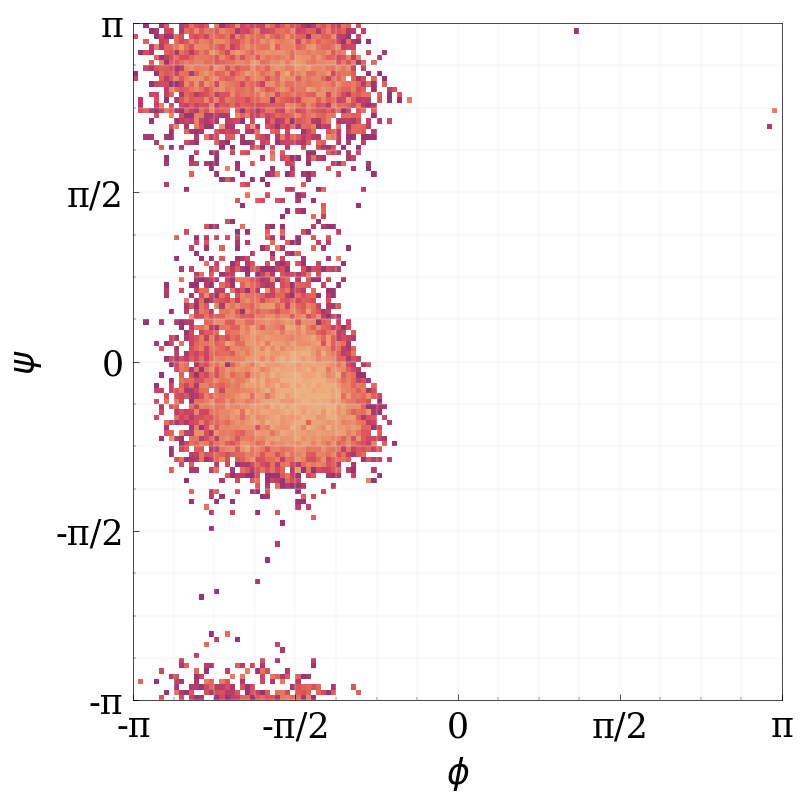} &
    &
    \\
    % ----- row 4: ala6 energy at (4,1); rest free -----
    \includegraphics[width=\imgw]{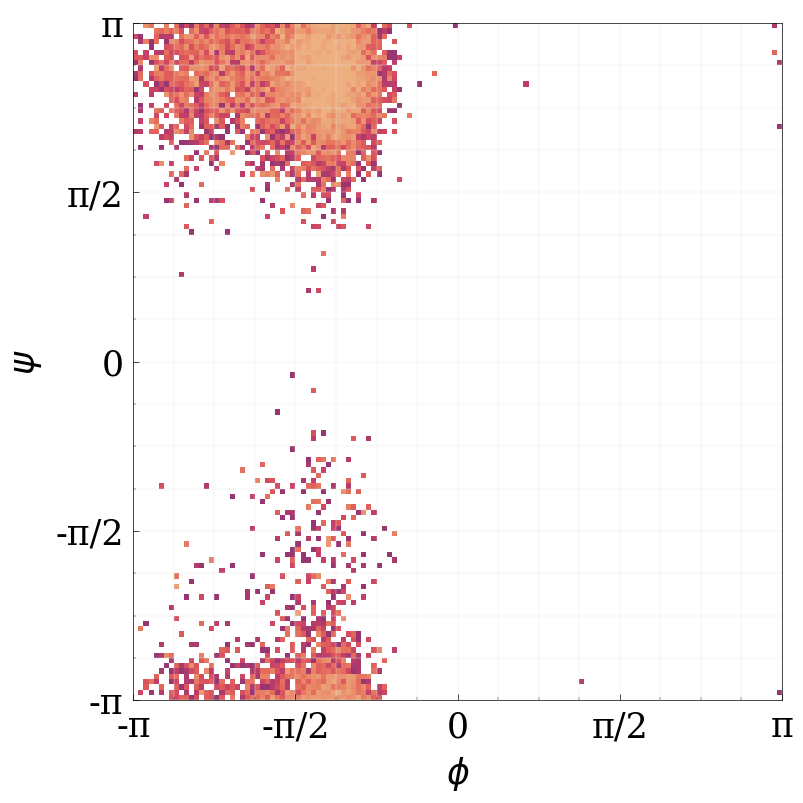} &
    \includegraphics[width=\imgw]{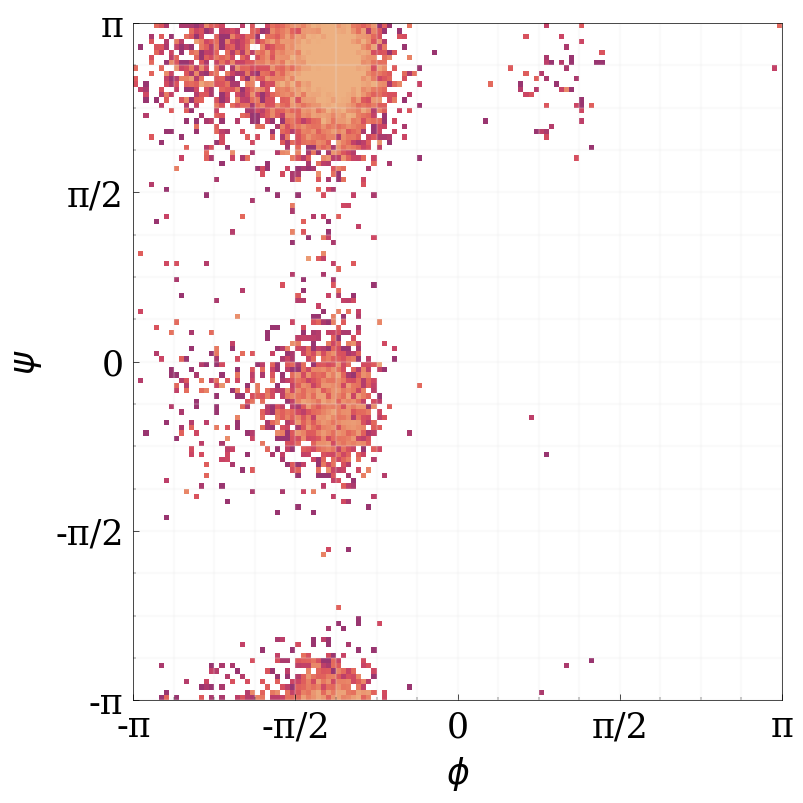}
    & \includegraphics[width=\imgw]{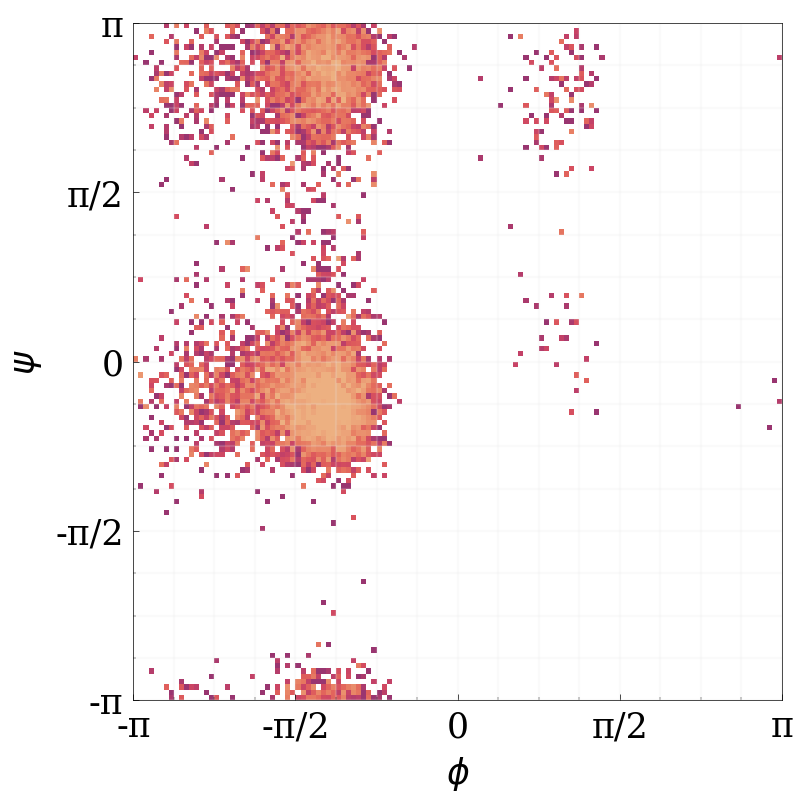}&\includegraphics[width=\imgw]{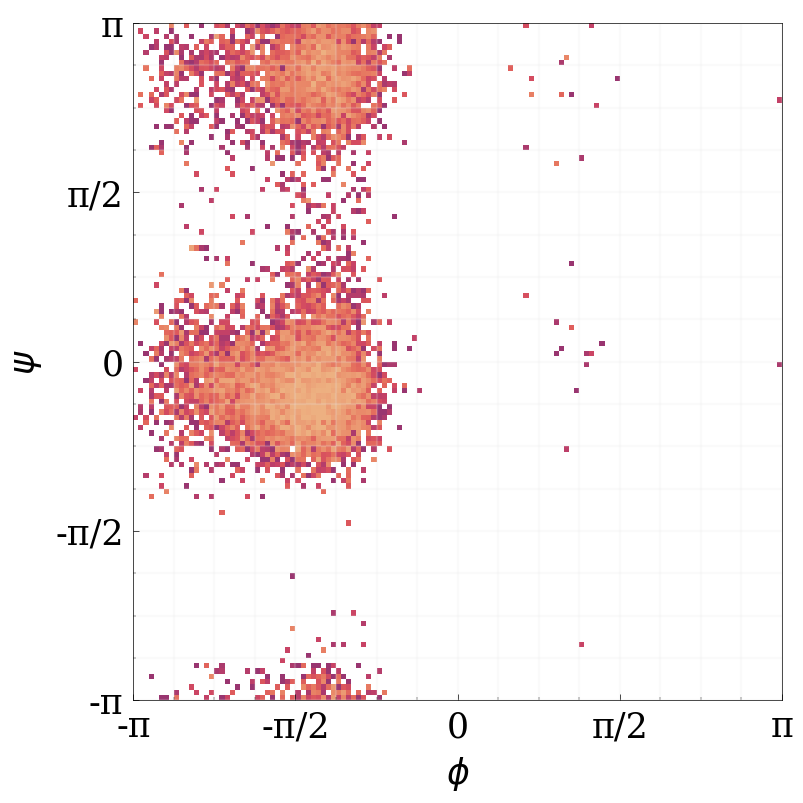} & \includegraphics[width=\imgw]{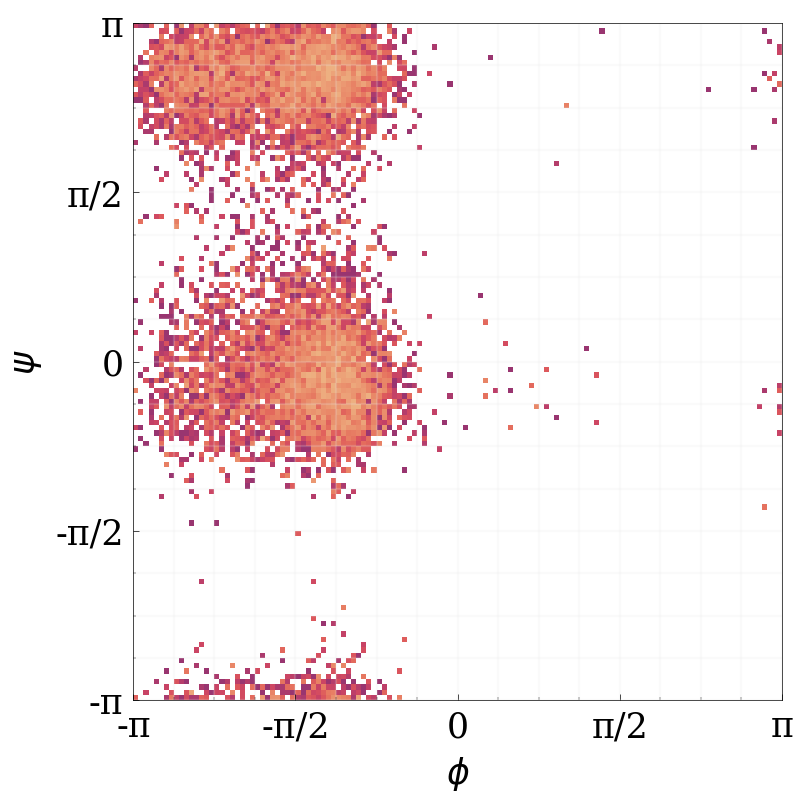}\\
  \end{tabular}
  \caption{Ramachandran plots for ALA-$N$ data generated by F2D2, from ALA-2 (TOP) to ALA-6 (BOTTOM). Torsional angle indices are ordered from LEFT to RIGHT.}
  \label{fig:f2d2-diheral}
\end{figure}

\begin{figure}[t]
    \centering
    \begin{subfigure}[b]{\linewidth}
        \centering
        \includegraphics[width=\linewidth]{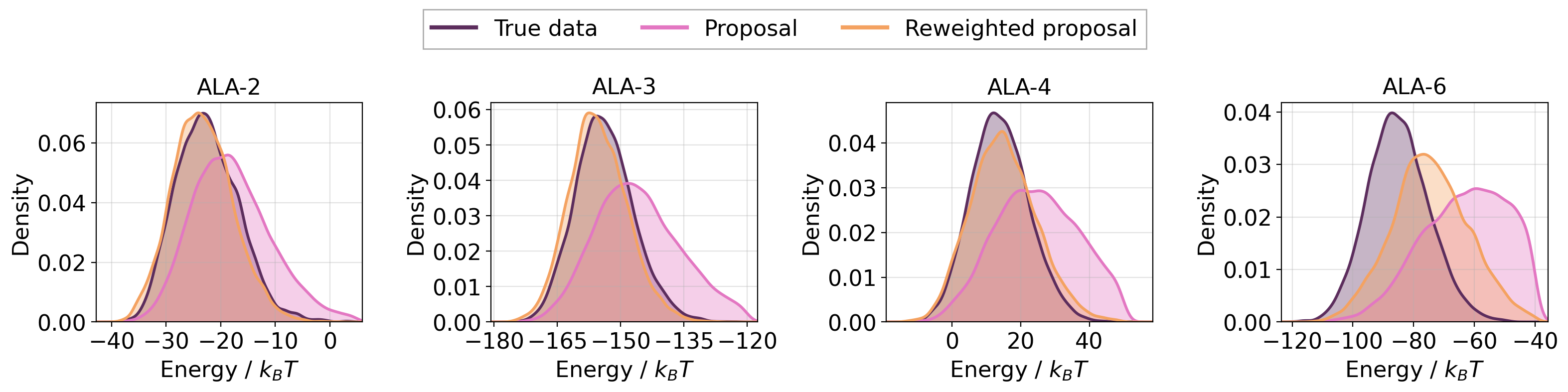}
        \subcaption{\scalloptt}
        \label{fig:scallop-energy-hist}
    \end{subfigure}
    \vfill
    \begin{subfigure}[b]{\linewidth}
        \centering
        \includegraphics[width=\linewidth]{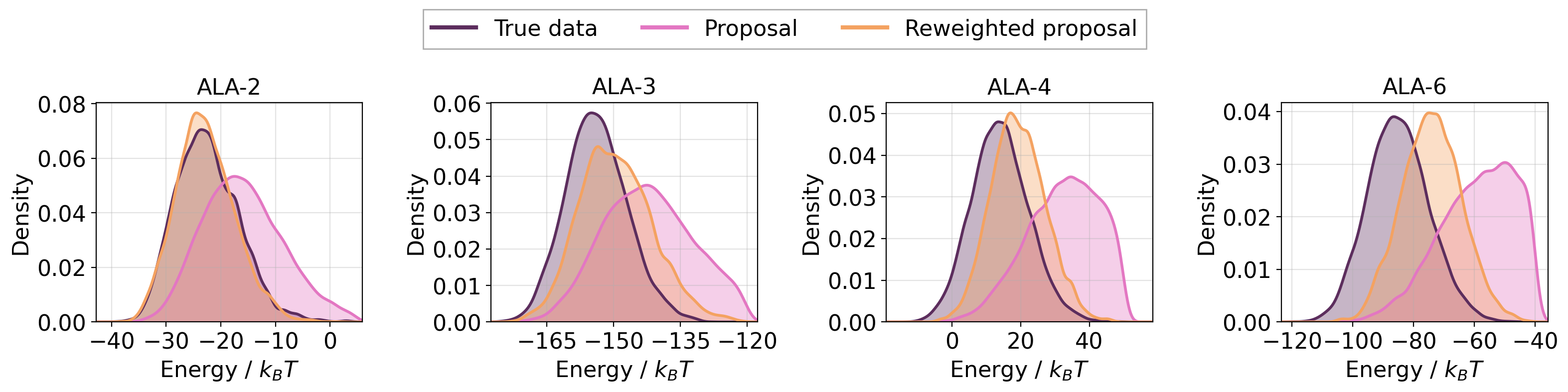}
        \subcaption{F2D2}
        \label{fig:f2d2-energy-hist}
    \end{subfigure}
    \caption{Energy histograms of resampled data from \scalloptt (TOP) and F2D2 (BOTTOM), ranging from ALA-2 to ALA-6.}
    \label{fig:combined-energy-hist}
\end{figure}

Recall that we employ (4, 8, 8, 16) steps for ALA-(2, 3, 4, 6) systems, following \cite{rehman2025falconfewstepaccuratelikelihoods}.

\paragraph{Ground Truth Torsional Angles Histograms.} We provide the Ramachandran plots, \textit{i.e.} the histograms of torsional angles, of the training data, from ALA-2 to ALA-6, in \Cref{fig:all-diheral}.

\paragraph{Resampled Torsional Angles Histograms.} We provide the Ramachandran plots for the resampled data in \Cref{fig:scallop-diheral} (for \scalloptt) and \Cref{fig:f2d2-diheral} (for F2D2).

\paragraph{Resampled Energy Histogram.} We provide the energy histograms of the resampled data in \Cref{fig:scallop-energy-hist} (for \scalloptt) and \Cref{fig:f2d2-energy-hist} (for F2D2).

\section{Experimental Details for Images}\label{app:sec:img-exp-details}

We use CelebA-64 \cite{liu2015faceattributes} as the image-generation benchmark and report both likelihood-based and sample-quality metrics. Unless otherwise specified, the model architecture and optimization hyperparameters follow the configuration summarized in \Cref{tab:config-celeba64}.

\paragraph{Dataset terms of use.} CelebA is provided for non-commercial research purposes only under the dataset authors' \href{https://mmlab.ie.cuhk.edu.hk/projects/CelebA.html}{usage agreement}, which also restricts redistribution of the dataset. We cite the original dataset paper above.

\begin{table}[t]
\centering
\renewcommand{\arraystretch}{1.2}
\caption{Training hyperparameters on CelebA-64.}
\begin{tabular}{lc}
\toprule
& \textbf{CelebA-64} \\
\midrule
Noise embedding      & Positional \\
Channels             & 128 \\
Channel multipliers  & 1,2,3,4 \\
Attention resolution & 16,8 \\
Residual blocks      & 3 \\
Dropout              & 0.0 \\
Batch size           & 256 \\
GPUs                 & 4 A6000s \\
Iterations           & 350k \\
Learning Rate        & 1e-2 (Sqrt decay at 35k) \\
Precision            & bfloat16 \\
Optimizer            & RAdam \\
EMA rate             & 0.9999 \\
\bottomrule
\end{tabular}
\label{tab:config-celeba64}
\end{table}

\textbf{LSD Model Training.} For the vanilla LSD baseline, we follow the experimental setup of \cite{boffi2026flowmaps}. We adopt the batch-allocation strategy used in \cite{ai2026f2d2}: during the first 200k iterations, 75\% of each batch is used for the flow-matching loss and the remaining 25\% is used for the self-distillation loss; during the subsequent 150k iterations, the allocation is changed to 50\% for flow matching and 50\% for self-distillation. For the self-distillation objective, we sample two time steps independently and uniformly from $[0,1]$, without imposing an ordering constraint between them. Equivalently, the full time-step pair is sampled as $(t,s) \sim \mathcal{U}([0,1]^2)$, which allows the model to learn both forward and backward estimators. 

\textbf{SCALLOP Training.} For \scalloptt{}, we follow the same image-generation training setup as \cite{ai2026f2d2} and initialize the model from the pretrained LSD model, adding only the divergence-prediction head. We train the model for 60k iterations with an initial learning rate of $10^{-4}$. At each iteration, 50\% of the batch is used for the flow-matching loss, while the remaining 50\% is used for the self-distillation loss. The vectorized divergence head is implemented on the decoder of the image UNet: instead of producing a single scalar divergence estimate, it applies a $1{\times}1$ convolution to predict a spatially resolved divergence field with the same shape as the noisy image state, namely $(B,C,H,W)$.

\textbf{FID Evaluation.} We compute FID using 50k images generated by each model.

\textbf{BPD Evaluation.} We compute bits per dimension (BPD) on the full CelebA-64 test set, following the procedure of \cite{lipman2023flow}. For numerical integration, we discretize the reverse-time interval from $t=1$ to $t=0$ into equally spaced segments.

\section{Additional Experimental Results for Images}
We train \scalloptt on CelebA-64 and compare to F2D2.

\paragraph{Stability of Training.}
\Cref{fig:loss-grad-comparison} compares the training dynamics of F2D2 and \scalloptt over the first 60k training steps, including both the training loss and the gradient norm. The results show that \scalloptt has more stable training dynamics with lower variance than F2D2, which is consistent with the observation in molecular systems.
\begin{figure}[t]
    \centering
    \includegraphics[width=0.45\linewidth]{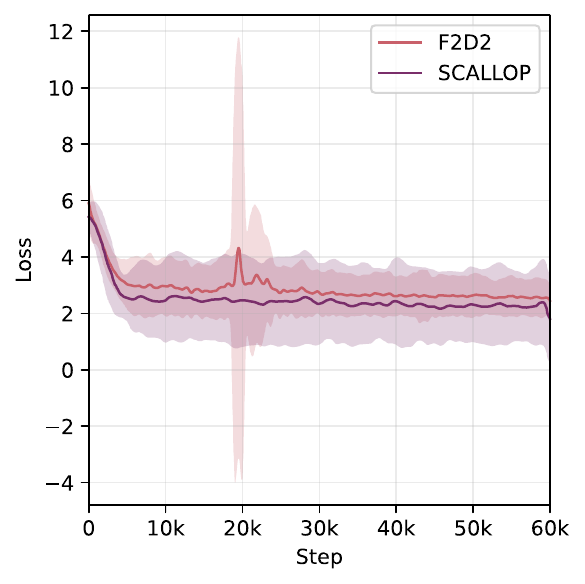}
    \hfill
    \includegraphics[width=0.45\linewidth]{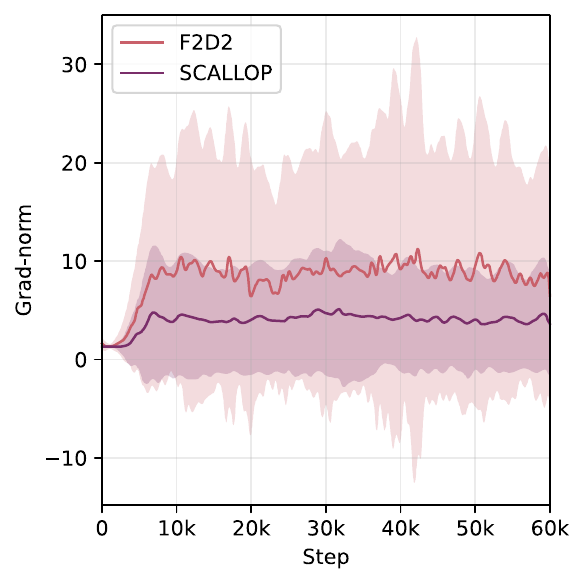}
    \caption{Training dynamics comparison between F2D2 and \scalloptt{} over 60k training steps.
    \textbf{Left:} Training loss.
    \textbf{Right:} Training gradient norm.
    Curves are smoothed for visualization, and shaded regions indicate local variability.}
    \label{fig:loss-grad-comparison}
\end{figure}

\paragraph{Image Generation Examples.} We provide qualitative generation results of \scalloptt on unconditional CelebA-64 in \Cref{fig:1step,fig:2step,fig:4step,fig:8step} 
% Figures~\ref{fig:1step},\ref{fig:2step},\ref{fig:4step},\ref{fig:8step}
, using 1, 2, 4, and 8 sampling steps respectively.

\begin{figure}[t]
    \centering
    \includegraphics[width=\linewidth]{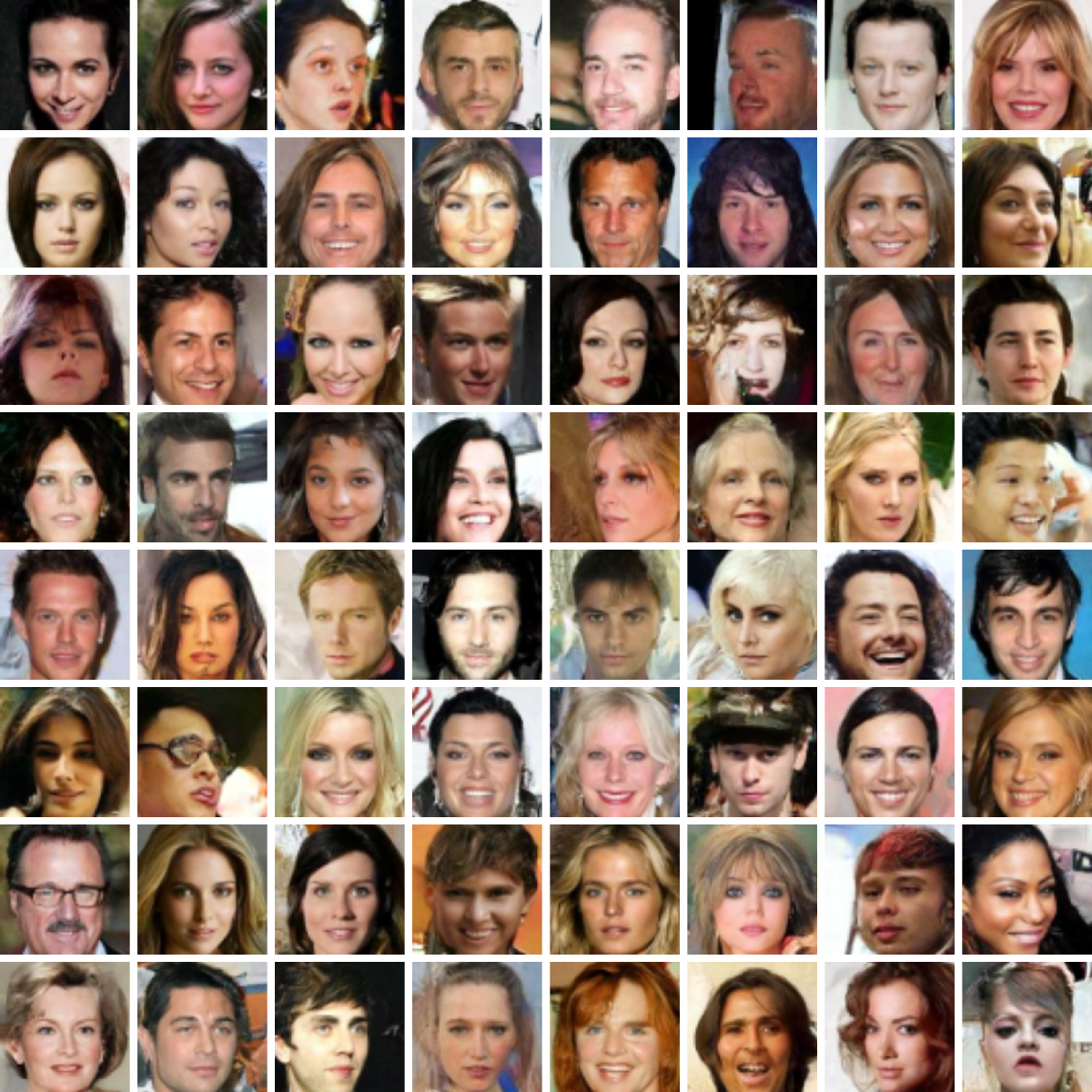}
    \caption{1-step unconditional CelebA-64 generation with \scalloptt. }
    \label{fig:1step}
\end{figure}

\begin{figure}[t]
    \centering
    \includegraphics[width=\linewidth]{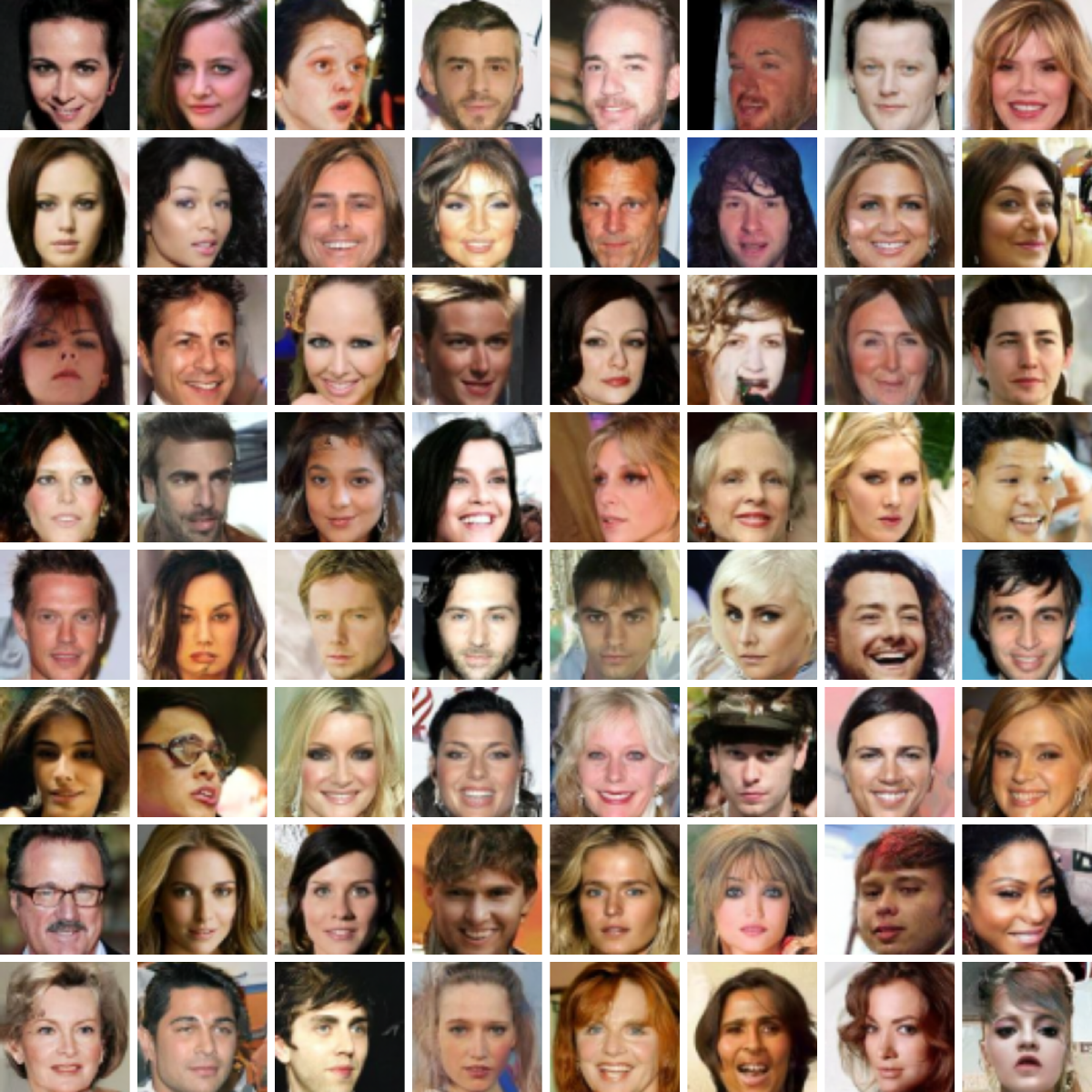}
    \caption{2-step unconditional CelebA-64 generation with \scalloptt. }
    \label{fig:2step}
\end{figure}

\begin{figure}[t]
    \centering
    \includegraphics[width=\linewidth]{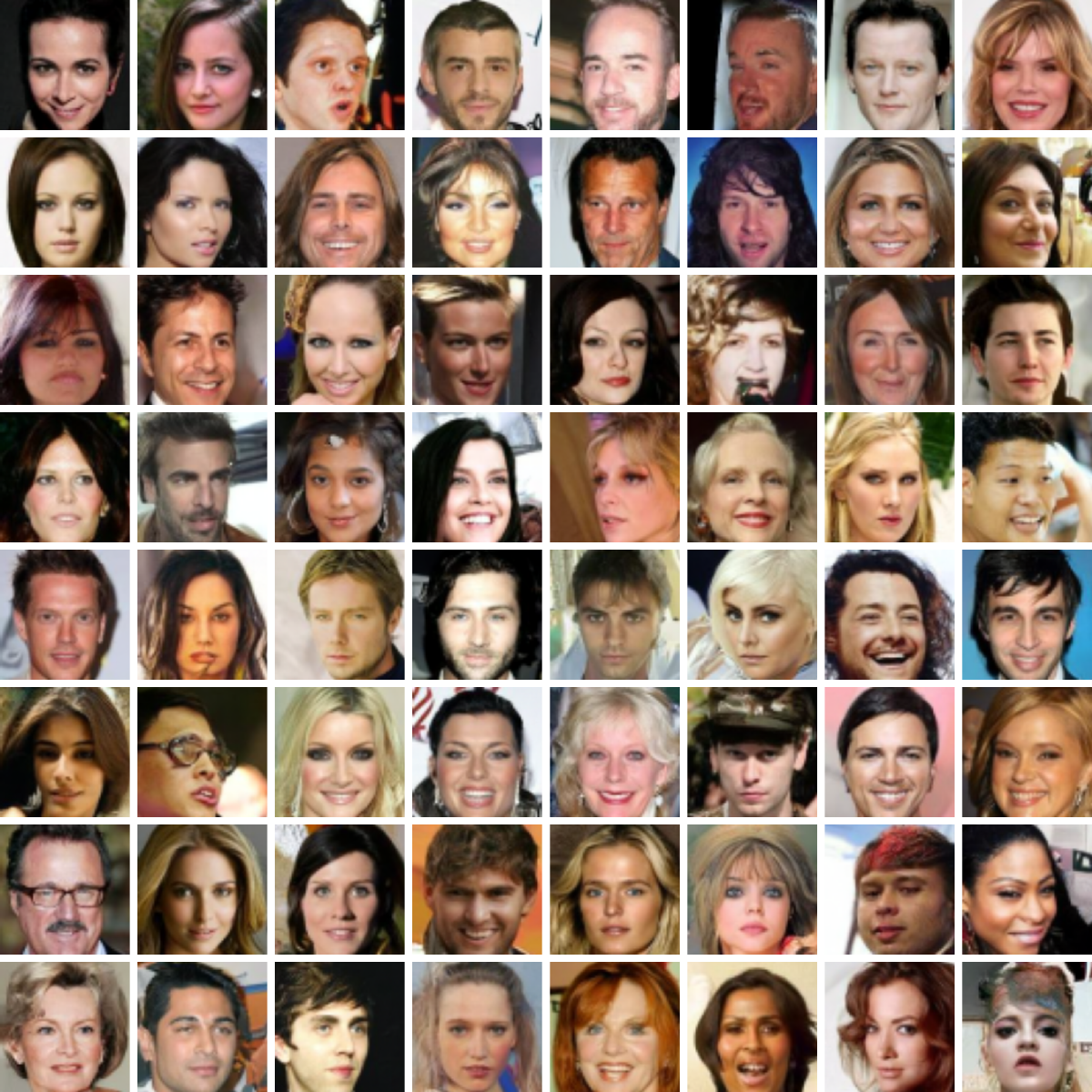}
    \caption{4-step unconditional CelebA-64 generation with \scalloptt. }
    \label{fig:4step}
\end{figure}

\begin{figure}[t]
    \centering
    \includegraphics[width=\linewidth]{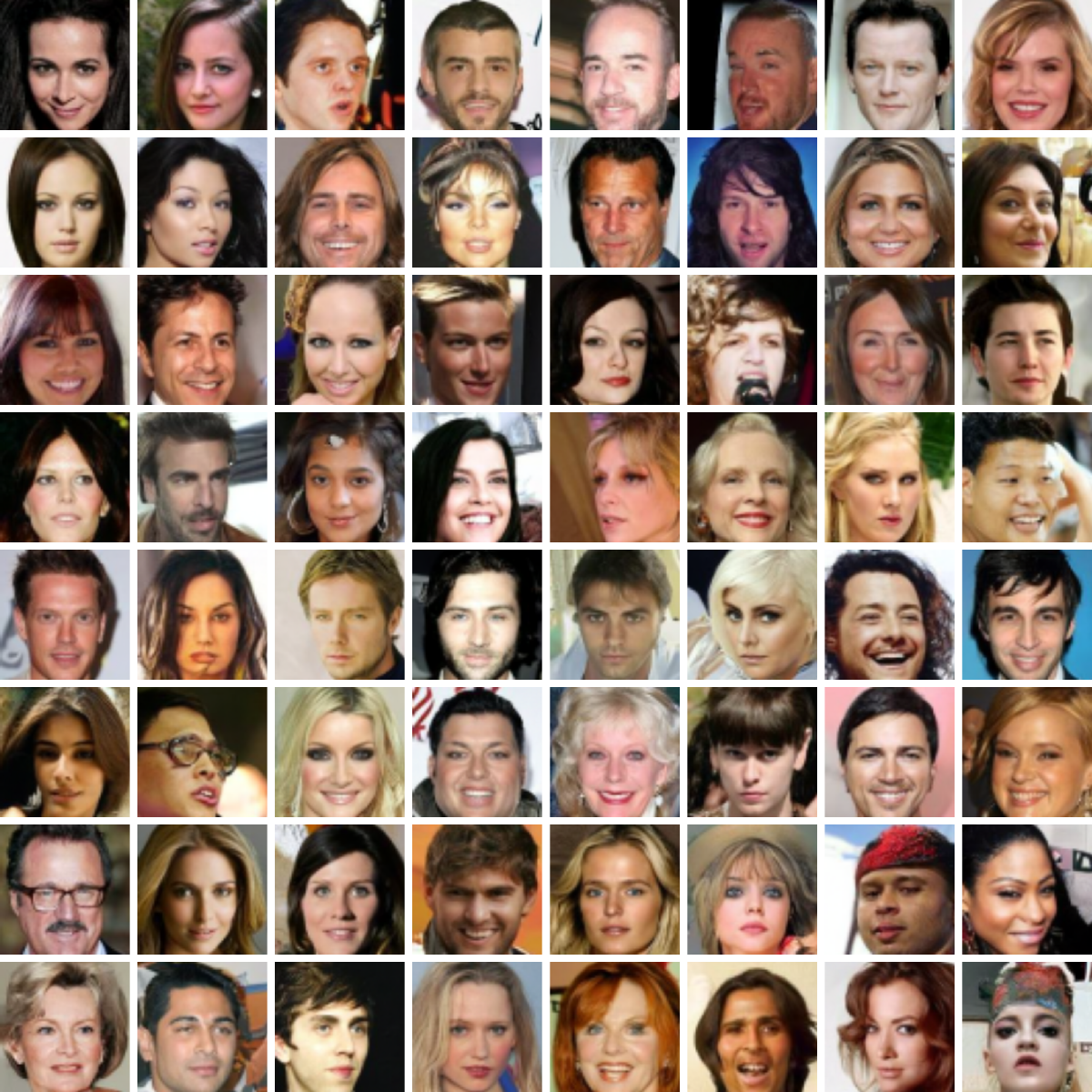}
    \caption{8-step unconditional CelebA-64 generation with \scalloptt. }
    \label{fig:8step}
\end{figure}

\end{document}